\documentclass[twoside,11pt]{article}

\usepackage{jmlr2e}
\usepackage{mathrsfs,amsmath,amsfonts,amssymb,bm,bbm,eufrak,dsfont,pifont,epsfig,euscript,stmaryrd}

\newcommand{\bb}{\mathbb}
\newcommand{\eu}{\EuScript}
\newcommand{\Scr}{\mathscr}

\ShortHeadings{Hilbert Space Embedding and Characteristic Kernels}{Sriperumbudur, Gretton, Fukumizu, Sch\"{o}lkopf and Lanckriet}
\firstpageno{1}

\begin{document}

\title{Hilbert Space Embeddings and Metrics on Probability Measures}

\author{\name Bharath K. Sriperumbudur \email bharathsv@ucsd.edu \\
       \addr Department of Electrical and Computer Engineering\\
       University of California, San Diego\\
       La Jolla, CA 92093-0407, USA
       \AND
		 \name Arthur Gretton \email arthur@tuebingen.mpg.de\\
		 \addr MPI for Biological Cybernetics\\
		 Spemannstra\ss e 38\\
		 72076, T\"{u}bingen, Germany
		 \AND
		 \name Kenji Fukumizu \email fukumizu@ism.ac.jp\\
		 \addr Institute of Statistical Mathematics\\
		 4-6-7 Minami-Azabu, Minato-ku\\
		 Tokyo 106-8569, Japan
		 \AND
		 \name Bernhard Sch\"{o}lkopf \email 	bernhard.schoelkopf@tuebingen.mpg.de\\
		 \addr MPI for Biological Cybernetics\\
		 Spemannstra\ss e 38\\
		 72076, T\"{u}bingen, Germany
		 \AND
		 \name Gert R. G. Lanckriet \email gert@ece.ucsd.edu \\
       \addr Department of Electrical and Computer Engineering\\
       University of California, San Diego\\
       La Jolla, CA 92093-0407, USA
}


\editor{}

\maketitle

\begin{abstract}
A Hilbert space embedding for probability measures has recently been proposed, with applications including dimensionality reduction, homogeneity testing, and independence testing. This embedding represents any probability measure as a mean element in a reproducing kernel Hilbert space (RKHS). A pseudometric on the space of probability measures can be defined as the distance between  distribution embeddings: we denote this as $\gamma_k$, indexed by the kernel function $k$ that defines the inner product in the RKHS.
\par We present three theoretical properties of $\gamma_k$. First, we consider the question of determining the conditions on the kernel $k$
for which $\gamma_k$ is a metric: such $k$ are denoted {\em characteristic kernels}. Unlike pseudometrics, a metric is zero only when two distributions
coincide, thus ensuring the RKHS embedding maps all distributions uniquely (i.e., the embedding is injective). While previously published conditions may apply only in restricted circumstances (e.g. on compact domains), and are difficult to check, our conditions are straightforward and intuitive: \emph{integrally strictly positive definite kernels} are characteristic. Alternatively, if a bounded continuous kernel is translation-invariant on $\bb{R}^d$, then it is characteristic if and only if the support of its Fourier transform is the entire $\bb{R}^d$. Second, we show that the distance between distributions under $\gamma_k$ results from an interplay between the properties of the kernel and the distributions, by demonstrating that distributions are close in the embedding space when their differences occur at higher frequencies. Third, to understand the nature of the topology induced by $\gamma_k$, we relate $\gamma_k$ to other popular metrics on probability measures, and present conditions on the kernel $k$ under which $\gamma_k$ metrizes the weak topology.
\end{abstract}


\section{Introduction}\label{Sec:Introduction}
The concept of distance between probability measures is a fundamental one and has found many applications in probability theory, information theory and statistics \citep{Rachev-91, Rachev-98, Liese-06}. In statistics, distances between probability measures are used in a variety of applications, including hypothesis tests (homogeneity tests, independence tests, and goodness-of-fit tests), density estimation, Markov chain monte carlo, etc. As an example, homogeneity testing, also called the two-sample problem, involves choosing whether to accept or reject a null hypothesis $H_0\,:\,\bb{P}=\bb{Q}$ versus the alternative $H_1\,:\,\bb{P}\ne \bb{Q}$, using random samples $\{X_j\}^{m}_{j=1}$ and $\{Y_j\}^{n}_{j=1}$ drawn i.i.d. from probability distributions $\bb{P}$ and $\bb{Q}$ on a topological space $(M,\eu{A})$. It is easy to see that solving this problem is equivalent to testing $H_0:\gamma(\bb{P},\bb{Q})=0$ versus $H_1:\gamma(\bb{P},\bb{Q})>0$, where $\gamma$ is a metric (or, more generally, a semi-metric\footnote{Given a set $M$, a \emph{metric} for $M$ is a function $\rho:M\times M\rightarrow\bb{R}_+$ such that \emph{(i)} $\forall\,x,\,\rho(x,x)=0$, \emph{(ii)} $\forall\,x,y,\,\rho(x,y)=\rho(y,x)$, \emph{(iii)} $\forall\,x,y,z,\,\rho(x,z)\le \rho(x,y)+\rho(y,z)$, and \emph{(iv)} $\rho(x,y)=0\Rightarrow x=y$. A semi-metric only satisfies \emph{(i)}, \emph{(ii)} and \emph{(iv)}. A pseudometric only satisfies \emph{(i)-(iii)} of the properties of a metric. 
Unlike a metric space $(M,\rho)$, points in a pseudometric space need not be distinguishable: one may have $\rho(x,y)=0$ for $x\ne y$. \par Now, in the two-sample test, though we mentioned that $\gamma$ is a metric/semi-metric, it is sufficient that $\gamma$ satisfies \emph{(i)} and \emph{(iv)}.}) on the space of all probability measures defined on $M$. The problems of testing independence and goodness-of-fit can be posed in an analogous form. In non-parametric density estimation, $\gamma(p_n,p_0)$ can be used to study the quality of the density estimate, $p_n$, that is based on the samples $\{X_j\}^n_{j=1}$ drawn i.i.d. from $p_0$. Popular examples for $\gamma$ in these statistical applications include the \emph{Kullback-Leibler divergence},
the \emph{total variation distance}, the \emph{Hellinger distance} \citep{Vajda-89}~---~these three are specific instances of the generalized $\phi$-divergence 
\citep{Ali-66, Csiszar-67}~---~the \emph{Kolmogorov distance} \citep[Section 14.2]{Lehmann-05}, the \emph{Wasserstein distance} \citep{Barrio-99}, 
etc. 
\par In probability theory, the distance between probability measures is used in studying limit theorems, the popular example being the central limit theorem. Another application is in metrizing the weak convergence of probability measures on a separable metric space, where the \emph{L\'{e}vy-Prohorov distance} \citep[Chapter 11]{Dudley-02} and \emph{dual-bounded Lipschitz distance} (also called the \emph{Dudley metric}) \citep[Chapter 11]{Dudley-02} are commonly used.
\par In the present work, we will consider a particular pseudometric\footnotemark[1] on probability
distributions which is an instance of an \emph{integral probability metric} (IPM) \citep{Muller-97}. Denoting $\Scr{P}$ the set of all Borel probability measures on $(M,\eu{A})$, the IPM between $\bb{P}\in \Scr{P}$ and $\bb{Q}\in \Scr{P}$ is defined as
\begin{equation}\label{Eq:MMD}
\gamma_\eu{F}(\bb{P},\bb{Q})=\sup_{f\in\eu{F}}\left|\int_Mf\,d\bb{P}-\int_Mf\,d\bb{Q}\right|,
\end{equation}
where $\eu{F}$ is a class of real-valued bounded measurable functions on $M$. In addition to the general application domains discussed earlier for metrics on probabilities, IPMs have been used in proving central limit theorems using Stein's method \citep{Stein-72, Barbour-05}, and are popular in empirical process theory \citep{Vaart-96}. Since most of the applications listed above require $\gamma_\eu{F}$ to be a metric on $\Scr{P}$, the choice of $\eu{F}$ is critical (note that irrespective of $\eu{F}$, $\gamma_\eu{F}$ is a pseudometric on $\Scr{P}$). The following are some examples of $\eu{F}$ for which $\gamma_\eu{F}$ is a metric.
\begin{itemize}
\item[(a)] $\eu{F}=C_b(M)$, the space of bounded continuous functions on $(M,\rho)$, where $\rho$ is a metric \citep[Chapter 19, Definition 1.1]{Shorack-00}.
\item[(b)] $\eu{F}=C_{bu}(M)$, the space of bounded $\rho$-uniformly continuous functions on $(M,\rho)$~---~Portmonteau theorem \citep[Chapter 19, Theorem 1.1]{Shorack-00}.
\item[(c)] $\eu{F}=\{f:\Vert f\Vert_\infty\le 1\}=:\eu{F}_{TV}$, where $\Vert f\Vert_\infty=\sup_{x\in M}|f(x)|$. $\gamma_\eu{F}$ is called the \emph{total variation distance} \citep[Chapter 19, Proposition 2.2]{Shorack-00}, which we denote as $TV$, i.e., $\gamma_{\eu{F}_{TV}}=:TV$. 
\item[(d)] $\eu{F}=\{f:\Vert f\Vert_L\le 1\}=:\eu{F}_W$, where $\Vert f\Vert_L:=\sup\{|f(x)-f(y)|/\rho(x,y):x\ne y\,\,\text{in}\,\, M\}$. $\Vert f\Vert_L$ is the Lipschitz semi-norm of a real-valued function $f$ on $M$ and $\gamma_\eu{F}$ is called the \emph{Kantorovich metric}. If $(M,\rho)$ is separable, then $\gamma_\eu{F}$ equals the \emph{Wasserstein distance} \cite[Theorem 11.8.2]{Dudley-02}, denoted as $W:=\gamma_{\eu{F}_W}$.
\item[(e)] $\eu{F}=\{f:\Vert f\Vert_{BL}\le 1\}=:\eu{F}_\beta$, where $\Vert f\Vert_{BL}:=\Vert f\Vert_{L}+\Vert f\Vert_\infty$. $\gamma_\eu{F}$ is called the \emph{Dudley metric} \citep[Chapter 19, Definition 2.2]{Shorack-00}, denoted as $\beta:=\gamma_{\eu{F}_\beta}$.
\item[(f)] $\eu{F}=\{\mathds{1}_{(-\infty,t]}:t\in\bb{R}^d\}=:\eu{F}_{KS}$. $\gamma_\eu{F}$ is called the \emph{Kolmogorov distance} \citep[Theorem 2.4]{Shorack-00}.
\item[(g)] $\eu{F}=\{e^{\sqrt{-1}\langle\omega,\cdot\rangle}:\omega\in\bb{R}^d\}=:\eu{F}_c$. This choice of $\eu{F}$ results in the maximal difference between the characteristic functions of $\bb{P}$ and $\bb{Q}$. That $\gamma_{\eu{F}_c}$ is a metric on $\Scr{P}$ follows from the \emph{uniqueness theorem} for characteristic functions \citep[Theorem 9.5.1]{Dudley-02}.
\end{itemize}
\par Recently, \cite{Gretton-06} and \cite{Smola-07}  considered $\eu{F}$ to be the unit ball in a reproducing kernel Hilbert space (RKHS) $\eu{H}$ \citep{Aronszajn-50}, with $k$ as its reproducing kernel (r.k.), i.e., $\eu{F}=\{f:\Vert f\Vert_\eu{H}\le 1\}=:\eu{F}_k$ (also see Chapter 4 of \cite{Berlinet-04} and references therein for  related work): we denote $\gamma_{\eu{F}_k}=:\gamma_k$. 
While we have seen many possible $\eu{F}$ for which $\gamma_\eu{F}$ a metric,
$\eu{F}_k$ has a number of important advantages:
\begin{itemize}
\item \textbf{Estimation of $\gamma_\eu{F}$:} In applications such as hypothesis testing, $\bb{P}$ and $\bb{Q}$ are known only through the respective random samples $\{X_j\}^m_{j=1}$ and $\{Y_j\}^n_{j=1}$ drawn i.i.d. from each, and $\gamma_\eu{F}(\bb{P},\bb{Q})$ is estimated based on these samples. One approach is to compute $\gamma_\eu{F}(\bb{P},\bb{Q})$ using the empirical measures $\bb{P}_m=\frac{1}{m}\sum^m_{j=1}\delta_{X_j}$ and $\bb{Q}_n=\frac{1}{n}\sum^n_{j=1}\delta_{Y_j}$, where $\delta_x$ represents a Dirac measure at $x$. It can be shown that choosing $\eu{F}$ as 
$C_b(M)$, $C_{bu}(M)$, $\eu{F}_{TV}$ or $\eu{F}_c$ results in this approach not yielding consistent estimates of $\gamma_\eu{F}(\bb{P},\bb{Q})$ for all $\bb{P}$ and $\bb{Q}$ \citep{Devroye-90}. Although choosing $\eu{F}=\eu{F}_W$ or $\eu{F}_\beta$ yields consistent estimates of $\gamma_\eu{F}(\bb{P},\bb{Q})$ for all $\bb{P}$ and $\bb{Q}$ when $M=\bb{R}^d$, the rates of convergence are dependent on $d$ and become slow for large $d$ \citep{Sriperumbudur-09}. On the other hand, $\gamma_{k}(\bb{P}_m,\bb{Q}_n)$ is a $\sqrt{mn/(m+n)}$-consistent estimator of $\gamma_{k}(\bb{P},\bb{Q})$ if $k$ is measurable and bounded, for all $\bb{P}$ and $\bb{Q}$. If $k$ is translation invariant on $M=\bb{R}^d$, the rate is independent of $d$ \citep{Gretton-06, Sriperumbudur-09}, an important property when dealing with high dimensions. Moreover, $\gamma_\eu{F}$ is not straightforward to compute when $\eu{F}$ is $C_b(M),\,C_{bu}(M),\,\eu{F}_W$ or $\eu{F}_\beta$ \citep[Section 2.3]{Weaver-99}: by contrast, $\gamma^2_k(\bb{P},\bb{Q})$ is simply a sum of expectations of the kernel $k$ (see Theorem~\ref{Theorem:MMD-II} and (\ref{Eq:computeMMD})).
\item \textbf{Comparison to $\phi$-divergences:} Instead of using $\gamma_\eu{F}$ in statistical applications, one can also use $\phi$-divergences. However, the estimators of $\phi$-divergences (especially the Kullback-Leibler divergence) exhibit arbitrarily slow rates of convergence depending on the distributions (see \citet{Wang-05, Nguyen-08} and references therein for details), while, as noted above, $\gamma_k(\bb{P}_m,\bb{Q}_n)$ exhibits good convergence behavior.
\item \textbf{Structured domains:} Since $\gamma_k$ is dependent only on the kernel (see Theorem~\ref{Theorem:MMD-II}) and kernels can be defined on arbitrary domains $M$ \citep{Aronszajn-50}, choosing $\eu{F}=\eu{F}_k$ provides the flexibility of measuring the distance between probability measures defined on structured domains \citep{Borgwardt-06b} like graphs, strings, etc., unlike $\eu{F}=\eu{F}_{KS}$ or $\eu{F}_c$, which can handle only $M=\bb{R}^d$.
\end{itemize}
\par The distance measure $\gamma_k$ has appeared in a wide variety of applications. These include statistical hypothesis testing, of homogeneity \citep{Gretton-06}, independence \citep{Gretton-08}, and conditional independence \citep{Fukumizu-08a}; as well as in machine learning
applications including kernel independent component analysis \citep{Bach-02,Gretton-05a} and  kernel based dimensionality reduction for supervised learning \citep{Fukumizu-04}. In these applications, kernels offer a linear approach to deal with higher order statistics: given the problem of homogeneity testing, for example, differences in higher order moments are encoded as differences in the means of nonlinear features of the variables. To capture all nonlinearities that are relevant to the problem at hand, the embedding RKHS therefore has to be ``sufficiently large" that differences in the embeddings correspond to differences of interest in the distributions. Thus, a natural question is how to guarantee  $k$ provides a sufficiently rich RKHS so as to detect {\em any} difference in distributions. 
A second problem is to determine what properties of  distributions result in their being proximate or distant in the embedding space. Finally, we would like to compare  $\gamma_k$  to the classical integral probability metrics listed earlier, when used to measure convergence of distributions. In the following
 section, we describe the contributions of the present paper, addressing each of these three questions in turn.

\subsection{Contributions}
The contributions in this paper are three-fold and explained in detail below. 
\subsubsection{When is $\eu{H}$ characteristic?}\label{subsubsec:contribution1}
Recently, \cite{Fukumizu-08a} introduced the concept of a \emph{characteristic kernel}, i.e., a reproducing kernel for which $\gamma_k(\bb{P},\bb{Q})=0\Leftrightarrow\bb{P}=\bb{Q}$, $\bb{P},\bb{Q}\in\Scr{P}$, i.e., $\gamma_k$ is a metric on $\Scr{P}$. The corresponding RKHS, $\eu{H}$ is referred to as a \emph{characteristic RKHS}. 
The following are two characterizations for characteristic RKHSs that have already been studied in literature:
\begin{enumerate}
\item When $M$ is compact, \citet{Gretton-06} showed that $\eu{H}$ is characteristic if $k$ is \emph{universal} in the sense of \citet[Definition 4]{Steinwart-01}, i.e., $\eu{H}$ is dense in the Banach space of bounded continuous functions with respect to the supremum norm. Examples of such $\eu{H}$ include those induced by the Gaussian and Laplacian kernels on every compact subset of $\bb{R}^d$.
\item \citet{Fukumizu-08a, Fukumizu-09} extended this characterization to non-compact $M$ and showed that $\eu{H}$ is characteristic if and only if the direct sum of $\eu{H}$ and $\bb{R}$ is dense in the Banach space of $r$-integrable (for some $r\ge 1$) functions. Using this characterization, they showed that the RKHSs induced by the Gaussian and Laplacian kernels (supported on the entire $\bb{R}^d$) are characteristic.
\end{enumerate}
\par In the present study, we provide alternative conditions for characteristic RKHSs which address several limitations of the foregoing. First, it can be difficult to verify the conditions of denseness in both of the above characterizations. Second, universality is in any case an overly restrictive condition because universal kernels assume $M$ to be compact, i.e., they induce a metric only on the space of probability measures that are supported on compact $M$. In addition, there are compactly supported kernels which are not universal, e.g., $B_{2n+1}$-splines \citep{Steinwart-01}, which can be shown to be characteristic. 
\par In Section~\ref{subsec:strictpd}, we present the simple characterization that \emph{integrally strictly positive definite} (pd) kernels (see Section~\ref{Sec:Notation} for the definition) are characteristic, i.e., the induced RKHS is characteristic \citep[also see][Theorem 4]{Sriperumbudur-09c}. This condition is more natural~---~strict pd is a natural property of interest for kernels, unlike the denseness condition~---~and much easier to understand than the characterizations mentioned above. 
Examples of integrally strictly pd kernels on $\bb{R}^d$ include the Gaussian, Laplacian, inverse multiquadratics, Mat\'{e}rn kernel family, $B_{2n+1}$-splines, etc. 
\par Although the above characterization of integrally strictly pd kernels being characteristic is simple to understand, it is only a sufficient condition and does not provide an answer for kernels that are not integrally strictly pd,\footnote{It can be shown that integrally strictly pd kernels are strictly pd (see footnote~\ref{footnote:ispd}). Therefore, examples of kernels that are not integrally strictly pd include those kernels that are not strictly pd.} e.g., a Dirichlet kernel. 
Therefore, in Section~\ref{subsec:Rd}, we provide an easily checkable condition, after making some assumptions on the kernel. We present a complete characterization of characteristic kernels when the kernel is translation invariant on $\bb{R}^d$. We show that a bounded continuous translation invariant kernel on $\bb{R}^d$ is characteristic if and only if the support of the Fourier transform of the kernel is the entire $\bb{R}^d$. This condition is easy to check compared to the characterizations described above. 
An earlier version of this result was provided by \citet{Sriperumbudur-08}: by comparison, we now present a simpler and more elegant proof. We also show that all compactly supported translation invariant kernels on $\bb{R}^d$ are characteristic. Note, however, that the characterization of integral strict positive definiteness in Section~\ref{subsec:strictpd} does not assume $M$ to be $\bb{R}^d$ nor $k$ to be translation invariant.
\par We  extend the result of Section~\ref{subsec:Rd} to $M$ being a $d$-Torus, i.e., $\bb{T}^d=S^1\times\stackrel{d}{\ldots}\times S^1\equiv [0,2\pi)^d$, where $S^1$ is a circle. In Section~\ref{subsec:Td}, we show that a translation invariant kernel on $\bb{T}^d$ is characteristic if and only if the Fourier series coefficients of the kernel are positive, i.e., the support of the Fourier spectrum is the entire $\bb{Z}^d$. The proof of this result is similar in flavor to the one in Section~\ref{subsec:Rd}. As examples, the Poisson kernel can be shown to be characteristic, while the Dirichlet kernel is not.
\par Based on the discussion so far, it is clear that the characteristic property of $k$ is characterized in many ways. Given these characterizations, we would like to understand the relation betweeen them. For example, we know that if $k$ is universal, then it is characteristic. Is the converse true? Similarly, as we mentioned before, integrally strictly pd kernels are characteristic and are also also strictly pd. Then what is the relation between characteristic and strictly pd kernels? In Section~\ref{subsec:relation}, we address these questions by exploring the relation between these characterizations, which are summarized in Figure~\ref{Fig:relation}.

\subsubsection{Dissimilar distributions with small $\gamma_k$}\label{subsubsec:contribution2}
As we have seen, the characteristic property of a kernel is critical in distinguishing between distinct probability measures. Suppose, however, that for a given characteristic kernel $k$ and for any $\varepsilon>0$, there exist $\bb{P}$ and $\bb{Q}$, $\bb{P}\ne \bb{Q}$, such that $\gamma_k(\bb{P},\bb{Q})<\varepsilon$. Though $k$ distinguishes between such $\bb{P}$ and $\bb{Q}$, it can be difficult to tell the distributions apart in applications (even with characteristic kernels), since $\bb{P}$ and $\bb{Q}$ are then replaced with finite samples, and the distance between them may not be statistically significant \citep{Gretton-06}. Therefore, given a characteristic kernel, it is of interest to determine the properties of distributions $\bb{P}$ and $\bb{Q}$ that will cause their embeddings to be close. To this end, in Section~\ref{Sec:limitation}, we show that given a kernel $k$ (see Theorem~\ref{Thm:generic} for conditions on the kernel), for any $\varepsilon>0$, there exists $\bb{P}\ne \bb{Q}$ (with non-trivial differences between them) such that $\gamma_{k}(\bb{P},\bb{Q})<\varepsilon$. 
These distributions are constructed so as to differ at a sufficiently high frequency, which is then penalized by the RKHS norm when computing $\gamma_k$.

\subsubsection{When does $\gamma_k$ metrize the weak topology on $\Scr{P}$?}\label{subsubsec:contribution3}
Given $\gamma_k$, which is a metric on $\Scr{P}$, a natural question of theoretical and practical importance to ask is ``how is $\gamma_k$ related to other probability metrics, such as the Dudley metric ($\beta$), Wasserstein distance ($W$), total variation metric ($TV$), etc?" For example, in applications like density estimation, wherein the unknown density is estimated based on finite samples drawn i.i.d. from it, the quality of the estimate is measured by computing the distance between the true density and the estimated density. In such a setting, given two probability metrics, $\rho_1$ and $\rho_2$, one might want to use the \emph{stronger}\footnote{Two metrics $\rho_1:Y\times Y\rightarrow\bb{R}_+$ and $\rho_2:Y\times Y\rightarrow \bb{R}_+$ are said to be equivalent if $\rho_1(x,y)=0\Leftrightarrow\rho_2(x,y)=0,\,\forall\,x,y\in Y$. On the other hand, $\rho_1$ is said to be stronger than $\rho_2$ if $\rho_1(x,y)=0\Rightarrow\rho_2(x,y)=0,\,\forall\,x,y\in Y$ but not vice-versa. If $\rho_1$ is stronger than $\rho_2$, then we say $\rho_2$ is weaker than $\rho_1$. Note that if $\rho_1$ is stronger (\emph{resp.} weaker) than $\rho_2$, then the topology induced by $\rho_1$ is finer (\emph{resp.} coarser) than the one induced by $\rho_2$.\label{fnote:strong-weak}} of the two to determine this distance, as the convergence of the estimated density to the true density in the stronger metric implies the convergence in the weaker metric, while the converse is not true. On the other hand, one might need to use a metric of weaker topology (i.e., coarser topology) to show convergence of some estimators, as the convergence might not occur w.r.t. a metric of strong topology. Clarifying and comparing the topology of a metric on the probabilities is, thus, important in the analysis of density estimation. Based on this motivation, in Section~\ref{Sec:weak}, we analyze the relation between $\gamma_k$ and other probability metrics, and show that $\gamma_k$ is weaker than all these other metrics.
\par It is well known in probability theory that $\beta$ is weaker than $W$ and $TV$, and it metrizes the weak topology (we will provide formal definitions in Section~\ref{Sec:weak}) on $\Scr{P}$  \citep{Shorack-00, Gibbs-02}. Since $\gamma_k$ is weaker than all these other probability metrics, i.e., the topology induced by $\gamma_k$ is coarser than the one induced by these metrics, the next interesting question to answer would be, ``When does $\gamma_k$ metrize the weak topology on $\Scr{P}$?" In other words, for what $k$, does the topology induced by $\gamma_k$ coincides with the weak topology? Answering this question would show that $\gamma_k$ is equivalent to $\beta$, while it is weaker than $W$ and $TV$. In probability theory, the metrization of weak topology is of prime importance in proving results related to the weak convergence of probability measures. Therefore, knowing the answer to the above question will help in using $\gamma_k$ as a theoretical tool in probability theory. To this end, in Section~\ref{Sec:weak}, we show that universal kernels on compact $(M,\rho)$ metrize the weak topology on $\Scr{P}$. For the non-compact setting, we assume $M=\bb{R}^d$ and provide sufficient conditions on the kernel such that $\gamma_k$ metrizes the weak topology on $\Scr{P}$.\\
\par In the following section, we introduce the notation and some definitions that are used throughout the paper. 
Supplementary results used in proofs are collected in Appendix A.
\subsection{Definitions and notation}\label{Sec:Notation}
For $M\subset\bb{R}^d$ and $\mu$ a Borel measure on $M$, $L^r(M,\mu)$ denotes the Banach space of $r$-power ($r\ge 1$) $\mu$-integrable functions. We will also use $L^r(M)$ for $L^r(M,\mu)$ and $dx$ for $d\mu(x)$ if $\mu$ is the Lebesgue measure on $M$. $C_b(M)$ denotes the space of all bounded, continuous functions on $M$. The space of all $r$-continuously differentiable functions on $M$ is denoted by $C^r(M),\,0\le r\le\infty$. For $x\in\bb{C}$, $\overline{x}$ represents the complex conjugate of $x$. We denote as $i$ the imaginary unit $\sqrt{-1}$. 
\par For a measurable function $f$ and a signed measure $\bb{P}$, $\bb{P}f:=\int f\,d\bb{P}=\int_{M}f(x)\,d\bb{P}(x)$. $\delta_{x}$ represents the Dirac measure at $x$. The symbol $\delta$ is overloaded to represent the Dirac measure, the Dirac-delta distribution, and the Kronecker-delta, which should be distinguishable from the context. For $M=\bb{R}^d$, the characteristic function, $\phi_\bb{P}$ of $\bb{P}\in\Scr{P}$ is defined as  $\phi_\bb{P}(\omega):=\int_{\bb{R}^d} e^{i\omega^Tx}\,d\bb{P}(x),\,\omega\in\bb{R}^d$. \vspace{2mm}\par\noindent
\textbf{Vanishing at infinity and $C_0(M)$:} A complex function $f$ on a locally compact Hausdorff space $M$ is said to \emph{vanish at infinity} if for every $\epsilon>0$ there exists a compact set $K\subset M$ such that $|f(x)|<\epsilon$ for all $x\notin K$. The class of all continuous $f$ on $M$ which vanish at infinity is denoted as $C_0(M)$. \vspace{2mm}\par\noindent
\textbf{Holomorphic and entire functions:} Let $D\subset\bb{C}^d$ be an open subset and $f:D\rightarrow\bb{C}$ be a function. $f$ is said to be \emph{holomorphic} at the point $z_0\in D$ if 
\begin{equation}f^\prime(z_0):=\lim_{z\rightarrow z_0} \frac{f(z_0)-f(z)}{z_0-z}
\end{equation} exists. Moreover, $f$ is called holomorphic if it is holomorphic at every $z_0\in D$. $f$ is called an \emph{entire function} if $f$ is holomorphic and $D=\bb{C}^d$.\vspace{2mm}\par\noindent
\textbf{Positive definite and strictly positive definite:} A function $k:M\times M\rightarrow\bb{R}$ is called \emph{positive definite} (pd) if, for all $n\in\bb{N}$, $\alpha_1,\ldots,\alpha_n\in\bb{R}$ and all $x_1,\ldots,x_n\in M$, we have 
\begin{equation}\label{Eq:pd}
\sum^n_{i,j=1}\alpha_i\alpha_jk(x_i,x_j)\ge 0.
\end{equation}
Furthermore, $k$ is said to be \emph{strictly pd} if, for mutually distinct $x_1,\ldots,x_n\in X$, equality in (\ref{Eq:pd}) only holds for $\alpha_1=\cdots=\alpha_n=0$. $\psi$ is said to be a positive definite function on $\bb{R}^d$ if $k(x,y)=\psi(x-y)$ is positive definite.\vspace{2mm}\par\noindent
\textbf{Integrally strictly positive definite:} Let $M$ be a topological space. A measurable and bounded kernel, $k$ is said to be integrally strictly positive definite if 
\begin{equation}\label{Eq:ispd}
\int\!\!\!\int_M k(x,y)\,d\mu(x)\,d\mu(y)>0
\end{equation}
for all finite non-zero signed Borel measures, $\mu$ defined on $M$.
\par The above definition is a generalization of \emph{integrally strictly positive definite functions} \cite[Section 6]{Stewart-76}: $\int\!\!\!\int_M k(x,y)f(x)f(y)\,dx\,dy>0$ for all $f\in L_2(\bb{R}^d)$, which is the strictly positive definiteness of the integral operator given by the kernel. Note that the above definition is \emph{not} equivalent to the definition of strictly pd kernels: if $k$ is integrally strictly pd, then it is strictly pd, while the converse is not true.\footnote{Suppose $k$ is not strictly pd. This means for all $n\in\bb{N}$ and for all mutually distinct $x_1,\ldots,x_n\in M$, there exists $\bb{R}\ni \alpha_j\ne 0$ for some $j\in \{1,\ldots, n\}$ such that $\sum^n_{j,l=1}\alpha_j\alpha_l k(x_j,x_l)=0$. By defining $\mu=\sum^n_{j=1}\alpha_j\delta_{x_j}$, it is easy to see that there exists $\mu\ne 0$ such that $\int\!\!\!\int_M k(x,y)\,d\mu(x)\,d\mu(y)=0$, which means $k$ is not integrally strictly pd. Therefore, if $k$ is integrally strictly pd, then it is strictly pd. However, the converse is not true. See \citet[Proposition 4.60, Theorem 4.62]{Steinwart-08} for an example.\label{footnote:ispd}} \vspace{2mm}\par\noindent
\textbf{Fourier transform in $\bb{R}^d$:} For $f\in L^1(\bb{R}^d)$, $\widehat{f}$ and $f^\vee$ represent the Fourier transform and inverse Fourier transform of $f$ respectively, defined as
\begin{eqnarray}
\widehat{f}(y)&\!\!\!:=\!\!\!&\frac{1}{(2\pi)^{d/2}}\int_{\bb{R}^d}e^{-iy^Tx} f(x)\,dx,\,\,y\in\bb{R}^d,\label{Eq:FT}\\
f^\vee(x)&\!\!\!:=\!\!\!&\frac{1}{(2\pi)^{d/2}}\int_{\bb{R}^d}e^{ix^Ty} f(y)\,dy,\,\,x\in\bb{R}^d.\label{Eq:IFT}
\end{eqnarray}\par\noindent
\textbf{Convolution:} If $f$ and $g$ are complex functions in $\bb{R}^d$, their convolution $f\ast g$ is defined by
\begin{equation}\label{Eq:conv}
(f\ast g)(x):=\int_{\bb{R}^d}f(y)g(x-y)\,dy,
\end{equation}
provided that the integral exists for almost all $x\in\bb{R}^d$, in the Lebesgue sense. Let $\mu$ be a finite Borel measure on $\bb{R}^d$ and $f$ be a bounded measurable function on $\bb{R}^d$. The convolution of $f$ and $\mu$, $f\ast\mu$, which is a bounded measurable function, is defined by
\begin{equation}\label{Eq:measureconvolve}
(f\ast\mu)(x):=\int_{\bb{R}^d}f(x-y)\,d\mu(y).
\end{equation}
\par\noindent
\textbf{Rapidly decaying functions, $\Scr{D}_d$ and $\Scr{S}_d$:} Let $\Scr{D}_d$ be the space of compactly supported infinitely differentiable functions on $\bb{R}^d$, i.e., $\Scr{D}_d=\{f\in C^\infty(\bb{R}^d)\,|\,\text{supp}(f)\,\,\text{is bounded}\}$, where $\text{supp}(f)=\text{cl}\left(\{x\in\bb{R}^d\,|\,f(x)\ne 0\}\right)$. A function $f:\bb{R}^d\rightarrow\bb{C}$ is said to decay rapidly, or be rapidly decreasing, if for all $N\in\bb{N}$,
\begin{equation}
\sup_{\Vert\alpha\Vert_1\le N}\sup_{x\in\bb{R}^d}(1+\Vert x\Vert^2_2)^N|(T_\alpha f)(x)|<\infty,
\end{equation}
where $\alpha=(\alpha_1,\ldots,\alpha_d)$ is an ordered $d$-tuple of non-negative $\alpha_j$, $\Vert\alpha\Vert_1=\sum^d_{j=1}\alpha_j$ and $T_\alpha=\left(\frac{1}{i}\frac{\partial}{\partial x_1}\right)^{\alpha_1}\cdots\left(\frac{1}{i}\frac{\partial}{\partial x_d}\right)^{\alpha_d}$. $\Scr{S}_d$, called the Schwartz class, denotes the vector space of rapidly decreasing functions. Note that $\Scr{D}_d\subset\Scr{S}_d$. It can be shown that for any $f\in\Scr{S}_d$, $\widehat{f}\in\Scr{S}_d$ and $f^\vee\in\Scr{S}_d$ (see \citet[Chapter 9]{Folland-99} and \citet[Chapter 6]{Rudin-91} for details).
\vspace{2mm}\par\noindent
\textbf{Distributions, tempered distributions, $\Scr{D}^\prime_d$ and $\Scr{S}^\prime_d$:} A linear functional on $\Scr{D}_d$ which is continuous  with respect to the Fr\'{e}chet topology \citep[see][Definition 6.3]{Rudin-91} is called a \emph{distribution} in $\bb{R}^d$. The space of all distributions in $\bb{R}^d$ is denoted by $\Scr{D}^\prime_d$. A linear continuous functional over the space $\Scr{S}_d$ is called a \emph{tempered distribution} and the space of all tempered distributions in $\bb{R}^d$ is denoted by $\Scr{S}^\prime_d$.
\vspace{2mm}\par\noindent
\textbf{Support of a distribution:} For an open set $U\subset\mathbb{R}^d$, $\mathscr{D}_d(U)$ denotes the subspace of $\mathscr{D}_d$ consisting of the functions with support contained in $U$. Suppose $D\in\Scr{D}^\prime_d$. If $U$ is an open set of $\bb{R}^d$ and if $D(\varphi)=0$ for every $\varphi\in\Scr{D}_d(U)$, then $D$ is said to \emph{vanish} or be \emph{null} in $U$. Let $W$ be the union of all open $U\subset\mathbb{R}^d$ in which $D$ vanishes. The complement of $W$ is the \emph{support} of $D$.
\par For complete details on distribution theory and Fourier transforms of distributions, we refer the reader to \citet[Chapter 9]{Folland-99} and \citet[Chapter 6]{Rudin-91}.

\section{Hilbert Space Embedding of Probability Measures}\label{Sec:MMD}
We previously mentioned that $\gamma_k$ is related to the theory of RKHS embedding of probability measures described in \citet{Gretton-06,Smola-07},
and originally introduced and studied  in the late 70's and early 80's (see \citet[Chapter~4]{Berlinet-04} and references therein for details). The following result shows how such embedding can be obtained through an alternative representation for $\gamma_k$.
\begin{theorem}\label{Theorem:MMD-II}
Let $\Scr{P}_k:=\{\bb{P}\in\Scr{P}:\int_M \sqrt{k(x,x)}\,d\bb{P}(x)<\infty\}$, where $k$ is measurable on $M$. 
Then for any $\bb{P},\bb{Q}\in\Scr{P}_k$, \begin{equation}
\gamma_k(\bb{P},\bb{Q})=\left\Vert\int_M k(\cdot,x)\,d\bb{P}(x)-\int_M k(\cdot,x)\,d\bb{Q}(x)\right\Vert_\eu{H}=:\Vert \bb{P}k-\bb{Q}k\Vert_{\eu{H}},\label{Eq:MMD-II}
\end{equation}
where $\eu{H}$ is the RKHS generated by $k$.
\end{theorem}
\begin{proof}
Let $T_\bb{P}:\eu{H}\rightarrow\bb{R}$ be the linear functional defined as $T_\bb{P}[f]:=\int_{M}f(x)\,d\bb{P}(x)$ with $\Vert T_\bb{P}\Vert:=\sup_{f\in\eu{H},f\ne 0}\frac{|T_\bb{P}[f]|}{\Vert f\Vert_{\eu{H}}}$. Consider 
\begin{equation}
|T_\bb{P}[f]|=\left|\int_{M}f\,d\bb{P}\right|\le\int_{M} |f(x)|\,d\bb{P}(x)=\int_{M}|\langle f,k(\cdot,x)\rangle_{\eu{H}}|\,d\bb{P}(x)
\le\int_M\sqrt{k(x,x)}\Vert f\Vert_\eu{H}\,d\bb{P}(x),\nonumber
\end{equation} 
which implies $\Vert T_{\bb{P}}\Vert<\infty,\,\forall\,\bb{P}\in\Scr{P}_k$, i.e., $T_\bb{P}$ is a bounded linear functional on $\eu{H}$. Therefore, by the Riesz representation theorem \citep[Theorem II.4]{Reed-72}, for each $\bb{P}\in\Scr{P}_k$, there exists a unique $\lambda_\bb{P}\in\eu{H}$ such that $T_\bb{P}[f]=\langle f,\lambda_\bb{P}\rangle_{\eu{H}},\,\forall\,f\in\eu{H}$. Let $f=k(\cdot,u)$ 
for some $u\in M$. Then, $T_\bb{P}[k(\cdot,u)]=\langle k(\cdot,u),\lambda_\bb{P}\rangle_{\eu{H}}=\lambda_\bb{P}(u)$, which implies $\lambda_\bb{P}=\int_{M}k(\cdot,x)\,d\bb{P}(x)=:\bb{P}k$. Therefore, with 
\begin{equation}
\left|\bb{P}f-\bb{Q}f\right|=\left|T_\bb{P}[f]-T_\bb{Q}[f]\right|=\left|\langle f,\lambda_\bb{P}\rangle_\eu{H}-\langle f,\lambda_\bb{Q}\rangle_\eu{H}\right|=\left|\langle f, \lambda_\bb{P}-\lambda_\bb{Q}\rangle_{\eu{H}}\right|,\nonumber
\end{equation}
we have 
\begin{equation}
\gamma_k(\bb{P},\bb{Q})=\sup_{\Vert f\Vert_{\eu{H}}\le 1}\left|\bb{P}f-\bb{Q}f\right|=\Vert\lambda_\bb{P}-\lambda_\bb{Q}\Vert_{\eu{H}}=\Vert \bb{P}k-\bb{Q}k \Vert_{\eu{H}}.\nonumber
\end{equation}
Note that this holds for any $\bb{P},\bb{Q}\in\Scr{P}_k$.
\end{proof}
Given a kernel, $k$, 
(\ref{Eq:MMD-II}) holds for all $\bb{P}\in\Scr{P}_k$. However, in practice, especially in statistical inference applications, it is not possible to check whether $\bb{P}\in\Scr{P}_k$ as $\bb{P}$ is not known. Therefore, one would prefer to have a kernel such that 
\begin{equation}\label{Eq:condition}
\int_M \sqrt{k(x,x)}\,d\bb{P}(x)<\infty,\,\forall\,\bb{P}\in\Scr{P}.
\end{equation} 
The following proposition shows that (\ref{Eq:condition}) is equivalent to the kernel being bounded. Therefore, combining Theorem~\ref{Theorem:MMD-II} and Proposition~\ref{proposition:condition} shows that if $k$ is measurable and bounded, then $\gamma_k(\bb{P},\bb{Q})=\Vert \bb{P}k-\bb{Q}k\Vert_\eu{H}$ for any $\bb{P},\bb{Q}\in\Scr{P}$.
\begin{proposition}\label{proposition:condition}
Let $f$ be a measurable function on $M$. Then $\int_M f(x)\,d\bb{P}(x)<\infty$ for all $\bb{P}\in\Scr{P}$ if and only if $f$ is bounded.
\end{proposition}
\begin{proof}
One direction is straightforward because if $f$ is bounded, then $\int_M f(x)\,d\bb{P}(x)<\infty$ for all $\bb{P}\in\Scr{P}$. Let us consider the other direction. Suppose $f$ is not bounded. Then there exists a sequence $\{x_n\}\subset M$ such that $f(x_n)\stackrel{n\rightarrow\infty}{\longrightarrow} \infty$. By taking a subsequence, if necessary, we can assume $f(x_n) > n^2$ for all $n$. Then, $A := \sum^\infty_{n=1} \frac{1}{f(x_n)}<\infty$. Define a probability measure $\bb{P}$ on $M$ by
$\bb{P}=\sum^\infty_{n=1} \frac{1}{A f(x_n)}\, \delta_{x_n}$, where $\delta_{x_n}$ is a Dirac measure at $x_n$. Then, $\int_M f(x)\, d\bb{P}(x) = \frac{1}{A} \sum^\infty_{n=1} \frac{f(x_n)}{f(x_n)} = \infty$, which means if $f$ is not bounded, then there exists a $\bb{P}\in\Scr{P}$ such that $\int_M f(x)\,d\bb{P}(x)=\infty$.
\end{proof}
The representation of $\gamma_k$ in (\ref{Eq:MMD-II}) yields the embedding, 
\begin{equation}\label{Eq:embedding}
\Pi:\Scr{P}\rightarrow\eu{H}\qquad \bb{P}\mapsto\int_M k(\cdot,x)\,d\bb{P}(x),
\end{equation} 
as proposed by \citet[Chapter 4, Section 1.1]{Berlinet-04} and \citet{Smola-07}. \citet{Berlinet-04} derived this embedding as a generalization of $\delta_x\mapsto k(\cdot,x)$ (see Chapter 4 of \citet{Berlinet-04} for details), while \cite{Gretton-06} arrived at this embedding by choosing $\eu{F}=\eu{F}_k$ in (\ref{Eq:MMD}). Since $\gamma_k(\bb{P},\bb{Q})=\Vert \Pi[\bb{P}]-\Pi[\bb{Q}]\Vert_\eu{H}$, the question ``When is $\gamma_k$ a metric on $\Scr{P}$?'' is equivalent to the question ``When is $\Pi$ injective?''. Addressing these questions is the central focus of the paper and is discussed in Section~\ref{Sec:mainresults}. 
\par Before proceeding further, we present some other, equivalent representations of $\gamma_k$ which will not only improve our understanding of $\gamma_k$, but also be helpful in its computation. First, note that by exploiting the reproducing property of $k$, $\gamma_k$ can be equivalently represented as
\begin{eqnarray}
\gamma^2_k(\bb{P},\bb{Q})&\!\!=\!\!& \left\Vert \int_M k(\cdot,x)\,d\bb{P}(x)-\int_M k(\cdot,x)\,d\bb{Q}(x)\right\Vert^2_\eu{H}\nonumber\\
&\!\!=\!\!& \left\langle\int_M k(\cdot,x)\,d\bb{P}(x)-\int_M k(\cdot,x)\,d\bb{Q}(x),\int_M k(\cdot,y)\,d\bb{P}(y)-\int_M k(\cdot,y)\,d\bb{Q}(y)\right\rangle_\eu{H}\nonumber\\
&\!\!=\!\!& \left\langle\int_M k(\cdot,x)\,d\bb{P}(x),\int_M k(\cdot,y)\,d\bb{P}(y)\right\rangle_\eu{H} \nonumber\\
&&\quad\quad+ \left\langle\int_M k(\cdot,x)\,d\bb{Q}(x),\int_M k(\cdot,y)\,d\bb{Q}(y)\right\rangle_\eu{H}\nonumber\\
&&\qquad\qquad\qquad -2\left\langle\int_M k(\cdot,x)\,d\bb{P}(x),\int_M k(\cdot,y)\,d\bb{Q}(y)\right\rangle_\eu{H}\nonumber\\
&\!\!\stackrel{(a)}{=}\!\!& \int\!\!\!\int_M k(x,y)\,d\bb{P}(x)\,d\bb{P}(y)+\int\!\!\!\int_M k(x,y)\,d\bb{Q}(x)\,d\bb{Q}(y)\nonumber\\
&&\quad\quad -2\int\!\!\!\int_M k(x,y)\,d\bb{P}(x)\,d\bb{Q}(y)\label{Eq:computeMMD}\\
\label{Eq:computeMMD-1}
&\!\!=\!\!&\int\!\!\!\int_M k(x,y)\,d(\bb{P}-\bb{Q})(x)\,d(\bb{P}-\bb{Q})(y),
\end{eqnarray}
where $(a)$ follows from the fact that $\int_M f(x)\,d\bb{P}(x)=\langle f,\int_M k(\cdot,x)\,d\bb{P}(x)\rangle_\eu{H}$ for all $f\in\eu{H}$, $\bb{P}\in\Scr{P}$ (see proof of Theorem~\ref{Theorem:MMD-II}), applied with $f=\int_M k(\cdot,y)\,d\bb{P}(y)$. As motivated in Section~\ref{Sec:Introduction}, $\gamma^2_k$ is a straightforward sum of expectations of  $k$, and can be computed easily, e.g., using (\ref{Eq:computeMMD}) either in  closed form or using numerical integration techniques, depending on the choice of $k$, $\bb{P}$ and $\bb{Q}$. It is easy to show that, if $k$ is a Gaussian kernel with $\bb{P}$ and $\bb{Q}$ being normal distributions on $\bb{R}^d$, then $\gamma_k$ can be computed in a closed form (see \citet[Section III-C]{Sriperumbudur-09} for examples). In the following corollary to Theorem~\ref{Theorem:MMD-II}, we prove three results which provide a nice interpretation for $\gamma_k$ when $M=\bb{R}^d$ and $k$ is translation invariant, i.e., $k(x,y)=\psi(x-y)$, where $\psi$ is a positive definite function. We provide a detailed explanation for Corollary~\ref{cor:L2distance} in Remark~\ref{rem:L2distance}. Before stating the results, we need a famous result due to Bochner, that characterizes $\psi$. We quote this result from \citet[Theorem 6.6]{Wendland-05}.
\begin{theorem}[Bochner]\label{Theorem:Bochner}
A continuous function $\psi:\bb{R}^d\rightarrow\bb{R}$ is positive definite 
if and only if it is the Fourier transform of a finite nonnegative Borel measure $\Lambda$ on $\bb{R}^d$, i.e.,
\begin{equation}\label{Eq:Bochner}
\psi(x)=\int_{\bb{R}^d}e^{-ix^T\omega}\,d\Lambda(\omega),\,\,x\in\bb{R}^d.
\end{equation}
\end{theorem}
\begin{corollary}[Different interpretations of $\gamma_k$]\label{cor:L2distance}
(i) Let $M=\bb{R}^d$ and $k(x,y)=\psi(x-y)$, where $\psi:M\rightarrow\bb{R}$ is a bounded, continuous positive definite function. Then for any $\bb{P},\bb{Q}\in\Scr{P}$,
\begin{equation}\label{Eq:L2distance}
\gamma_k(\bb{P},\bb{Q})=\sqrt{\int_{\bb{R}^d}\left|\phi_\bb{P}(\omega)-\phi_\bb{Q}(\omega)\right|^2\,d\Lambda(\omega)}=:\Vert\phi_\bb{P}-\phi_\bb{Q}\Vert_{L^2(\bb{R}^d,\Lambda)},
\end{equation}
where $\phi_\bb{P}$ and $\phi_\bb{Q}$ represent the characteristic functions of $\bb{P}$ and $\bb{Q}$ respectively.\vspace{1.5mm}\\
(ii) Suppose $\theta\in L^1(\bb{R}^d)$ is a continuous bounded positive definite function and $\int_{\bb{R}^d}\theta(x)\,dx=1$. Let $\psi(x):=\psi_t(x)=t^{-d}\theta(t^{-1}x)$. Assume that $p$ and $q$ are bounded uniformly continuous Radon-Nikodym derivatives of $\bb{P}$ and $\bb{Q}$ w.r.t. the Lebesgue measure, i.e., $d\bb{P}=p\,dx$ and $d\bb{Q}=q\,dx$. Then,
\begin{equation}\label{Eq:limiting}
\lim_{t\rightarrow 0}\gamma_k(\bb{P},\bb{Q})=\Vert p-q\Vert_{L^2(\bb{R}^d)}.
\end{equation}
In particular, if $|\theta(x)|\le C(1+\Vert x\Vert_2)^{-d-\varepsilon}$ for some $C,\,\varepsilon>0$, then (\ref{Eq:limiting}) holds for all bounded $p$ and $q$ (not necessarily uniformly continuous).\vspace{1.5mm}\\
(iii) Suppose $\psi\in L^1(\bb{R}^d)$ and $\sqrt{\widehat{\psi}}\in L^1(\bb{R}^d)$. Then,
\begin{equation}\label{Eq:noise}
\gamma_k(\bb{P},\bb{Q})=(2\pi)^{-d/4}\Vert \Phi\ast\bb{P}-\Phi\ast\bb{Q}\Vert_{L^2(\bb{R}^d)},
\end{equation}
where $\Phi:=\left(\sqrt{\widehat{\psi}}\right)^\vee$ and $d\Lambda=(2\pi)^{-d/2}\widehat{\psi}\,d\omega$. Here, $\Phi\ast\bb{P}$ represents the convolution of $\Phi$ and $\bb{P}$.
\end{corollary}
\begin{proof}
\emph{(i)} Let us consider (\ref{Eq:computeMMD-1}) with $k(x,y)=\psi(x-y)$. Then, we have
\begin{eqnarray}\label{Eq:fourier}
\gamma^2_k(\bb{P},\bb{Q})&\!\!\!=\!\!\!&\int\!\!\!\int_{\bb{R}^d} \psi(x-y)\,d(\bb{P}-\bb{Q})(x)\,d(\bb{P}-\bb{Q})(y)\nonumber\\
&\!\!\!\stackrel{(a)}{=}\!\!\!&\int\!\!\!\int\!\!\!\int_{\bb{R}^d} e^{-i(x-y)^T\omega}\,d\Lambda(\omega)\,d(\bb{P}-\bb{Q})(x)\,d(\bb{P}-\bb{Q})(y)\nonumber\\
&\!\!\!\stackrel{(b)}{=}\!\!\!&\int\!\!\!\int_{\bb{R}^d} e^{-ix^T\omega}\,d(\bb{P}-\bb{Q})(x)\int_{\bb{R}^d} e^{iy^T\omega}\,d(\bb{P}-\bb{Q})(y)\,d\Lambda(\omega)\nonumber\\
&\!\!\!=\!\!\!&\int_{\bb{R}^d}\left(\phi_\bb{P}(\omega)-\phi_\bb{Q}(\omega)\right)\left(\overline{\phi_\bb{P}(\omega)}-\overline{\phi_\bb{Q}(\omega)}\right)\,d\Lambda(\omega)\nonumber\\
&\!\!\!=\!\!\!&\int_{\bb{R}^d}\left|\phi_\bb{P}(\omega)-\phi_\bb{Q}(\omega)\right|^2\,d\Lambda(\omega)\nonumber,
\end{eqnarray}
where Bochner's theorem (Theorem~\ref{Theorem:Bochner}) is invoked in $(a)$, while Fubini's theorem \citep[Theorem 2.37]{Folland-99}
is invoked in $(b)$.\vspace{1.5mm}\\
\emph{(ii)} Consider (\ref{Eq:computeMMD}) with $k(x,y)=\psi_t(x-y)$,
\begin{eqnarray}
\hspace{-.1in}\gamma^2_k(\bb{P},\bb{Q})&\!\!\!\!=\!\!\!\!&\int\!\!\!\int_{\bb{R}^d}\psi_t(x-y)p(x)p(y)\,dx\,dy
+\int\!\!\!\int_{\bb{R}^d}\psi_t(x-y)q(x)q(y)\,dx\,dy\nonumber\\
&&\qquad\qquad-2\int\!\!\!\int_{\bb{R}^d}\psi_t(x-y)p(x)q(y)\,dx\,dy\nonumber\\
&\!\!\!\!=\!\!\!\!&\int_{\bb{R}^d}(\psi_t\ast p)(x)p(x)\,dx
+\int_{\bb{R}^d}(\psi_t\ast q)(x)q(x)\,dx
-2\int_{\bb{R}^d}(\psi_t\ast q)(x)p(x)\,dx.\label{Eq:chain}
\end{eqnarray}
Note that $\lim_{t\rightarrow 0} \int_{\bb{R}^d}(\psi_t\ast p)(x)p(x)\,dx=\int_{\bb{R}^d}\lim_{t\rightarrow 0}(\psi_t\ast p)(x) p(x)\,dx$, by invoking the dominated convergence theorem. Since $p$ is bounded and uniformly continuous, by Theorem~\ref{Thm:Folland-1} (see Appendix A), we have $p\ast\psi_t\rightarrow p$ uniformly as $t\rightarrow 0$, which means $\lim_{t\rightarrow 0} \int_{\bb{R}^d}(\psi_t\ast p)(x)p(x)\,dx=\int_{\bb{R}^d}p^2(x)\,dx$. Using this in (\ref{Eq:chain}), we have
\begin{equation}
\lim_{t\rightarrow 0}\gamma^2_k(\bb{P},\bb{Q})=\int_{\bb{R}^d}(p^2(x)+q^2(x)-2p(x)q(x))\,dx=\Vert p-q\Vert^2_{L^2(\bb{R}^d)}.\nonumber
\end{equation}
Suppose $|\theta(x)|\le(1+\Vert x\Vert_2)^{-d-\varepsilon}$ for some $C,\,\varepsilon>0$. Since $p\in L^1(\bb{R}^d)$, by Theorem~\ref{Thm:Folland-2} (see Appendix A), we have $(p\ast\psi_t)(x)\rightarrow p(x)$ as $t\rightarrow 0$ for almost every $x$. Therefore $\lim_{t\rightarrow 0} \int_{\bb{R}^d}(\psi_t\ast p)(x)p(x)\,dx=\int_{\bb{R}^d}p^2(x)\,dx$ and the result follows.\vspace{1.5mm}\\
\emph{(iii)} Since $\psi$ is positive definite, $\widehat{\psi}$ is nonnegative and therefore $\sqrt{\widehat{\psi}}$ is valid. Since $\sqrt{\widehat{\psi}}\in L^1(\bb{R}^d)$, $\Phi$ exists.
Define $\phi_{\bb{P},\bb{Q}}:=\phi_\bb{P}-\phi_\bb{Q}$. Now, consider
\begin{eqnarray}
\Vert\Phi\ast\bb{P}-\Phi\ast\bb{Q}\Vert^2_{L^2(\bb{R}^d)}&\!\!\!=\!\!\!&\int_{\bb{R}^d}\left|(\Phi\ast(\bb{P}-\bb{Q}))(x)\right|^2\,dx\nonumber\\
&\!\!\!=\!\!\!&\int_{\bb{R}^d}\left|\int_{\bb{R}^d}\Phi(x-y)\,d(\bb{P}-\bb{Q})(y)\right|^2\,dx\nonumber\\
&\!\!\!=\!\!\!&\frac{1}{(2\pi)^d}\int_{\bb{R}^d}\left|\int\!\!\!\int_{\bb{R}^d}\sqrt{\widehat{\psi}(\omega)}\,e^{i(x-y)^T\omega}\,d\omega\,\,d(\bb{P}-\bb{Q})(y)\right|^2\,dx\nonumber\\
&\!\!\!\stackrel{(c)}{=}\!\!\!&\frac{1}{(2\pi)^d}\int_{\bb{R}^d}\left|\int_{\bb{R}^d}\sqrt{\widehat{\psi}(\omega)}(\overline{\phi_\bb{P}(\omega)}-\overline{\phi_\bb{Q}(\omega)})\,e^{ix^T\omega}\,d\omega\right|^2\,dx\nonumber\\
&\!\!\!=\!\!\!&\frac{1}{(2\pi)^d}\int\!\!\!\int\!\!\!\int_{\bb{R}^d}\sqrt{\widehat{\psi}(\omega)}\sqrt{\widehat{\psi}(\xi)}\,\overline{\phi_{\bb{P},\bb{Q}}(\omega)}\,\phi_{\bb{P},\bb{Q}}(\xi)\,e^{i(\omega-\xi)^Tx}\,d\omega\,d\xi\,dx\nonumber\\
&\!\!\!\stackrel{(d)}{=}\!\!\!&\int\!\!\!\int_{\bb{R}^d}\sqrt{\widehat{\psi}(\omega)}\sqrt{\widehat{\psi}(\xi)}\,\overline{\phi_{\bb{P},\bb{Q}}(\omega)}\,\phi_{\bb{P},\bb{Q}}(\xi)\left[\frac{1}{(2\pi)^d}\int_{\bb{R}^d}e^{i(\omega-\xi)^Tx}\,dx\right]\,d\omega\,d\xi\nonumber\\
&\!\!\!=\!\!\!&\int\!\!\!\int_{\bb{R}^d}\sqrt{\widehat{\psi}(\omega)}\sqrt{\widehat{\psi}(\xi)}\,\overline{\phi_{\bb{P},\bb{Q}}(\omega)}\,\phi_{\bb{P},\bb{Q}}(\xi)\,\delta(\omega-\xi)\,d\omega\,d\xi\nonumber\\
&\!\!\!=\!\!\!&\int_{\bb{R}^d}\widehat{\psi}(\omega)\left|\phi_\bb{P}(\omega)-\phi_\bb{Q}(\omega)\right|^2\,d\omega\nonumber\\
&\!\!\!=\!\!\!&(2\pi)^{d/2}\gamma^2_k(\bb{P},\bb{Q}),\nonumber
\end{eqnarray}
where $(c)$ and $(d)$ are obtained by invoking Fubini's theorem.\vspace{-6mm}
\end{proof}
\begin{remark}\label{rem:L2distance}
(a) (\ref{Eq:L2distance}) shows that $\gamma_k$ is the $L^2$-distance between the characteristic functions of $\bb{P}$ and $\bb{Q}$ computed w.r.t. the non-negative finite Borel measure, $\Lambda$, which is the Fourier transform of $\psi$. If $\psi\in L^1(\bb{R}^d)$, then (\ref{Eq:L2distance}) is a rephrase of the well known fact \citep[Theorem 10.12]{Wendland-05}: for any $f\in\eu{H}$, 
\begin{equation}\label{Eq:rkhsnorm}
\Vert f\Vert^2_\eu{H}=\int_{\bb{R}^d} \frac{|\widehat{f}(\omega)|^2}{\widehat{\psi}(\omega)}\,d\omega.
\end{equation}
Choosing $f=(\bb{P}-\bb{Q})\ast\psi$ in (\ref{Eq:rkhsnorm}) yields $\widehat{f}=(\phi_\bb{P}-\phi_\bb{Q})\widehat{\psi}$ and therefore the result in (\ref{Eq:L2distance}).\vspace{1.5mm}\\
(b) Suppose $d\Lambda(\omega)=(2\pi)^{-d}\,d\omega$. Assume $\bb{P}$ and $\bb{Q}$ have $p$ and $q$ as Radon-Nikodym derivatives w.r.t. the Lebesgue measure, i.e., $d\bb{P}=p\,dx$ and $d\bb{Q}=q\,dx$. Using these in (\ref{Eq:L2distance}), it can be shown that $\gamma_k(\bb{P},\bb{Q})=\Vert p-q\Vert_{L^2(\bb{R}^d)}$. However, this result should be interpreted in a limiting sense as mentioned in Corollary~\ref{cor:L2distance}(ii) because
the choice of $d\Lambda(\omega)=(2\pi)^{-d}\,d\omega$ implies $\psi(x)=\delta(x)$, which does not satisfy the conditions of  Corollary~\ref{cor:L2distance}(i). It can be shown that $\psi(x)=\delta(x)$ is obtained in a limiting sense \citep[Proposition 9.1]{Folland-99}: $\psi_t\rightarrow \delta$ in $\Scr{D}^\prime_d$ as $t\rightarrow 0$.
\vspace{1.5mm}\\
(c) Choosing $\theta(x)=(2\pi)^{-d/2}e^{-\Vert x\Vert^2_2/2}$ in Corollary~\ref{cor:L2distance}(ii) corresponds to $\psi_t$ being a Gaussian kernel (with appropriate normalization such that $\int_{\bb{R}^d}\psi_t(x)\,dx=1$). Therefore, (\ref{Eq:limiting}) shows that as the bandwidth, $t$ of the Gaussian kernel approaches zero, $\gamma_k$ approaches the $L^2$-distance between the densities $p$ and $q$. The same result also holds for choosing $\psi_t$ as the Laplacian kernel, $B_{2n+1}$-spline, inverse multiquadratic, etc. Therefore, $\gamma_k(\bb{P},\bb{Q})$ can be seen as a generalization of the $L^2$-distance between probability measures, $\bb{P}$ and $\bb{Q}$.\vspace{1.5mm}\\
(d) The result in (\ref{Eq:limiting}) holds if $p$ and $q$ are bounded and uniformly continuous. Since any condition on $\bb{P}$ and $\bb{Q}$ is usually difficult to check in statistical applications, it is better to impose conditions on $\psi$ rather than on $\bb{P}$ and $\bb{Q}$. In Corollary~\ref{cor:L2distance}(ii), by imposing additional conditions on $\psi_t$, the result in (\ref{Eq:limiting}) is shown to hold for all $\bb{P}$ and $\bb{Q}$ with bounded densities $p$ and $q$. The condition, $|\theta(x)|\le C(1+\Vert x\Vert_2)^{-d-\varepsilon}$ for some $C,\,\varepsilon>0$, is, e.g., satisfied by the inverse multiquadratic kernel, $\theta(x)=\widetilde{C}(1+\Vert x\Vert^2_2)^{-\tau},\,x\in\bb{R}^d,\,\tau>d/2$, where $\widetilde{C}=\left(\int_{\bb{R}^d}(1+\Vert x\Vert^2_2)^{-\tau}\,dx\right)^{-1}$.\vspace{1.5mm}\\
(e) The result in Corollary~\ref{cor:L2distance}(ii) has connections to the kernel density estimation in $L^2$-sense using Parzen windows \citep{Rosenblatt-75}, where $\psi$ can be chosen as the Parzen window.
\vspace{1.5mm}\\
(f) (\ref{Eq:noise}) shows that $\gamma_k$ is proportional to the $L^2$-distance between $\Phi\ast\bb{P}$ and $\Phi\ast\bb{Q}$. Let $\Phi$ be such that $\Phi$ is nonnegative and $\Phi\in L^1(\bb{R}^d)$. Then, defining $\widetilde{\Phi}:=\left(\int_{\bb{R}^d}\Phi(x)\,dx\right)^{-1}\Phi=\Phi/\sqrt{\widehat{\psi}(0)}=\left(\int_{\bb{R}^d}\psi(x)\,dx\right)^{-1/2}\Phi$ and using this in (\ref{Eq:noise}), we have
\begin{equation}\label{Eq:noisemodified}
\gamma_k(\bb{P},\bb{Q})=(2\pi)^{-d/4}\sqrt{\widehat{\psi}(0)}\left\Vert\widetilde{\Phi}\ast\bb{P}-\widetilde{\Phi}\ast\bb{Q}\right\Vert_{L^2(\bb{R}^d)}.
\end{equation}
The r.h.s. of (\ref{Eq:noisemodified}) can be interpreted as follows. Let $X$, $Y$ and $N$ be independent random variables such that $X\sim\bb{P}$, $Y\sim\bb{Q}$ and $N\sim\widetilde{\Phi}$. This means $\gamma_k$ is proportional to the $L^2$-distance computed between the densities associated with the perturbed random variables, $X+N$ and $Y+N$. Note that $\Vert p-q\Vert_{L^2(\bb{R}^d)}$ is the $L^2$-distance between the densities of $X$ and $Y$. Examples of $\psi$ that satisfy the conditions in Corollary~\ref{cor:L2distance}(iii) in addition to the conditions on $\Phi$ as mentioned here include the Gaussian and Laplacian kernels on $\bb{R}^d$. The result in (\ref{Eq:noise}) holds even if $\sqrt{\widehat{\psi}}\notin L^1(\bb{R}^d)$ as the proof of (iii) can be handled using distribution theory. However, we assumed $\sqrt{\widehat{\psi}}\in L^1(\bb{R}^d)$ to keep the proof simple, without delving into distribution theory.
\end{remark}
Although we will not be using all the results of Corollary~\ref{cor:L2distance} in deriving our main results in the following sections, Corollary~\ref{cor:L2distance} was presented to provide a better intuitive understanding of $\gamma_k$. To summarize, the core results of this section are Theorem~\ref{Theorem:MMD-II} (combined with Proposition~\ref{proposition:condition}), which provides a closed form expression for $\gamma_k$ in terms of the measurable and bounded $k$, and Corollary~\ref{cor:L2distance}\emph{(i)}, which provides an alternative representation for $\gamma_k$ when $k$ is bounded, continuous and translation invariant on $\bb{R}^d$.

\section{Conditions for Characteristic Kernels}\label{Sec:mainresults}
In this section, we address the question ``When is $\gamma_k$ a metric on $\Scr{P}$?". In other words, ``When is $\Pi$ injective?" or ``Under what conditions is $k$ characteristic?". To this end, we start with the definition of characteristic kernels and provide some examples where $k$ is such that $\gamma_k$ is not a metric on $\Scr{P}$. As discussed in Section~\ref{subsubsec:contribution1}, although some characterizations are available for $k$ so that $\gamma_k$ is a metric on $\Scr{P}$, they are difficult to check in practice. So, in Section~\ref{subsec:strictpd}, we provide the characterization that if $k$ is integrally strictly pd, then $\gamma_k$ is a metric on $\Scr{P}$. In Section~\ref{subsec:Rd}, we present more easily checkable conditions wherein we show that if $\text{supp}(\Lambda)=\bb{R}^d$ (see footnote~\ref{footnote:Radon} for the definition of the support of a Borel measure), then $\gamma_k$ is a metric on $\Scr{P}$. This result is extended in a straightforward way to $\bb{T}^d$ ($d$-Torus) in Section~\ref{subsec:Td}. The main results of this section are summarized in Table~\ref{tab:summary}.
\begin{table}[t]
\begin{center}
\begin{tabular}{|ccccc|}\hline
& & & &\\
\multicolumn{5}{|c|}{\bfseries Summary of Main Results}\\
& & & &\\\hline\hline
Domain & Property & $\Scr{Q}$ & Characteristic & Reference \\
\hline\hline
& & & &\\
$M$ & $k$ is integrally strictly pd & $\Scr{P}$ & Yes & Theorem~\ref{thm:strictpd} \\
& & & &\\\hline\hline
& & & &\\
$\bb{R}^d$ & $\Omega=\bb{R}^d$ & $\Scr{P}$ & Yes & Theorem~\ref{Thm:countablyinfinite}\\
& & & &\\
$\bb{R}^d$ & $\text{supp}(\psi)$ is compact & $\Scr{P}$ & Yes & Corollary~\ref{cor:compactsupport}\\
& & & &\\
$\bb{R}^d$ & $\Omega\subsetneq\bb{R}^d$, $\text{int}(\Omega)\neq\emptyset$  & $\Scr{P}_1$ & Yes & Theorem~\ref{Thm:compact}\\
& & & &\\
$\bb{R}^d$ & $\Omega\subsetneq\bb{R}^d$ & $\Scr{P}$ & No & Theorem~\ref{Thm:countablyinfinite}\\
& & & &\\
\hline\hline
& & & &\\
$\bb{T}^d$ & $A_\psi(0)\ge 0,\,A_\psi(n)>0,\,\forall\,n\ne 0$ & $\Scr{P}$ & Yes & Theorem~\ref{Thm:torus}\\
& & & &\\
$\bb{T}^d$ & $\exists\,n\ne 0\,|\,A_\psi(n)=0$ & $\Scr{P}$ & No & Theorem~\ref{Thm:torus}\\
& & & &\\
\hline
\end{tabular}
\caption{The table should be read as: If ``Property" is satisfied on ``Domain", then $k$ is characteristic (or not) to $\Scr{Q}$. $\Scr{P}$ is the set of all Borel probability measures on a topological space, $M$. See Section~\ref{Sec:Notation} for the definition of integrally strictly pd kernels. When $M=\bb{R}^d$, $k(x,y)=\psi(x-y)$, where $\psi$ is a bounded, continuous positive definite function on $\bb{R}^d$. $\psi$ is the Fourier transform of a finite nonnegative Borel measure, $\Lambda$, and $\Omega:=\text{supp}(\Lambda)$ (see Theorem~\ref{Theorem:Bochner} and footnote~\ref{footnote:Radon} for details). $\Scr{P}_1:=\{\bb{P}\in\Scr{P}:\phi_\bb{P}\in L^1(\bb{R}^d)\cup L^2(\bb{R}^d),\,\bb{P}\ll\lambda\,\,\text{and supp}(\bb{P})\,\,\text{is compact}\}$, where $\phi_\bb{P}$ is the characteristic function of $\bb{P}$ and $\lambda$ is the Lebesgue measure. $\bb{P}\ll\lambda$ denotes that $\bb{P}$ is absolutely continuous w.r.t. $\lambda$. When $M=\bb{T}^d$, $k(x,y)=\psi(x-y)$, where $\psi$ is a bounded, continuous positive definite function on $\bb{T}^d$. $\{A_\psi(n)\}^\infty_{n=-\infty}$ are the Fourier series coefficients of $\psi$ which are nonnegative and summable (see Theorem~\ref{Theorem:Bochnerdiscrete} for details).\label{tab:summary}}
\end{center}
\vspace{-7.5mm}
\end{table}
\par We start by defining characteristic kernels.
\begin{definition}[Characteristic kernel]\label{def:characterizing}
A bounded measurable positive definite kernel $k$ is characteristic to a set $\Scr{Q}\subset\Scr{P}$ of probability measures defined on $(M,\eu{A})$ if for $\bb{P},\bb{Q}\in\Scr{Q}$, $\gamma_k(\bb{P},\bb{Q})=0\Leftrightarrow \bb{P}=\bb{Q}$. $k$ is simply said to be characteristic if it is characteristic to $\Scr{P}$. The RKHS, $\eu{H}$ induced by such a $k$ is called a characteristic RKHS.
\end{definition}
As mentioned before, the injectivity of $\Pi$ is related to the characteristic property of $k$. If $k$ is characteristic, then $\gamma_k(\bb{P},\bb{Q})=\Vert \Pi[\bb{P}]-\Pi[\bb{Q}]\Vert_\eu{H}=0\Rightarrow \bb{P}=\bb{Q}$, which means $\bb{P}\mapsto\int_M k(\cdot,x)\,d\bb{P}(x)$, i.e., $\Pi$ is injective. Therefore, when $M=\bb{R}^d$, the embedding of a distribution to a characteristic RKHS 
can be seen as a generalization of the characteristic function, $\phi_\bb{P}=\int_{\bb{R}^d}e^{i\langle\cdot,x\rangle}\,d\bb{P}(x)$. This is because, by the uniqueness theorem for characteristic functions \citep[Theorem 9.5.1]{Dudley-02}, $\phi_\bb{P}=\phi_\bb{Q}\Rightarrow\bb{P}=\bb{Q}$, which means $\bb{P}\mapsto \int_{\bb{R}^d}e^{i\langle\cdot,x\rangle}\,d\bb{P}(x)$ is injective. So, in this context, intuitively $e^{i\langle y,x\rangle}$ can be treated as the characteristic kernel, $k$, although, formally, this is not true as $e^{i\langle y,x\rangle}$ is not a pd kernel.
\par Before we get to the characterization of characteristic kernels, the following examples show that there exist bounded measurable kernels that are not characteristic.
\begin{example}[Trivial kernel]\label{Exm:trivial}
Let $k(x,y)=\psi(x-y)=C$, $\forall\,x,y\in\bb{R}^d$ with $C>0$. Using this in (\ref{Eq:computeMMD}), we have $\gamma^2_k(\bb{P},\bb{Q})=C+C-2C=0$ for any $\bb{P},\bb{Q}\in\Scr{P}$, which means $k$ is not characteristic.\vspace{-2mm}
\end{example} 
\begin{example}[Dot product kernel]\label{Exm:dotproduct}
Let $k(x,y)= x^Ty$, $x,y\in\bb{R}^d$. Using this in (\ref{Eq:computeMMD}), we have \begin{equation}
\gamma^2_k(\bb{P},\bb{Q})=\mu^T_\bb{P}\mu_\bb{P}+\mu^T_\bb{Q}\mu_\bb{Q}-2\mu^T_\bb{P}\mu_\bb{Q}=\Vert \mu_\bb{P}-\mu_\bb{Q}\Vert^2_2,\nonumber
\end{equation}
where $\mu_\bb{P}$ and $\mu_\bb{Q}$ represent the means associated with $\bb{P}$ and $\bb{Q}$ respectively, i.e., $\mu_\bb{P}:=\int_{\bb{R}^d}x\,d\bb{P}(x)$. It is clear that $k$ is not characteristic as $\gamma_k(\bb{P},\bb{Q})=0\Rightarrow \mu_\bb{P}=\mu_\bb{Q}\nRightarrow\bb{P}=\bb{Q}$ for all $\bb{P},\bb{Q}\in\Scr{P}$.
\end{example}
\begin{example}[Polynomial kernel of order 2]\label{Exm:polynomial}
Let $k(x,y)=(1+x^Ty)^2$, $x,y\in\bb{R}^d$. Using this in (\ref{Eq:computeMMD-1}), we have \begin{eqnarray}
\gamma^2_k(\bb{P},\bb{Q})&\!\!\!=\!\!\!&\int\!\!\!\int_{\bb{R}^d}(1+2x^Ty+x^Tyy^Tx)\,d(\bb{P}-\bb{Q})(x)\,d(\bb{P}-\bb{Q})(y)\nonumber\\
&\!\!\!=\!\!\!& 2\Vert\mu_\bb{P}-\mu_\bb{Q}\Vert^2_2+\Vert\Sigma_\bb{P}-\Sigma_\bb{Q}+\mu_\bb{P}\mu^T_\bb{P}-\mu_\bb{Q}\mu^T_\bb{Q}\Vert^2_F,\nonumber
\end{eqnarray}
where $\Sigma_\bb{P}$ and $\Sigma_\bb{Q}$ represent the covariance matrices associated with $\bb{P}$ and $\bb{Q}$ respectively, i.e., $\Sigma_\bb{P}:=\int_{\bb{R}^d}xx^T\,d\bb{P}(x)-\mu_\bb{P}\mu^T_\bb{P}$. $\Vert \cdot\Vert_F$ represents the Frobenius norm. Since $\gamma_k(\bb{P},\bb{Q})=0\Rightarrow (\mu_\bb{P}=\mu_\bb{Q}\,\,\text{and}\,\,\Sigma_\bb{P}=\Sigma_\bb{Q})\nRightarrow\bb{P}=\bb{Q}$ for all $\bb{P},\bb{Q}\in\Scr{P}$, $k$ is not characteristic.
\end{example}
In the following sections, we address the question of when $k$ is characteristic, i.e., for what $k$ is $\gamma_k$ a metric on $\Scr{P}$?

\subsection{Integrally strictly positive definite kernels are characteristic}\label{subsec:strictpd}
Compared to the existing characterizations in literature \citep{Gretton-06, Fukumizu-08a, Fukumizu-09}, the following result provides a more natural and easily understandable characterization for characteristic kernels, which shows that integrally strictly pd kernels are characteristic to $\Scr{P}$. 
\begin{theorem}[Integrally strictly pd kernels are characteristic]\label{thm:strictpd}
If $k$ is integrally strictly positive definite on a topological space, $M$, then $k$ is characteristic to $\Scr{P}$.
\end{theorem}
Before proving Theorem~\ref{thm:strictpd}, we provide a supplementary result in Lemma~\ref{lemma:strictpd} that provides necessary and sufficient conditions for a kernel \emph{not} to be  characteristic. We show that choosing $k$ to be integrally strictly pd violates the conditions in Lemma~\ref{lemma:strictpd}, and $k$ is therefore characteristic to $\Scr{P}$.
\begin{lemma}\label{lemma:strictpd}
Let $k$ be measurable and bounded on a topological space, $M$. Then $\exists\,\bb{P}\ne \bb{Q},\,\bb{P},\bb{Q}\in\Scr{P}$ such that $\gamma_{k}(\bb{P},\bb{Q})=0$ if and only if there exists a finite non-zero signed Borel measure $\mu$ that satisfies:
\begin{itemize}
\item[(i)] $\int\!\!\!\int_Mk(x,y)\,d\mu(x)\,d\mu(y)=0$,
\item[(ii)] $\mu(M)=0$.
\end{itemize}
\end{lemma}
\begin{proof}
($\,\Leftarrow\,$) Suppose there exists a finite non-zero signed Borel measure, $\mu$ that satisfies \emph{(i)} and \emph{(ii)} in Lemma~\ref{lemma:strictpd}. By the Jordan decomposition theorem \cite[Theorem 5.6.1]{Dudley-02}, there exist unique positive measures $\mu^+$ and $\mu^-$ such that $\mu=\mu^+-\mu^-$ and $\mu^+\perp \mu^-$ ($\mu^+$ and $\mu^-$ are singular). By \emph{(ii)}, we have $\mu^+(M)=\mu^-(M)=:\alpha$. Define $\bb{P}=\alpha^{-1}\mu^+$ and $\bb{Q}=\alpha^{-1}\mu^-$. Clearly, $\bb{P}\ne \bb{Q},\,\bb{P},\bb{Q}\in\Scr{P}$. Then, by (\ref{Eq:computeMMD-1}), we have
\begin{equation}
\gamma^2_{k}(\bb{P},\bb{Q})=\int\!\!\!\int_{M}k(x,y)\,d(\bb{P}-\bb{Q})(x)\,d(\bb{P}-\bb{Q})(y)=\alpha^{-2}\int\!\!\!\int_{M}k(x,y)\,d\mu(x)\,d\mu(y)\stackrel{(a)}{=}0,\nonumber
\end{equation}
where \emph{(a)} is obtained by invoking \emph{(i)}. So, we have constructed $\bb{P}\ne\bb{Q}$ such that $\gamma_k(\bb{P},\bb{Q})=0$.\vspace{1mm}\\
($\,\Rightarrow\,$) Suppose $\exists\,\bb{P}\ne \bb{Q},\,\bb{P},\bb{Q}\in\Scr{P}$ such that $\gamma_{k}(\bb{P},\bb{Q})=0$. Let $\mu=\bb{P}-\bb{Q}$. Clearly $\mu$ is a finite non-zero signed Borel measure that satisfies $\mu(M)=0$. Note that by (\ref{Eq:computeMMD-1}), 
\begin{equation}
\gamma^2_k(\bb{P},\bb{Q})=\int\!\!\!\int_M k(x,y)\,d(\bb{P}-\bb{Q})(x)\,d(\bb{P}-\bb{Q})(y)=\int\!\!\!\int_M k(x,y)\,d\mu(x)\,d\mu(y),\nonumber
\end{equation} and therefore \emph{(i)} follows.\vspace{-2mm}
\end{proof}
\begin{proof}
\hspace{-.05in}\textbf{(of Theorem~\ref{thm:strictpd})} Since $k$ is integrally strictly pd on $M$, we have 
\begin{equation}
\int\!\!\!\int_M k(x,y)\,d\eta(x)d\eta(y)>0,\nonumber
\end{equation}
for any finite non-zero signed Borel measure $\eta$. This means there does not exist a finite non-zero signed Borel measure that satisfies \emph{(i)} in Lemma~\ref{lemma:strictpd}. Therefore, by Lemma~\ref{lemma:strictpd}, there does not exist $\bb{P}\ne\bb{Q},\,\bb{P},\bb{Q}\in\Scr{P}$ such that $\gamma_{k}(\bb{P},\bb{Q})=0$, which implies $k$ is characteristic.
\end{proof}
Examples of integrally strictly pd kernels on $\bb{R}^d$ include the Gaussian, $\exp(-\sigma\Vert x-y\Vert^2_2),\,\sigma>0$; the Laplacian, $\exp(-\sigma\Vert x-y\Vert_1),\,\sigma>0$; inverse multiquadratics, $(\sigma^2+\Vert x-y\Vert^2_2)^{-c},\,c>0,\,\sigma>0$, etc, 
which are translation invariant kernels on $\bb{R}^d$. A \emph{translation variant} integrally strictly pd kernel, $\widetilde{k}$, can be obtained from a translation invariant integrally strictly pd kernel, $k$, as $\widetilde{k}(x,y)=f(x)k(x,y)f(y)$, where $f:M\rightarrow\bb{R}$ is a bounded continuous function. 
A simple example of a translation variant integrally strictly pd kernel on $\bb{R}^d$ is $\widetilde{k}(x,y)=\exp(\sigma x^Ty),\,\sigma>0$, where we have chosen $f(.)=\exp(\sigma\Vert .\Vert^2_2/2)$ and $k(x,y)=\exp(-\sigma\Vert x-y\Vert^2_2/2),\,\sigma>0$. Clearly, this kernel is characteristic on compact subsets of $\bb{R}^d$. The same result can also be obtained from the fact that $\widetilde{k}$ is universal on compact subsets of $\bb{R}^d$ \citep[Section 3, Example 1]{Steinwart-01}.
\par Although the condition for characteristic $k$ in Theorem~\ref{thm:strictpd} is easy to understand compared to other characterizations in literature, it is 
not always easy to check for integral strict positive definiteness of $k$. In the following section, we assume $M=\bb{R}^d$ and $k$ to be translation invariant and present a complete characterization for characteristic $k$ which is simple to check.

\subsection{Characterization for translation invariant $k$ on $\bb{R}^d$}\label{subsec:Rd}
The complete, detailed proofs of the main results in this section are provided in Section~\ref{subsubsec:proofs}. Compared to \citet{Sriperumbudur-08}, we now present simple proofs for these results without resorting to distribution theory. Let us start with the following assumption.
\begin{assumption}\label{assume-1}
$k(x,y)=\psi(x-y)$ where $\psi$ is a bounded continuous real-valued positive definite function on $M=\bb{R}^d$.
\end{assumption}
The following theorem characterizes all translation invariant kernels in $\bb{R}^d$ that are characteristic.
\begin{theorem}\label{Thm:countablyinfinite}
Suppose $k$ satisfies Assumption~\ref{assume-1}. Then $k$ is characteristic if and only if $\emph{supp}(\Lambda)=\bb{R}^d$, where $\Lambda$ is defined as in (\ref{Eq:Bochner}).\footnote{For a finite regular measure $\mu$, there is a largest open set $U$ with $\mu(U)=0$. The complement of $U$ is called the \emph{support} of $\mu$, denoted by $\text{supp}(\mu)$.\label{footnote:Radon}}
\end{theorem}
First, note that the condition $\text{supp}(\Lambda)=\bb{R}^d$ is easy to check compared to all other, aforementioned characterizations for characteristic $k$. Table~\ref{tab:kernel} shows some popular translation invariant kernels on $\bb{R}$ along with their Fourier spectra, $\widehat{\psi}$ and its support: Gaussian, Laplacian, $B_{2n+1}$-spline\footnote{A $B_{2n+1}$-spline is a $B_n$-spline of odd order. Only $B_{2n+1}$-splines are admissible, i.e. $B_n$ splines of odd order are positive definite kernels whereas the ones of even order have negative components in their Fourier spectrum, $\widehat{\psi}$ and, therefore, are not admissible kernels. In Table~\ref{tab:kernel}, the symbol $\ast^{(2n+2)}_1$ represents the $(2n+2)$-fold convolution. An important point to be noted with the $B_{2n+1}$-spline kernel is that its Fourier spectrum, $\widehat{\psi}$ has vanishing points at $\omega=2\pi\alpha,\,\alpha\in\bb{Z}\backslash\{0\}$ unlike Gaussian and Laplacian kernels which do not have any vanishing points in their Fourier spectrum. Nevertheless, the spectrum of all these kernels has support $\bb{R}$.} \citep{Scholkopf-02} and Sinc kernels are aperiodic while Poisson \citep{Bremaud-01, Steinwart-01, Vapnik-98}, Dirichlet \citep{Bremaud-01, Scholkopf-02}, F\'{e}jer \citep{Bremaud-01} and cosine kernels are periodic. Although the Gaussian and Laplacian kernels are shown to be characteristic by all the characterizations we have mentioned so far, the case of $B_{2n+1}$-splines is addressed only by Theorem~\ref{Thm:countablyinfinite}, which shows them to be characteristic (note that $B_{2n+1}$-splines being integrally strictly pd also follow from Theorem~\ref{Thm:countablyinfinite}). In fact, one can provide a more general result on compactly supported translation invariant kernels, which we do later in Corollary~\ref{cor:compactsupport}. The Mat\'{e}rn class of kernels \citep[Section 4.2.1]{Rasmussen-06}, given by
\begin{equation}\label{Eq:matern}
k(x,y)=\psi(x-y)=\frac{2^{1-\nu}}{\Gamma(\nu)}\left(\frac{\sqrt{2\nu}\Vert x-y\Vert_2}{\sigma}\right)^\nu K_\nu\left(\frac{\sqrt{2\nu}\Vert x-y\Vert_2}{\sigma}\right),\,\nu>0,\,\sigma>0,
\end{equation}
is characteristic as the Fourier spectrum of $\psi$, given by
\begin{equation}\label{Eq:inversequadratic}
\widehat{\psi}(\omega)=\frac{2^{d+\nu}\pi^{d/2}\Gamma(\nu+d/2)\nu^\nu}{\Gamma(\nu)\sigma^{2\nu}}\left(\frac{2\nu}{\sigma^2}+4\pi^2\Vert\omega\Vert^2_2\right)^{-(\nu+d/2)},\,\omega\in\bb{R}^d,
\end{equation}
is positive for any $\omega\in\bb{R}^d$. Here, $\Gamma$ is the Gamma function, $K_\nu$ is the modified Bessel function of the second kind of order $\nu$, where $\nu$ controls the smoothness of $k$. The case of $\nu=\frac{1}{2}$ in the Mat\'{e}rn class gives the exponential kernel, $k(x,y)=\exp(-\Vert x-y\Vert_2/\sigma)$, while $\nu\rightarrow\infty$ gives the Gaussian kernel. Note that $\widehat{\psi}(x-y)$ in (\ref{Eq:inversequadratic}) is actually the inverse multiquadratic kernel, which is characteristic both by Theorem~\ref{thm:strictpd} and Theorem~\ref{Thm:countablyinfinite}.
\begin{table}[t]
\begin{center}
\small{
\begin{tabular}{|cccc|}\hline
& & &\\
Kernel & $\psi(x)$ & $\widehat{\psi}(\omega)$ & $\text{supp}(\widehat{\psi})$ \\
& & &\\\hline
& & &\\
Gaussian & $\exp\left(-\frac{x^2}{2\sigma^2}\right)$& $\sigma\exp\left(-\frac{\sigma^2\omega^2}{2}\right)$& $\bb{R}$\\
& & &\\
Laplacian &$\exp(-\sigma|x|)$ &$\sqrt{\frac{2}{\pi}}\frac{\sigma}{\sigma^2+\omega^2}$ &$\bb{R}$ \\
& & &\\
$B_{2n+1}$-spline & $\ast^{(2n+2)}_{1}\mathds{1}_{\left[-\frac{1}{2},\frac{1}{2}\right]}(x)$& $\frac{4^{n+1}}{\sqrt{2\pi}}\frac{\sin^{2n+2}\left(\frac{
\omega}{2}\right)}{\omega^{2n+2}}$& $\bb{R}$\\
& & &\\
Sinc & $\frac{\sin(\sigma x)}{x}$&$\sqrt{\frac{\pi}{2}}\mathds{1}_{[-\sigma,\sigma]}(\omega)$ & $[-\sigma,\sigma]$ \\
& & &\\
\hline
& & &\\
Poisson & $\frac{1-\sigma^2}{\sigma^2-2\sigma\cos(x)+1},\,0<\sigma<1$ &
$\sqrt{2\pi}\sum^\infty_{j=-\infty}\sigma^{|j|}\,\delta(\omega-j)$ & $\bb{Z}$\\
& & &\\
Dirichlet & $\frac{\sin\frac{(2n+1)x}{2}}{\sin\frac{x}{2}}$ & $\sqrt{2\pi}\sum^n_{j=-n}\delta(\omega-j)$ & $\{0,\pm 1,\ldots,\pm n\}$ \\
& & &\\
F\'{e}jer & $\frac{1}{n+1}\frac{\sin^2\frac{(n+1)x}{2}}{\sin^2\frac{x}{2}}$ & $\sqrt{2\pi}\sum^n_{j=-n}\left(1-\frac{|j|}{n+1}\right)\delta(\omega-j)$ & $\{0,\pm 1,\ldots,\pm n\}$\\
& & &\\
Cosine & $\cos(\sigma x)$&$\sqrt{\frac{\pi}{2}}\left[\delta(\omega-\sigma)+\delta(\omega+\sigma)\right]$& $\{-\sigma,\sigma\}$\\
& & &\\
\hline
\end{tabular}
}
\caption{Translation invariant kernels on $\bb{R}$ defined by $\psi$, their spectra, $\widehat{\psi}$ and its support, $\text{supp}(\widehat{\psi})$. The first four are aperiodic kernels while the last four are periodic. The domain is considered to be $\bb{R}$ for simplicity. For $x\in\bb{R}^d$, the above formulae can be extended by computing $\psi(x)=\prod^d_{j=1}\psi(x_j)$ where $x=(x_1,\ldots,x_d)$ and $\widehat{\psi}(\omega)=\prod^d_{j=1}\widehat{\psi}(\omega_j)$ where $\omega=(\omega_1,\ldots,\omega_d)$. $\delta$ represents the Dirac-delta function.\label{tab:kernel}}
\end{center}
\vspace{-6mm}
\end{table}
\vspace{1mm}
\par By Theorem~\ref{Thm:countablyinfinite}, the Sinc kernel in Table~\ref{tab:kernel} is not characteristic, which is not easy to show using other characterizations. By combining Theorem~\ref{thm:strictpd} with Theorem~\ref{Thm:countablyinfinite}, it can be shown that the Sinc, Poisson, Dirichlet, F\'{e}jer and cosine kernels are not integrally strictly pd. Therefore, for translation invariant kernels on $\bb{R}^d$, the integral strict positive definiteness of the kernel (or the lack of it) can be tested using Theorems \ref{thm:strictpd} and \ref{Thm:countablyinfinite}.
\par We note that, of all the kernels shown in Table~\ref{tab:kernel}, only the Gaussian, Laplacian and $B_{2n+1}$-spline kernels are integrable and their corresponding $\widehat{\psi}$ are computed using (\ref{Eq:FT}). The other kernels shown in Table~\ref{tab:kernel} are not integrable and their corresponding $\widehat{\psi}$ have to be treated as distributions (see \citet[Chapter 9]{Folland-99} and \citet[Chapter 6]{Rudin-91} for details), except for the Sinc kernel whose Fourier transform can be computed in the $L^2$ sense.\footnote{If $f\in L^2(\bb{R}^d)$, the Fourier transform $\digamma[f]:=\widehat{f}$ of $f$ is defined to be the limit, in the $L^2$-norm, of the sequence $\{\widehat{f}_n\}$ of Fourier transforms of any sequence $\{f_n\}$ of functions belonging to $\Scr{S}_d$, such that $f_n$ converges in the $L^2$-norm to the given function $f\in L^2(\bb{R}^d)$, as $n\rightarrow\infty$. The function $\widehat{f}$ is defined almost everywhere on $\bb{R}^d$ and belongs to $L^2(\bb{R}^d)$. Thus, $\digamma$ is a linear operator, mapping $L^2(\bb{R}^d)$ into $L^2(\bb{R}^d)$. See \citet[Chapter IV, Lesson 22]{Gasquet-99} for details.\label{footnote:schwartz}} \vspace{1mm}
\begin{proof}
\hspace{-.05in}\textbf{(Theorem~\ref{Thm:countablyinfinite})} We provide an outline of the complete proof, which is presented in Section~\ref{subsubsec:proofs}. The sufficient condition in Theorem~\ref{Thm:countablyinfinite} is simple to prove and follows from Corollary~\ref{cor:L2distance}\emph{(i)}, whereas we need a supplementary result to prove its necessity, which is presented in Lemma~\ref{lem:constructp} (see Section~\ref{subsubsec:proofs}). Proving the necessity of Theorem~\ref{Thm:countablyinfinite} is equivalent to showing that if $\text{supp}(\Lambda)\subsetneq\bb{R}^d$, then $\exists\,\bb{P}\ne\bb{Q}$, $\bb{P},\bb{Q}\in\Scr{P}$ such that $\gamma_k(\bb{P},\bb{Q})=0$. In Lemma~\ref{lem:constructp}, we present equivalent conditions for the existence of $\bb{P}\ne\bb{Q}$ such that $\gamma_k(\bb{P},\bb{Q})=0$ if $\text{supp}(\Lambda)\subsetneq\bb{R}^d$, using which we prove the necessity of Theorem~\ref{Thm:countablyinfinite}.\vspace{-2mm}
\end{proof}
\par The whole family of compactly supported translation invariant continuous bounded kernels on $\bb{R}^d$ is characteristic, as shown by the following corollary to Theorem~\ref{Thm:countablyinfinite}.
\begin{corollary}\label{cor:compactsupport}
Suppose $k\ne 0$ satisfies Assumption~\ref{assume-1} and $\emph{supp}(\psi)$ is compact. Then $k$ is characteristic.
\end{corollary}
\begin{proof}
Since $\text{supp}(\psi)$ is compact in $\bb{R}^d$, by the Paley-Wiener theorem (Theorem~\ref{Thm:paley-wiener} in Appendix A) and Lemma~\ref{lem:entire} (see Appendix A), 
we deduce that $\text{supp}(\Lambda)=\bb{R}^d$. Therefore, the result follows from Theorem~\ref{Thm:countablyinfinite}.
\end{proof}
The above result is interesting in practice because of the computational advantage in dealing with compactly supported kernels. Note that proving such a general result for compactly supported kernels on $\bb{R}^d$ is not straightforward (maybe not even possible) with the other characterizations.
\par As a corollary to Theorem~\ref{Thm:countablyinfinite}, the following result provides a method to construct new characteristic kernels from a given one.
\begin{corollary}\label{cor:newkernel}
Let $k$, $k_1$ and $k_2$ satisfy Assumption~\ref{assume-1}. Suppose $k$ is characteristic and $k_2\ne 0$. Then $k+k_1$ and $k\cdot k_2$ are characteristic.
\end{corollary}
\begin{proof}
Since $k$, $k_1$ and $k_2$ satisfy Assumption~\ref{assume-1}, $k+k_1$ and $k_2\cdot k$ also satisfy Assumption~\ref{assume-1}. In addition, 
\begin{eqnarray}
(k+k_1)(x,y)&\!\!\!:=\!\!\!&k(x,y)+k_1(x,y)=\psi(x-y)+\psi_1(x-y)=\int_{\bb{R}^d}e^{-i(x-y)^T\omega}\,d(\Lambda+\Lambda_1)(\omega),\nonumber\\
(k\cdot k_2)(x,y)&\!\!\!:=\!\!\!&k(x,y)k_2(x,y)=\psi(x-y)\psi_2(x-y)=\int\!\!\!\int_{\bb{R}^d}e^{-i(x-y)^T(\omega+\xi)}\,d\Lambda(\omega)\,d\Lambda_2(\xi)\nonumber\\
&\!\!\!\stackrel{(a)}{=:}\!\!\!&\int_{\bb{R}^d}e^{-i(x-y)^T\omega}\,d(\Lambda\ast\Lambda_2)(\omega),\nonumber
\end{eqnarray}
where $(a)$ follows from the definition of convolution of measures (see \citet[Section 9.14]{Rudin-91} for details). 
Since $k$ is characteristic, i.e., $\text{supp}(\Lambda)=\bb{R}^d$, and $\text{supp}(\Lambda)\subset\text{supp}(\Lambda+\Lambda_1)$, 
we have $\text{supp}(\Lambda+\Lambda_1)=\bb{R}^d$ and therefore $k+k_1$ is characteristic. Similarly, since $\text{supp}(\Lambda)\subset\text{supp}(\Lambda*\Lambda_2)$, 
we have $\text{supp}(\Lambda*\Lambda_2)=\bb{R}^d$ and therefore, $k\cdot k_2$ is characteristic.
\end{proof}
Note that in the above result, we do not need $k_1$ or $k_2$ to be characteristic. Therefore, one can generate all sorts of kernels that are characteristic by starting with a characteristic kernel, $k$.
\par So far, we have considered characterizations for $k$ such that it is characteristic to $\Scr{P}$. We showed in Theorem~\ref{Thm:countablyinfinite} that kernels with $\text{supp}(\Lambda)\subsetneq\bb{R}^d$ are not characteristic to $\Scr{P}$. Now, we can question whether such kernels can be characteristic to some proper subset $\Scr{Q}$ of $\Scr{P}$. The following result addresses this. Note that these kernels, i.e., the kernels with $\text{supp}(\Lambda)\subsetneq\bb{R}^d$ are usually not useful in practice, especially in statistical inference applications, because the conditions on $\Scr{Q}$ are usually not easy to check. On the other hand, the following result is of theoretical interest: along with Theorem~\ref{Thm:countablyinfinite}, it completes the characterization of characteristic kernels that are translation invariant on $\bb{R}^d$. Before we state the result, we denote $\bb{P}\ll\bb{Q}$ to mean that $\bb{P}$ is absolutely continuous w.r.t. $\bb{Q}$.
\begin{theorem}\label{Thm:compact}
Let $\Scr{P}_1:=\{\bb{P}\in\Scr{P}:\phi_\bb{P}\in L^1(\bb{R}^d)\cup L^2(\bb{R}^d),\,\,\bb{P}\ll\lambda\,\,\text{and}\,\,\emph{supp}(\bb{P})\,\,\text{is compact}\}$, where $\lambda$ is the Lebesgue measure. Suppose $k$ satisfies Assumption~\ref{assume-1} and $\emph{supp}(\Lambda)\subsetneq\bb{R}^d$ has a non-empty interior, where $\Lambda$ is defined as in (\ref{Eq:Bochner}). Then $k$ is characteristic to $\Scr{P}_1$.
\end{theorem}
\begin{proof}
See Section~\ref{subsubsec:proofs}.
\end{proof}
Although, by Theorem~\ref{Thm:countablyinfinite}, the kernels with $\text{supp}(\Lambda)\subsetneq\bb{R}^d$ are not characteristic to $\Scr{P}$, Theorem~\ref{Thm:compact} shows that there exists a subset of $\Scr{P}$ to which a subset of these kernels are characteristic. This type of result is not available for the previously mentioned characterizations. An example of a kernel that satisfies the conditions in Theorem~\ref{Thm:compact} is the Sinc kernel, $\psi(x)=\frac{\sin(\sigma x)}{x}$ which has $\text{supp}(\Lambda)=[-\sigma,\sigma]$. 
The condition that $\text{supp}(\Lambda)\subsetneq\bb{R}^d$ has a non-empty interior is important for Theorem~\ref{Thm:compact} to hold. If $\text{supp}(\Lambda)$ has an empty interior (examples include periodic kernels), then one can construct $\bb{P}\ne \bb{Q},\,\bb{P},\bb{Q}\in\Scr{P}_1$ such that $\gamma_k(\bb{P},\bb{Q})=0$. This is illustrated in Example~\ref{Exm:periodic-compact}, which is deferred to Section~\ref{subsubsec:proofs}. 

\par So far, we have characterized the characteristic property of kernels that satisfy (a) $\text{supp}(\Lambda)=\bb{R}^d$ or (b) $\text{supp}(\Lambda)\subsetneq\bb{R}^d$ with $\text{int}(\text{supp}(\Lambda))\ne\emptyset$. In the following section, we investigate kernels that have $\text{supp}(\Lambda)\subsetneq\bb{R}^d$ with $\text{int}(\text{supp}(\Lambda))=\emptyset$, examples of which include periodic kernels on $\bb{R}^d$. This discussion uses the fact that a periodic function on $\bb{R}^d$ can be treated as a function on $\bb{T}^d$, the $d$-Torus.

\subsection{Characterization for translation invariant $k$ on $\bb{T}^d$}\label{subsec:Td}
Let $M=\times^d_{j=1}[0,\tau_j)$ and $\tau:=(\tau_1,\ldots,\tau_d)$. A function defined on $M$ with periodic boundary conditions is equivalent to considering a periodic function on $\bb{R}^d$ with period $\tau$. With no loss of generality, we can choose $\tau_j=2\pi,\,\forall\,j$ which yields $M=[0,2\pi)^d=:\bb{T}^d$, called the $d$-Torus. The results presented here hold for any $0<\tau_j<\infty,\,\forall\,j$ but we choose $\tau_j=2\pi$ for simplicity. Similar to Assumption~\ref{assume-1}, we now make the following assumption.
\begin{assumption}\label{assume-2}
$k(x,y)=\psi((x-y)_{mod\,\,2\pi})$, where $\psi$ is a continuous real-valued positive definite function on $M=\bb{T}^d$. 
\end{assumption}
Similar to Theorem~\ref{Theorem:Bochner}, we now state Bochner's theorem on $M=\bb{T}^d$.
\begin{theorem}[Bochner]\label{Theorem:Bochnerdiscrete}
A continuous function $\psi:\bb{T}^d\rightarrow\bb{R}$ is positive definite if and only if 
\begin{equation}\label{Eq:discreteBochner}
\psi(x)=\sum_{n\in\bb{Z}^d}A_\psi(n)\,e^{ix^Tn},\,\,x\in\bb{T}^d,
\end{equation}
where $A_\psi:\bb{Z}^d\rightarrow\bb{R}_+$, $A_\psi(-n)=A_\psi(n)$ and $\sum_{n\in \bb{Z}^d}A_\psi(n)<\infty$. $A_\psi$ are called the Fourier series coefficients of $\psi$. 
\end{theorem}
Examples for $\psi$ include the Poisson, Dirichlet, F\'{e}jer and cosine kernels, which are shown in Table~\ref{tab:kernel}. We now state the result that defines characteristic kernels on $\bb{T}^d$.
\begin{theorem}\label{Thm:torus}
Suppose $k$ satisfies Assumption~\ref{assume-2}. Then $k$ is characteristic (to the set of all Borel probability measures on $\bb{T}^d$) if and only if $A_\psi(0)\ge 0$, $A_\psi(n)>0,\,\forall\,n\ne 0$.
\end{theorem}
The proof is provided in Section~\ref{subsubsec:proofs} and the idea is similar to that of Theorem~\ref{Thm:countablyinfinite}. Based on the above result, one can generate characteristic kernels by constructing an infinite sequence of positive numbers that are summable and then using them in (\ref{Eq:discreteBochner}). It can be seen from Table~\ref{tab:kernel} that the Poisson kernel on $\bb{T}$ is characteristic while the Dirichlet, F\'{e}jer and cosine kernels are not. Some examples of characteristic kernels on $\bb{T}$ are:
\begin{list}{\labelitemi}{\leftmargin=2em}
\item[(1)] $k(x,y)=e^{\alpha\cos(x-y)}\cos(\alpha\sin(x-y)),\,0<\alpha\le 1\, \leftrightarrow\, A_\psi(0)=1,\,A_\psi(n)=\frac{\alpha^{|n|}}{2|n|!},\,\forall\,n\ne 0.$
\item[(2)] $k(x,y)=-\log(1-2\alpha\cos(x-y)+\alpha^2),\,|\alpha|< 1 \,\leftrightarrow\,A_\psi(0)=0,\,A_\psi(n)=\frac{\alpha^n}{n},\,\forall\,n\ne 0.$
\item[(3)] $k(x,y)=(\pi-(x-y)_{mod\,\,2\pi})^2\,\leftrightarrow\,A_\psi(0)=\frac{\pi^2}{3},\,A_\psi(n)=\frac{2}{n^2},\,\forall\,n\ne 0.$
\item[(4)] $k(x,y)=\frac{\sinh \alpha}{\cosh \alpha-\cos(x-y)},\,\alpha>0\,\leftrightarrow
\,A_\psi(0)=1, A_\psi(n)=e^{-\alpha|n|},\,\forall\,n\ne 0.$
\item[(5)] $k(x,y)=\frac{\pi\cosh(\alpha(\pi-(x-y)_{mod\,2\pi}))}{\alpha\sinh(\pi\alpha)}\,\leftrightarrow\,A_\psi(0)=\frac{1}{\alpha^2},\,A_\psi(n)=\frac{1}{n^2+\alpha^2},\,\forall\,n\ne 0$.
\end{list}
The following result relates characteristic kernels and universal kernels defined on $\bb{T}^d$.
\begin{corollary}\label{cor:Td}
Let $k$ be a characteristic kernel satisfying Assumption~\ref{assume-2} with $A_\psi(0)>0$. Then $k$ is also universal.
\end{corollary}
\begin{proof}
Since $k$ is characteristic with $A_\psi(0)>0$, we have $A_\psi(n)>0,\,\forall\,n$. Therefore, by Corollary 11 of \citet{Steinwart-01}, $k$ is universal.
\end{proof}
Since $k$ being universal implies that it is characteristic, the above result shows that the converse is not true (though almost true except that $A_\psi(0)$ can be zero for characteristic kernels). The condition on $A_\psi$ in Theorem~\ref{Thm:torus}, i.e., $A_\psi(0)\ge 0,\,A_\psi(n)>0,\,\forall\,n\ne 0$ can be equivalently written as $\text{supp}(A_\psi)=\bb{Z}^d$. Therefore, Theorems~\ref{Thm:countablyinfinite} and \ref{Thm:torus} are of similar flavor. In fact, these results can be generalized to locally compact Abelian groups. \citet{Fukumizu-08b} shows that a bounded continuous translation invariant kernel on a locally compact Abelian group, $G$ is characteristic to the set of all probability measures on $G$ if and only if the support of the Fourier transform of the translation invariant kernel is the dual group of $G$. In our case, $(\bb{R}^d,+)$  and $(\bb{T}^d,+)$ are locally compact Abelian groups with $(\bb{R}^d,+)$ and $(\bb{Z}^d,+)$ as their respective dual groups. In \citet{Fukumizu-08b}, these results are also extended to translation invariant kernels on non-Abelian compact groups and the semigroup $\bb{R}^d_+$. 
\subsection{Relation between various characterizations of characteristic kernels}\label{subsec:relation}
So far, we have presented various characterizations of characteristic kernels, which are easily checkable compared to the characterizations proposed in literature \citep{Gretton-06, Fukumizu-08a, Fukumizu-08b}. Now, it is of interest to understand the relation between these characterizations. A summary of the relationship between these characterizations is shown in Figure~\ref{Fig:relation}, which is discussed below.\vspace{1mm}
\begin{figure*}
  \centering
\begin{tabular}{cccccc}
    \begin{minipage}{12cm} 
      \center{\epsfxsize=10cm 
      \epsffile{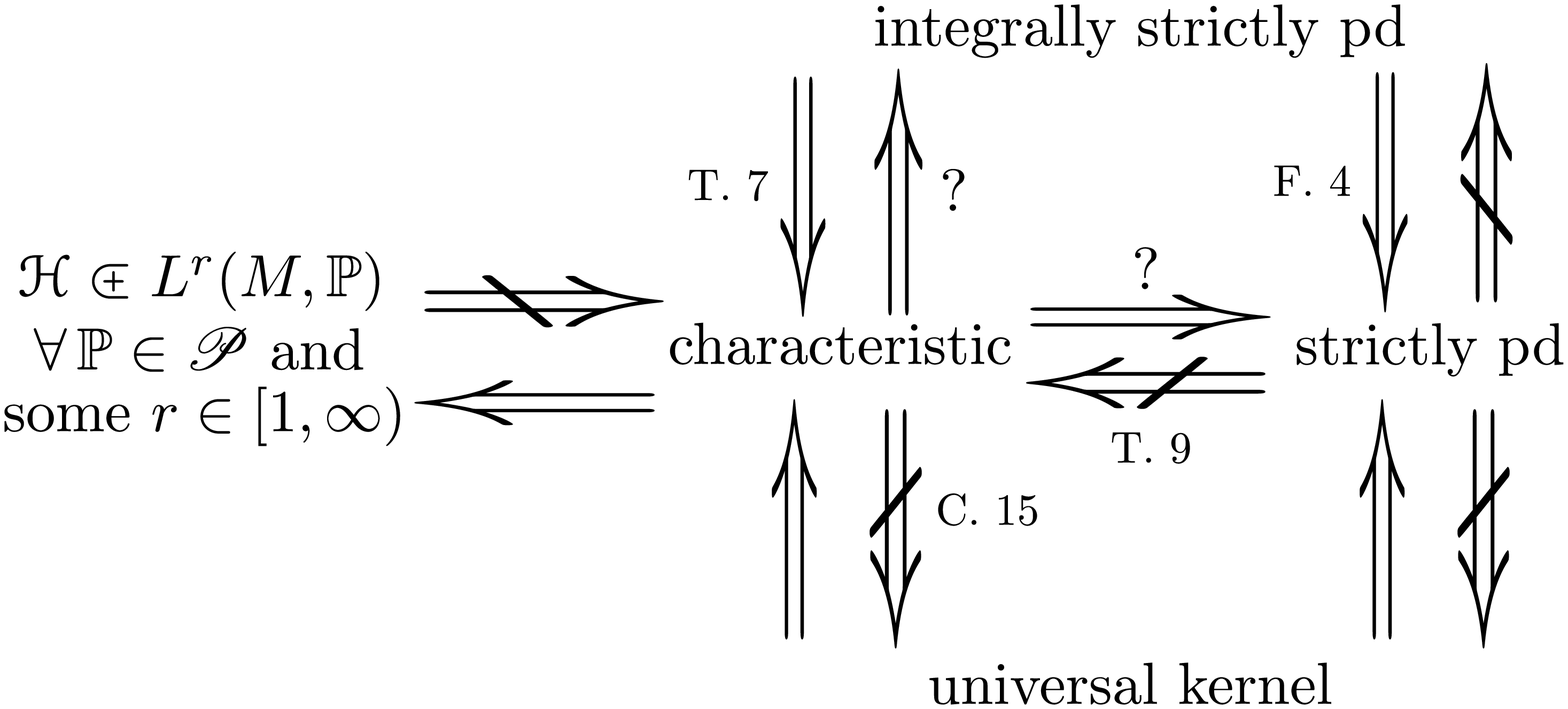}}
      \end{minipage}
  \end{tabular}
  \caption{Summary of the relationship between various characterizations is shown along with the reference. The letters ``C", ``F", and ``T" refer to Corollary, Footnote and Theorem respectively. For example, T. 7 refers to Theorem 7. The implications which are open problems are shown with ``?". $A\subsetplus B$ indicates that $A$ is a dense subset of $B$. Refer to Section~\ref{subsec:relation} for details.}
\label{Fig:relation}
\vspace{-6mm}
\end{figure*}
\par\noindent\textbf{Characteristic kernels vs. Integrally strictly pd kernels:} It is clear from Theorem~\ref{thm:strictpd} that integrally strictly pd kernels on a topological space, $M$ are characteristic, while it is not clear whether the converse is true or not. However, when $k$ is translation invariant on $\bb{R}^d$, then the converse holds. This is because if $k$ is characteristic, then by Theorem~\ref{Thm:countablyinfinite}, $\text{supp}(\Lambda)=\bb{R}^d$, where $\Lambda$ is defined as in (\ref{Eq:Bochner}). It is easy to check that if $\text{supp}(\Lambda)=\bb{R}^d$, then $k$ is integrally strictly pd.\vspace{1mm}
\par\noindent\textbf{Integrally strictly pd kernels vs. Strictly pd kernels:} The relation between integrally strictly pd and strictly pd kernels shown in Figure~\ref{Fig:relation} is straightforward, as one direction follows from footnote~\ref{footnote:ispd}, while the other direction is not true, which follows from \citet[Proposition 4.60, Theorem 4.62]{Steinwart-08}. However, if $M$ is a finite set, then $k$ being strictly pd also implies it is integrally strictly pd.\vspace{1mm}
\par\noindent\textbf{Characteristic kernels vs. Strictly pd kernels:} Since integrally strictly pd kernels are characteristic and are also strictly pd, a natural question to ask is, ``What is the relation between characteristic and strictly pd kernels?'' It can be seen that strictly pd kernels need not be characteristic because the sinc-squared kernel, $k(x,y)=\frac{\sin^2(\sigma(x-y))}{(x-y)^2}$ on $\bb{R}$, which has $\text{supp}(\Lambda)=[-\sigma,\sigma]\subsetneq\bb{R}$ is strictly pd \citep[Theorem 6.11]{Wendland-05}, while it is not characteristic by Theorem~\ref{Thm:countablyinfinite}. However, for any general $M$, it is not clear whether $k$ being characteristic implies that it is strictly pd. As a special case, if $M=\bb{R}^d$ or $M=\bb{T}^d$, then by Theorems~\ref{Thm:countablyinfinite} and \ref{Thm:compact}, it follows that a translation invariant $k$ being characteristic also implies that it is strictly pd.\vspace{1mm}
\par\noindent\textbf{Universal kernels vs. Characteristic kernels:} \citet{Gretton-06} have shown that if $k$ is universal in the sense of \citet{Steinwart-01}, then it is characteristic. As mentioned in Section~\ref{subsec:Td}, the converse is not true, i.e., if a kernel is characteristic, then it need not be universal, which follows from Corollary~\ref{cor:Td}. Note that in this case, $M$ is assumed to be a compact metric space. The notion of universality of kernels was extended to non-compact domains by \citet{Micchelli-06}: $k$ is said to universal on a non-compact Hausdroff space, $M$, if for any compact $Z\subset M$, the set $K(Z):=\overline{\text{span}}\{k(\cdot,y):y\in Z\}$ is dense in $C(Z)$ w.r.t. the supremum norm, where $C(Z)$ represents the space of continuous functions defined on $Z$. It is to be noted that when $M$ is compact, this notion of universality is same as that of \citet{Steinwart-01}. \citet[Proposition 15]{Micchelli-06} have provided a characterization of universality for translation invariant kernels on $\bb{R}^d$: $k$ is universal if $\lambda(\text{supp}(\Lambda))>0$, where $\lambda$ is the Lebesgue measure and $\Lambda$ is defined as in (\ref{Eq:Bochner}). This means, if a translation invariant kernel on $\bb{R}^d$ is characteristic, then it is also universal in the sense of \citet{Micchelli-06}, while the converse is not true. However, the relation between these notions for a general non-compact Hausdorff space, $M$ is not clear.
\par \citet{Fukumizu-08a, Fukumizu-08b} have shown that $k$ is characteristic if and only if $\eu{H}+\bb{R}$ is dense in $L^r(M,\bb{P})$ for all $\bb{P}\in\Scr{P}$ and for some $r\in [1,\infty)$. This means, if $k$ is characteristic, then $\eu{H}+\bb{R}\subsetplus L^r(M,\bb{P})$, which implies $\eu{H}\subsetplus L^r(M,\bb{P})$ for all $\bb{P}\in\Scr{P}$ and for some $r\in [1,\infty)$. Clearly, the converse is not true (refer to Figure~\ref{Fig:relation} for the definition of $\subsetplus$). However, if constant functions are included in $\eu{H}$, then it is easy to see that the converse is also true. \vspace{1mm}
\par\noindent \textbf{Universal kernels vs. Strictly pd kernels:} If a kernel is universal, then it is strictly pd, which follows from \citet[Definition 4.53, Proposition 4.54, Exercise 4.11]{Steinwart-08}. On the other hand, if a kernel is strictly pd, then it need not be universal, which follows from the results due to \citet{Dahmen-87} and \citet{Pinkus-04} for Taylor kernels \citep[Lemma 4.8, Corollary 4.57]{Steinwart-08}. Refer to \citet[Section 4.7, p. 161]{Steinwart-08} for more details.\\
\par\noindent Recently, in \citet{Sriperumbudur-09d,Sriperumbudur-09e}, we carried out a thorough study of relating characteristic kernels to various notions of universality, wherein we addressed some open questions mentioned in the above discussion and Figure~\ref{Fig:relation}. This is done by relating universality to the injective embedding of regular Borel measures into an RKHS, which can therefore be seen as a generalization of the notion of characteristic kernels, as the latter deal with the injective RKHS embedding of probability measures.

\subsection{Proofs}\label{subsubsec:proofs}
First, we present a supplementary result in Lemma~\ref{lem:constructp} that will be used to prove Theorem~\ref{Thm:countablyinfinite}. The idea of Lemma~\ref{lem:constructp} is to characterize the equivalent conditions for the existence of $\bb{P}\ne\bb{Q}$ such that $\gamma_k(\bb{P},\bb{Q})=0$ when $\text{supp}(\Lambda)\subsetneq\bb{R}^d$. Its proof
relies on the properties of characteristic functions, which we have collected in Theorem~\ref{Theorem:FTmeasure} in Appendix A.
\begin{lemma}\label{lem:constructp}
Let $\Scr{P}_0:=\{\bb{P}\in\Scr{P}:\phi_\bb{P}\in L^1(\bb{R}^d)\cup L^2(\bb{R}^d)\,\,\text{and}\,\,\bb{P}\ll\lambda\}$, where $\lambda$ is the Lebesgue measure. Suppose $k$ satisfies Assumption~\ref{assume-1} and $\emph{supp}(\Lambda)\subsetneq\bb{R}^d$, where $\Lambda$ is defined as in (\ref{Eq:Bochner}). Then, for any $\bb{Q}\in\Scr{P}_0$, $\exists\, \bb{P}\ne \bb{Q}$, $\bb{P}\in\Scr{P}_0$ 
such that $\gamma_k(\bb{P},\bb{Q})=0$ if and only if there exists a non-zero function $\theta:\bb{R}^d\rightarrow\bb{C}$ that satisfies the following conditions:
\begin{itemize}
\item[(i)] 
$\theta\in (L^1(\bb{R}^d)\cup L^2(\bb{R}^d))\cap C_b(\bb{R}^d)$ is conjugate symmetric\footnote{Note that $\text{Re}[\theta]$ and $\text{Im}[\theta]$ are even and odd functions in $\bb{R}^d$.}, i.e., $\overline{\theta(x)}=\theta(-x),\,\forall\,x\in\bb{R}^d$,
\item[(ii)] $\theta^\vee\in L^1(\bb{R}^d)\cap (L^2(\bb{R}^d)\cup C_b(\bb{R}^d))$,
\item[(iii)]  $\int_{\bb{R}^d}|\theta(x)|^2\,d\Lambda(x)=0$,
\item[(iv)] $\theta(0)=0$,
\item[(v)]$\inf_{x\in\bb{R}^d}\{\theta^\vee(x)+q(x)\}\ge 0$.
\end{itemize}
\end{lemma}
\begin{proof}
Define $L^1:=L^1(\bb{R}^d)$, $L^2:=L^2(\bb{R}^d)$ and $C_b:=C_b(\bb{R}^d)$.\vspace{2mm}\\\noindent
($\,\Leftarrow\,$) Suppose there exists a non-zero function $\theta$ satisfying \emph{(i)} \emph{--} \emph{(v)}. 
For any $\bb{Q}\in\Scr{P}_0$, we have $\phi_\bb{Q}\in(L^1\cup L^2)\cap C_b$. When $\phi_\bb{Q}\in L^1\cap C_b$, the Riemann-Lebesgue lemma (Lemma~\ref{lem:Riemann-Lebesgue} in Appendix A) implies that $q=[\overline{\phi_\bb{Q}}]^\vee\in L^1\cap C_b$, where $q$ is the Radon-Nikodym derivative of $\bb{Q}$ w.r.t. $\lambda$. When $\phi_\bb{Q}\in L^2\cap C_b$, the Fourier transform in the $L^2$ sense (see footnote~\ref{footnote:schwartz}) implies that $q=[\overline{\phi_\bb{Q}}]^\vee\in L^1\cap L^2$. Therefore, $q\in L^1\cap(L^2\cup C_b)$. Define $p:=q+\theta^\vee$. Clearly $p\in L^1\cap(L^2\cup C_b)$. In addition, $\overline{\phi_\bb{P}}=\widehat{p}=\widehat{q}+\widehat{\theta^\vee}=\overline{\phi_\bb{Q}}+\theta\in(L^1\cup L^2)\cap C_b$. Since $\theta$ is conjugate symmetric, $\theta^\vee$ is real valued and so is $p$. Consider \begin{equation}
\int_{\bb{R}^d}p(x)\,dx=\int_{\bb{R}^d}q(x)\,dx+\int_{\bb{R}^d}\theta^\vee(x)\,dx =1+\theta(0)=1.\nonumber
\end{equation} \emph{(v)} implies that $p$ is non-negative. Therefore, $p$ is the Radon-Nikodym derivative of a probability measure $\bb{P}$ w.r.t. $\lambda$, where $\bb{P}$ is such that $\bb{P}\ne \bb{Q}$ and $\bb{P}\in\Scr{P}_0$. By (\ref{Eq:L2distance}), we have \begin{equation}
\gamma^2_k(\bb{P},\bb{Q})=\int_{\bb{R}^d}|\phi_\bb{P}(x)-\phi_\bb{Q}(x)|^2\,d\Lambda(x)=\int_{\bb{R}^d}|\theta(x)|^2\,d\Lambda(x)=0.\nonumber
\end{equation}\\
($\,\Rightarrow\,$) Suppose that there exists $\bb{P}\ne\bb{Q}$, $\bb{P},\bb{Q}\in\Scr{P}_0$ such that $\gamma_k(\bb{P},\bb{Q})=0$. Define $\theta:=\phi_\bb{P}-\phi_\bb{Q}$. We need to show that $\theta$ satisfies \emph{(i)} \emph{--} \emph{(v)}. $\bb{P},\bb{Q}\in\Scr{P}_0$ implies $\phi_\bb{P},\phi_\bb{Q}\in (L^1\cup L^2)\cap C_b$ and $p,q\in L^1\cap(L^2\cup C_b)$. Therefore, $\theta=\overline{\phi}_\bb{P}-\overline{\phi}_\bb{Q}\in (L^1\cup L^2)\cap C_b$ and $\theta^\vee=p-q\in L^1\cap(L^2\cup C_b)$. By Theorem~\ref{Theorem:FTmeasure} (see Appendix A), $\phi_\bb{P}$ and $\phi_\bb{Q}$ are conjugate symmetric and so is $\theta$. Therefore $\theta$ satisfies \emph{(i)} and $\theta^\vee$ satisfies \emph{(ii)}. $\theta$ satisfies \emph{(iv)} as 
\begin{equation}
\theta(0)=\int_{\bb{R}^d}\theta^\vee(x)\,dx=\int_{\bb{R}^d} (p(x)-q(x))\,dx=0.\nonumber
\end{equation}
Non-negativity of $p$ yields \emph{(v)}. By (\ref{Eq:L2distance}), $\gamma_k(\bb{P},\bb{Q})=0$ implies \emph{(iii)}.\vspace{-4mm}
\end{proof}
\begin{remark}\label{rem:theta}
Note that the dependence of $\theta$ on the kernel appears in the form of (iii) in Lemma~\ref{lem:constructp}. This condition shows that $\lambda(\emph{supp}(\theta)\cap \emph{supp}(\Lambda))=0$, i.e., the supports of $\theta$ and $\Lambda$ are disjoint w.r.t. the Lebesgue measure, $\lambda$. In other words, $\emph{supp}(\theta)\subset\emph{cl}(\bb{R}^d\backslash\emph{supp}(\Lambda))$. So, the idea is to introduce the perturbation, $\theta$ over an open set, $U$ where $\Lambda(U)=0$. The remaining conditions characterize the nature of this perturbation so that the constructed measure, $p=q+\theta^\vee$, is a valid probability measure. Conditions (i), (ii) and (iv) simply follow from $\theta=\phi_\bb{P}-\phi_\bb{Q}$, while (v) ensures that $p(x)\ge 0,\,\forall\,x$.
\end{remark}
Using Lemma~\ref{lem:constructp}, we now present the proof of Theorem~\ref{Thm:countablyinfinite}.
\vspace{2mm}
\begin{proof}
\hspace{-.05in}\textbf{(Theorem~\ref{Thm:countablyinfinite})} The sufficiency follows from (\ref{Eq:L2distance}): if $\text{supp}(\Lambda)=\bb{R}^d$, then $\gamma^2_k(\bb{P},\bb{Q})=\int_{\bb{R}^d}|\phi_\bb{P}(x)-\phi_\bb{Q}(x)|^2\,d\Lambda(x)=0\Rightarrow \phi_\bb{P}=\phi_\bb{Q}$, a.e., which implies $\bb{P}=\bb{Q}$ and therefore $k$ is characteristic. To prove necessity, we need to show that if $\text{supp}(\Lambda)\subsetneq\bb{R}^d$, then there exists $\bb{P}\ne\bb{Q}$, $\bb{P},\bb{Q}\in\Scr{P}$ such that $\gamma_k(\bb{P},\bb{Q})=0$. By Lemma~\ref{lem:constructp}, this is equivalent to showing that there exists a non-zero $\theta$ satisfying the conditions in Lemma~\ref{lem:constructp}. Below, we provide a constructive procedure for such a $\theta$ when $\text{supp}(\Lambda)\subsetneq\bb{R}^d$, thereby proving the result.
\par Consider the following function, $f_{\beta,\omega_0}\in C^\infty(\bb{R}^d)$ supported in $[\omega_0-\beta,\omega_0+\beta]$,
\begin{equation}
f_{\beta,\omega_0}(\omega)=\prod^d_{j=1}h_{\beta_j,\omega_{0,j}}(\omega_j)\,\,\text{with}\,\,h_{a,b}(y):=\mathds{1}_{[-a,a]}(y-b)\,e^{-\frac{a^2}{a^2-(y-b)^2}},
\end{equation}
where $\omega=(\omega_1,\ldots,\omega_d),\,\omega_0=(\omega_{0,1},\ldots,\omega_{0,d}),\,\beta=(\beta_1,\ldots,\beta_d),\,a\in\bb{R}_{++},\,b\in\bb{R}$ and $y\in\bb{R}$. Since $\text{supp}(\Lambda)\subsetneq\bb{R}^d$, there exists an open set $U\subset\bb{R}^d$ such that $\Lambda(U)=0$. So, there exists $\beta\in\bb{R}^d_{++}$ and $\omega_0>\beta$ (element-wise inequality) such that $[\omega_0-\beta,\omega_0+\beta]\subset U$. Let 
\begin{equation}
\theta=\alpha(f_{\beta,\omega_0}+f_{\beta,-\omega_0}),\,\alpha\in\bb{R}\backslash\{0\},
\end{equation}
which implies $\text{supp}(\theta)=[-\omega_0-\beta,-\omega_0+\beta]\cup[\omega_0-\beta,\omega_0+\beta]$ is compact. Clearly $\theta\in\Scr{D}_d\subset\Scr{S}_d$ which implies 
$\theta^\vee\in\Scr{S}_d\subset L^1(\bb{R}^d)\cap L^2(\bb{R}^d)$. 
Therefore, by construction, $\theta$ satisfies \emph{(i)} -- \emph{(iv)} in Lemma~\ref{lem:constructp}. Since $\int_{\bb{R}^d}\theta^\vee(x)\,dx=\theta(0)=0$ (by construction), $\theta^\vee$ will take negative values, so we need to show that there exists $\bb{Q}\in\Scr{P}_0$ such that \emph{(v)} in Lemma~\ref{lem:constructp} holds. Let $\bb{Q}$ be such that it has a density given by 
\begin{equation}\label{Eq:cauchy-proof}
q(x)=C_l\prod^d_{j=1}\frac{1}{(1+|x_j|^2)^l},\,l\in\bb{N}\,\,\text{where }\,\,C_l=\prod^d_{j=1}\left(\int_\bb{R}(1+|x_j|^2)^{-l}\,dx_j\right)^{-1},
\end{equation}
and $x=(x_1,\ldots,x_d)$. It can be verified that choosing $\alpha$ such that
\begin{equation}
0<|\alpha|\le\frac{C_l}{2\sup_x\left|\prod^d_{j=1}h^\vee_{\beta_j,0}(x_j)(1+|x_j|^2)^l\cos(\omega^T_0x)\right|}<\infty,\nonumber
\end{equation} ensures that $\theta$ satisfies \emph{(v)} in Lemma~\ref{lem:constructp}. The existence of finite $\alpha$ is guaranteed as $h_{a,0}\in\mathscr{D}_1\subset\mathscr{S}_1$ which implies $h^\vee_{a,0}\in\mathscr{S}_1,\,\forall\,a$. 
We conclude there exists a non-zero $\theta$ as claimed earlier, which completes the proof. 
\end{proof}
\noindent To elucidate the necessity part in the above proof, in the following, we present a simple example that provides an intuitive understanding about the construction of $\theta$ such that for a given $\bb{Q}$, $\bb{P}\ne\bb{Q}$ can be constructed with $\gamma_k(\bb{P},\bb{Q})=0$.
\begin{example}\label{Exm:noncompact}
Let $\bb{Q}$ be a Cauchy distribution in $\bb{R}$, i.e., $q(x)=\frac{1}{\pi(1+x^2)}$ with characteristic function, $\phi_\bb{Q}(\omega)=\frac{1}{\sqrt{2\pi}}e^{-|\omega|}$ in $L^1(\bb{R})$. Let $\psi$ be a Sinc kernel, i.e., $\psi(x)=\sqrt{\frac{2}{\pi}}\frac{\sin(\beta x)}{x}$ with Fourier transform given by $\widehat{\psi}(\omega)=\mathds{1}_{[-\beta,\beta]}(\omega)$ and $\emph{supp}(\widehat{\psi})=[-\beta,\beta]\subsetneq\bb{R}$. Let $\theta$ be
\begin{equation}
\theta(\omega)=\frac{\alpha}{2i}\left[\ast^N_1\mathds{1}_{\left[-\frac{\beta}{2},\frac{\beta}{2}\right]}(\omega)\right]\ast\left[\delta(\omega-\omega_0)-\delta(\omega+\omega_0)\right],
\end{equation}
where $|\omega_0|\ge\left(\frac{N+2}{2}\right)\beta,\,N\ge 2$ and $\alpha\ne 0$. $\ast^N_1$ represents the $N$-fold convolution. Note that $\theta$ is such that $\emph{supp}(\theta)\cap\emph{supp}(\widehat{\psi})$ is a null set w.r.t. the Lebesgue measure, which satisfies (iii) in Lemma~\ref{lem:constructp}. It is easy to verify that $\theta\in L^1(\bb{R})\cap L^2(\bb{R})\cap C_b(\bb{R})$ also satisfies conditions (i) and (iv) in Lemma~\ref{lem:constructp}. $\theta^\vee$ can be computed as
\begin{equation}
\theta^\vee(x)=\frac{2^N\alpha}{\sqrt{2\pi}}\sin(\omega_0 x)\frac{\sin^N\left(\frac{\beta x}{2}\right)}{x^N},
\end{equation}
and $\theta^\vee\in L^1(\bb{R})\cap L^2(\bb{R})\cap C_b(\bb{R})$ satisfies $(ii)$ in Lemma~\ref{lem:constructp}. 
Choose 
\begin{equation}
0<|\alpha|\le\frac{\sqrt{2}}{\sqrt{\pi}\beta^N\sup_x\left|(1+x^2)\sin(\omega_0x)\emph{sinc}^N\left(\frac{\beta x}{2\pi}\right)\right|},
\end{equation}
where $\emph{sinc}(x):=\frac{\sin(\pi x)}{\pi x}$. Define $g(x):=\sin(\omega_0x)\emph{sinc}^N\left(\frac{\beta x}{2\pi}\right)$. Since $g\in \Scr{S}_1$, $0<\sup_x |(1+x^2)g(x)|<\infty$ and, therefore, $\alpha$ is a finite non-zero number. It is easy to see that $\theta$ satisfies $(v)$ of Lemma~\ref{lem:constructp}. Then, by Lemma~\ref{lem:constructp}, there exists $\bb{P}\ne\bb{Q}$, $\bb{P}\in \Scr{P}_0$, given by
\begin{equation}
p(x)=\frac{1}{\pi(1+x^2)}+\frac{2^N\alpha}{\sqrt{2\pi}}\sin(\omega_0 x)\frac{\sin^N\left(\frac{\beta x}{2}\right)}{x^N},
\end{equation}
with $\phi_\bb{P}=\phi_\bb{Q}+\theta=\phi_\bb{Q}+i\theta_I$ where $\theta_I=\emph{Im}[\theta]$ and $\phi_\bb{P}\in L^1(\bb{R})$. So, we have constructed $\bb{P}\ne \bb{Q}$, such that $\gamma_k(\bb{P},\bb{Q})=0$. Figure~\ref{fig:noncompact} shows the plots of $\psi$, $\widehat{\psi}$, $\theta$, $\theta^\vee$, $q$, $\phi_\bb{Q}$, $p$ and $|\phi_\bb{P}|$ for $\beta=2\pi$, $N=2$, $\omega_0=4\pi$ and $\alpha=\frac{1}{50}$.
\end{example}
\begin{figure*}
\vspace{-6mm}
  \centering
  \begin{tabular}{ccc}
    \begin{minipage}{7cm}
      \center{\epsfxsize=5.5cm
      \epsffile{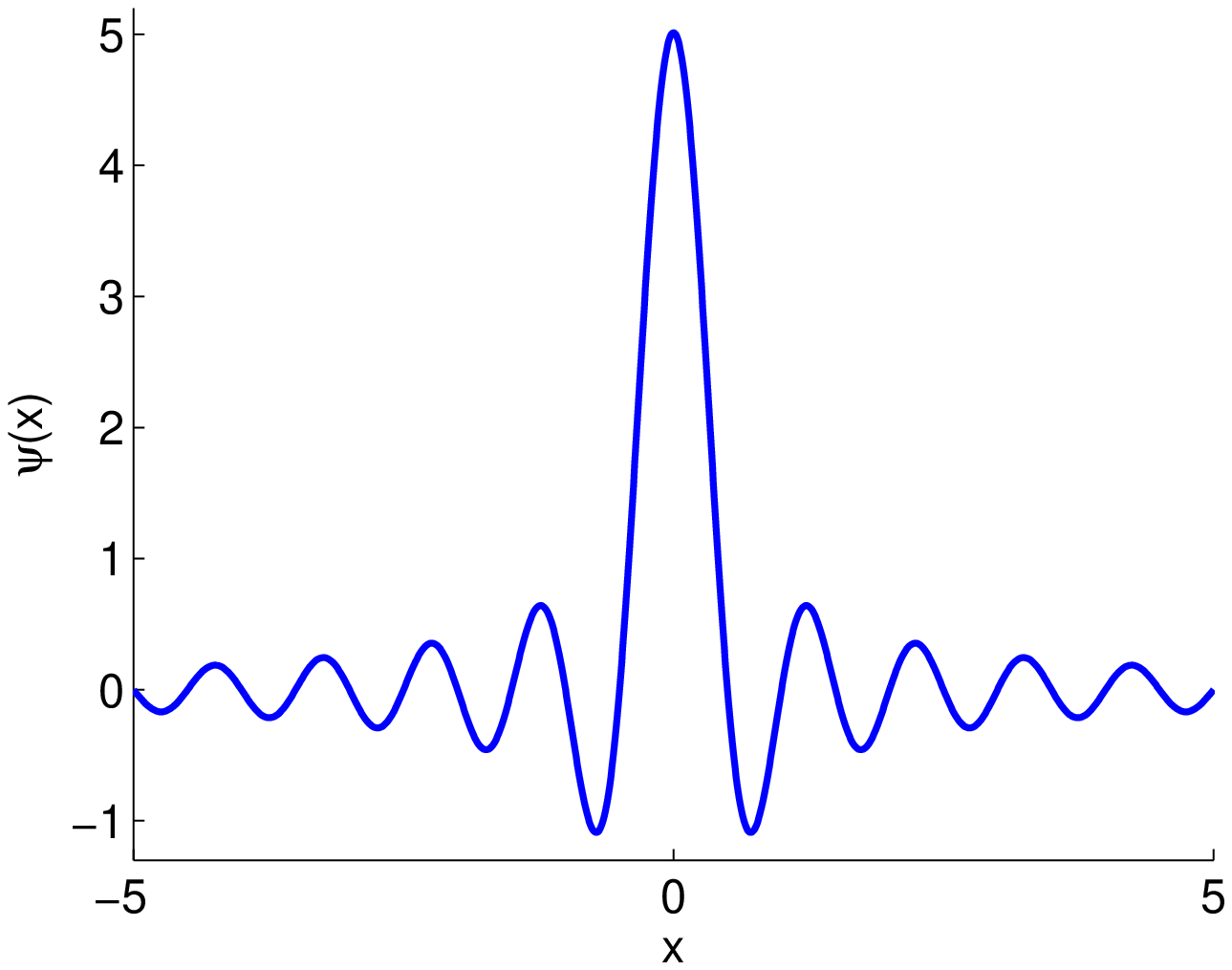}}\vspace{-3mm}
      {\small \center{(a)}}
    \end{minipage}
    \begin{minipage}{7cm}
      \center{\epsfxsize=5.5cm
      \epsffile{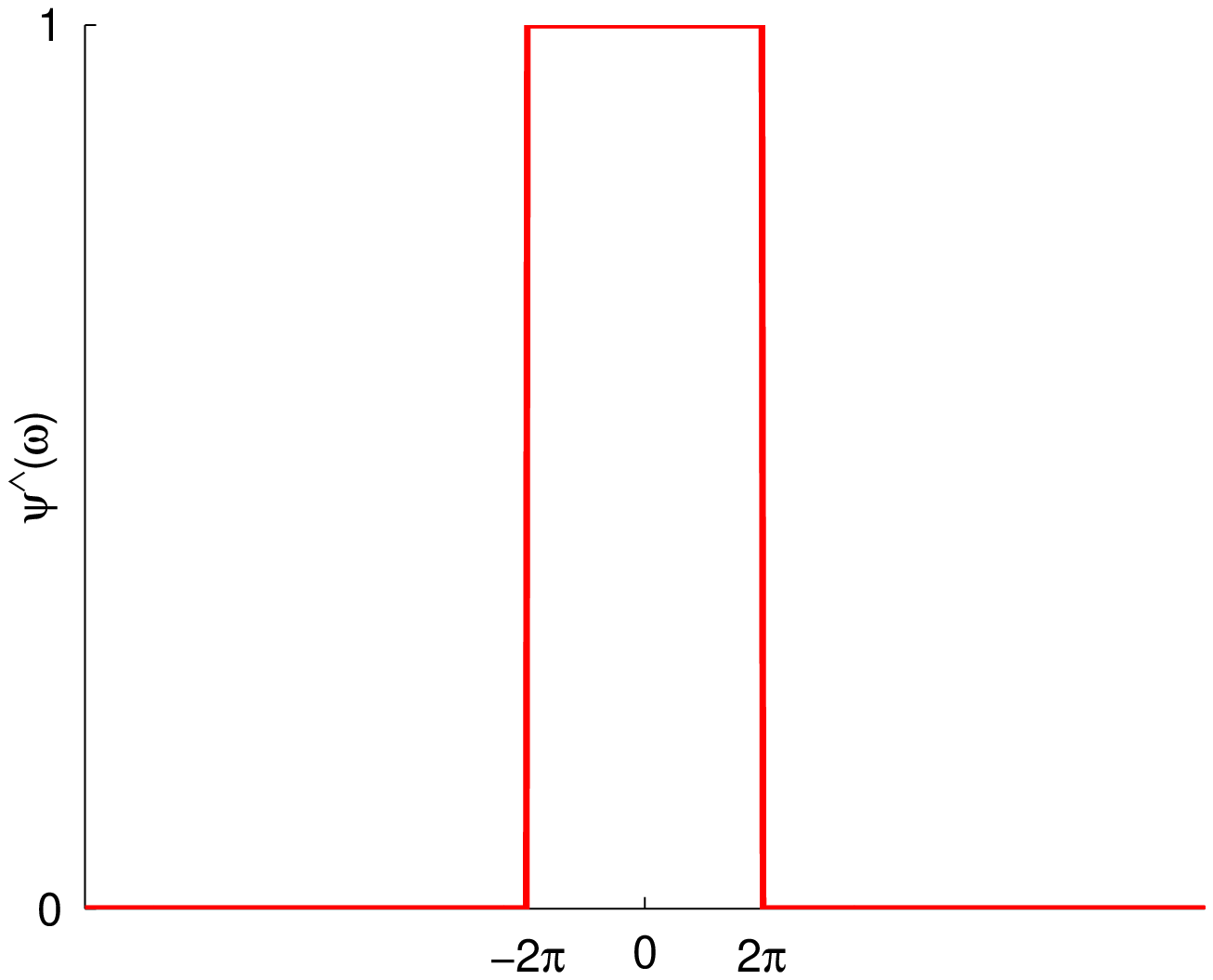}}\vspace{-3mm}
      {\small \center{(a$^\prime$)}}
    \end{minipage}
    \vspace{3mm}
  \end{tabular}
  \begin{tabular}{ccc}
    \begin{minipage}{7cm}
      \center{\epsfxsize=5.5cm
      \epsffile{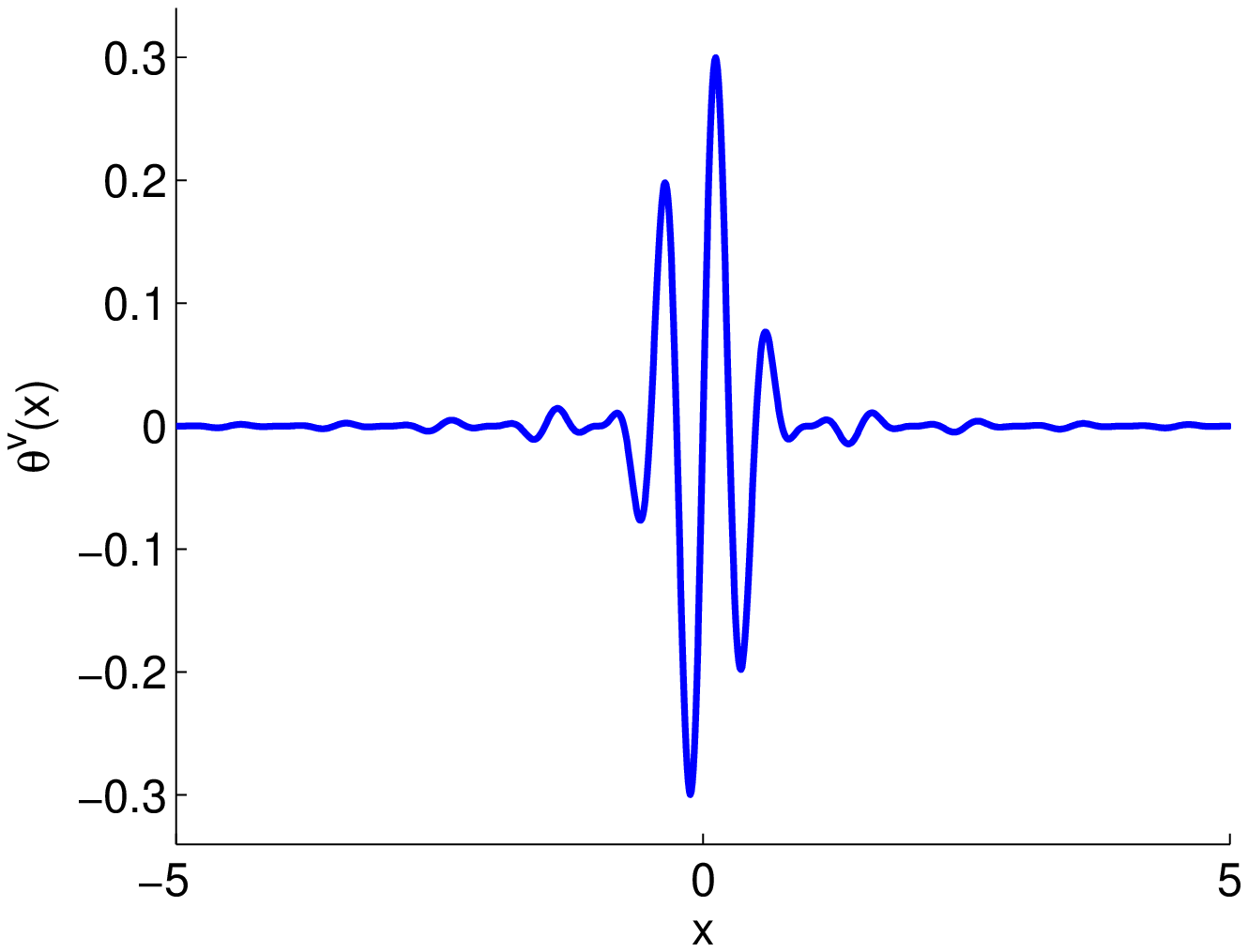}}\vspace{-3mm}
      {\small \center{(b)}}
    \end{minipage}
    \begin{minipage}{7cm}
      \center{\epsfxsize=5.5cm
      \epsffile{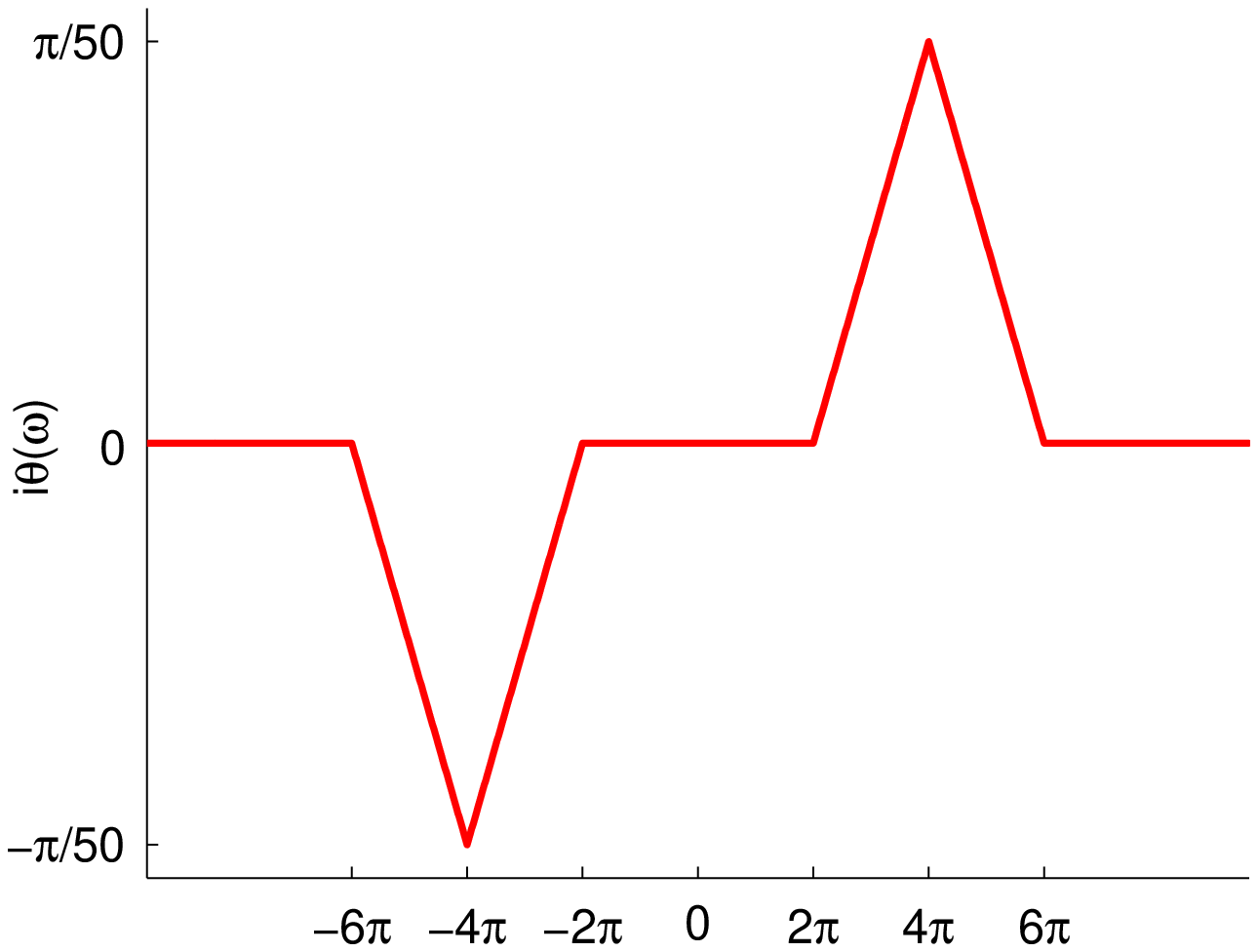}}\vspace{-3mm}
      {\small \center{(b$^\prime$)}}
    \end{minipage}
    \vspace{3mm}
  \end{tabular}
  \begin{tabular}{ccc}
    \begin{minipage}{7cm}
      \center{\epsfxsize=5.5cm
      \epsffile{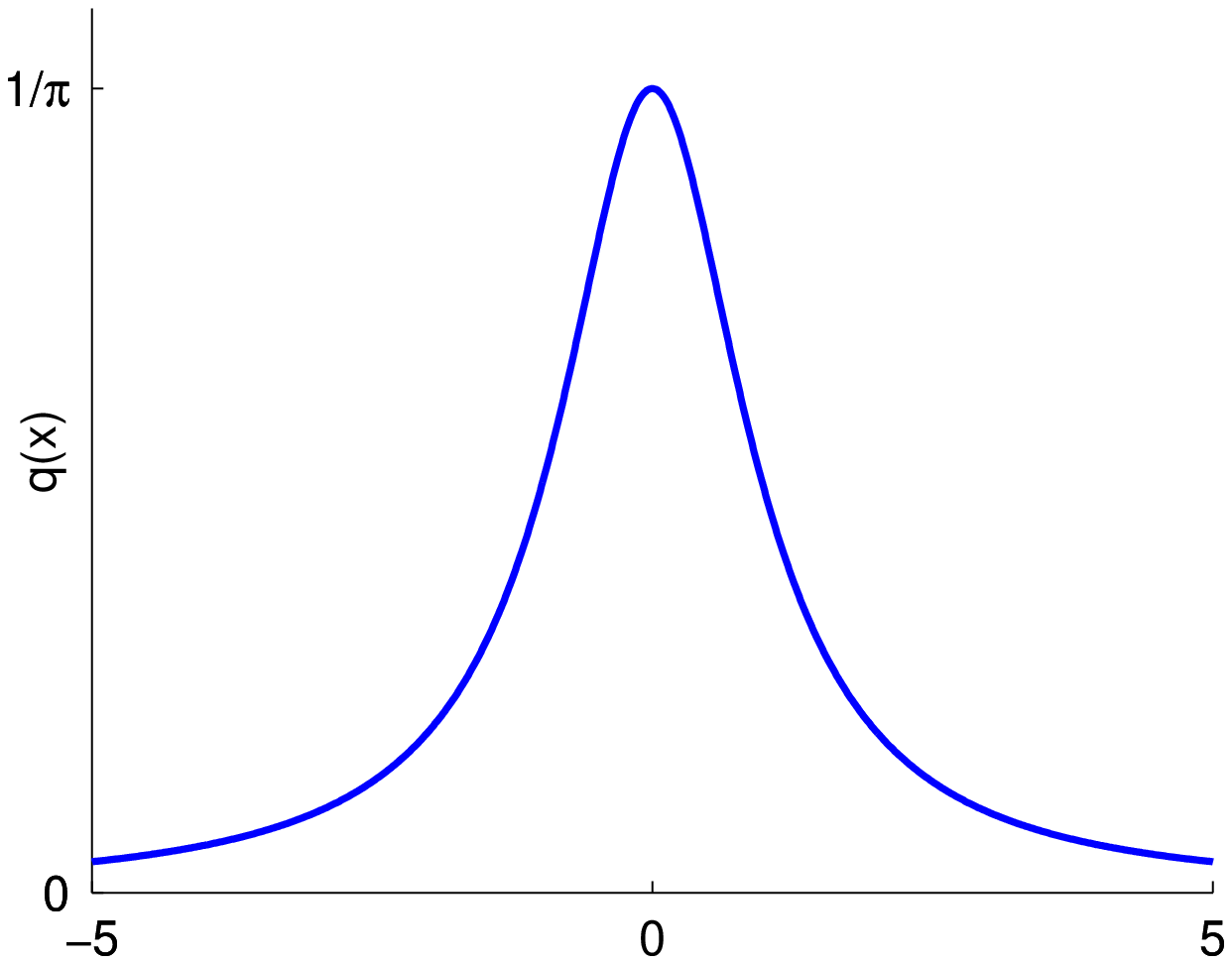}}\vspace{-3mm}
      {\small \center{(c)}}
    \end{minipage}
    \begin{minipage}{7cm}
      \center{\epsfxsize=5.5cm
      \epsffile{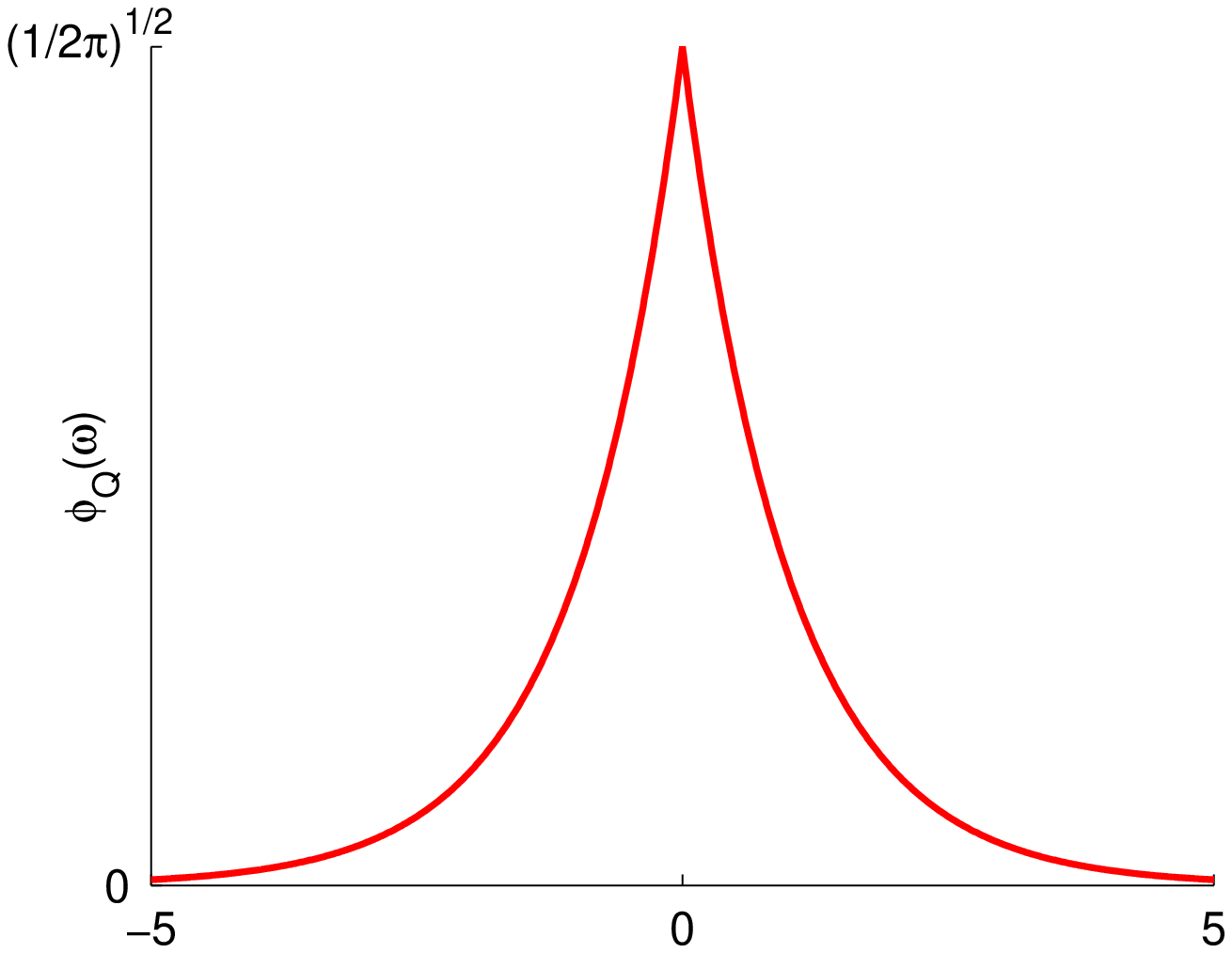}}\vspace{-3mm}
      {\small \center{(c$^\prime$)}}
    \end{minipage}
    \vspace{3mm}
  \end{tabular}
\begin{tabular}{ccc}
    \begin{minipage}{7cm}
      \center{\epsfxsize=5.5cm
      \epsffile{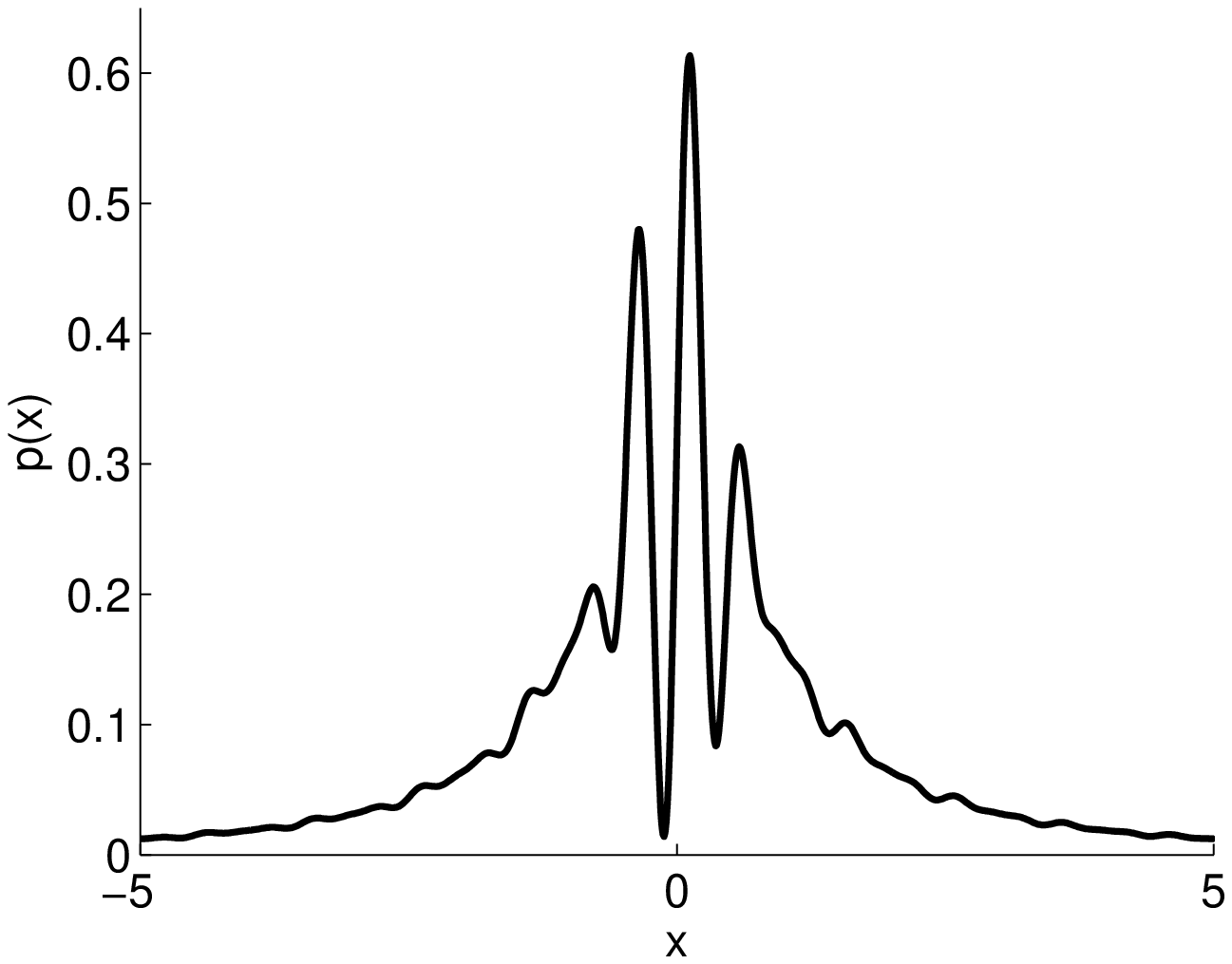}}\vspace{-3mm}
      {\small \center{(d)}}
    \end{minipage}
    \begin{minipage}{7cm}
      \center{\epsfxsize=5.5cm
      \epsffile{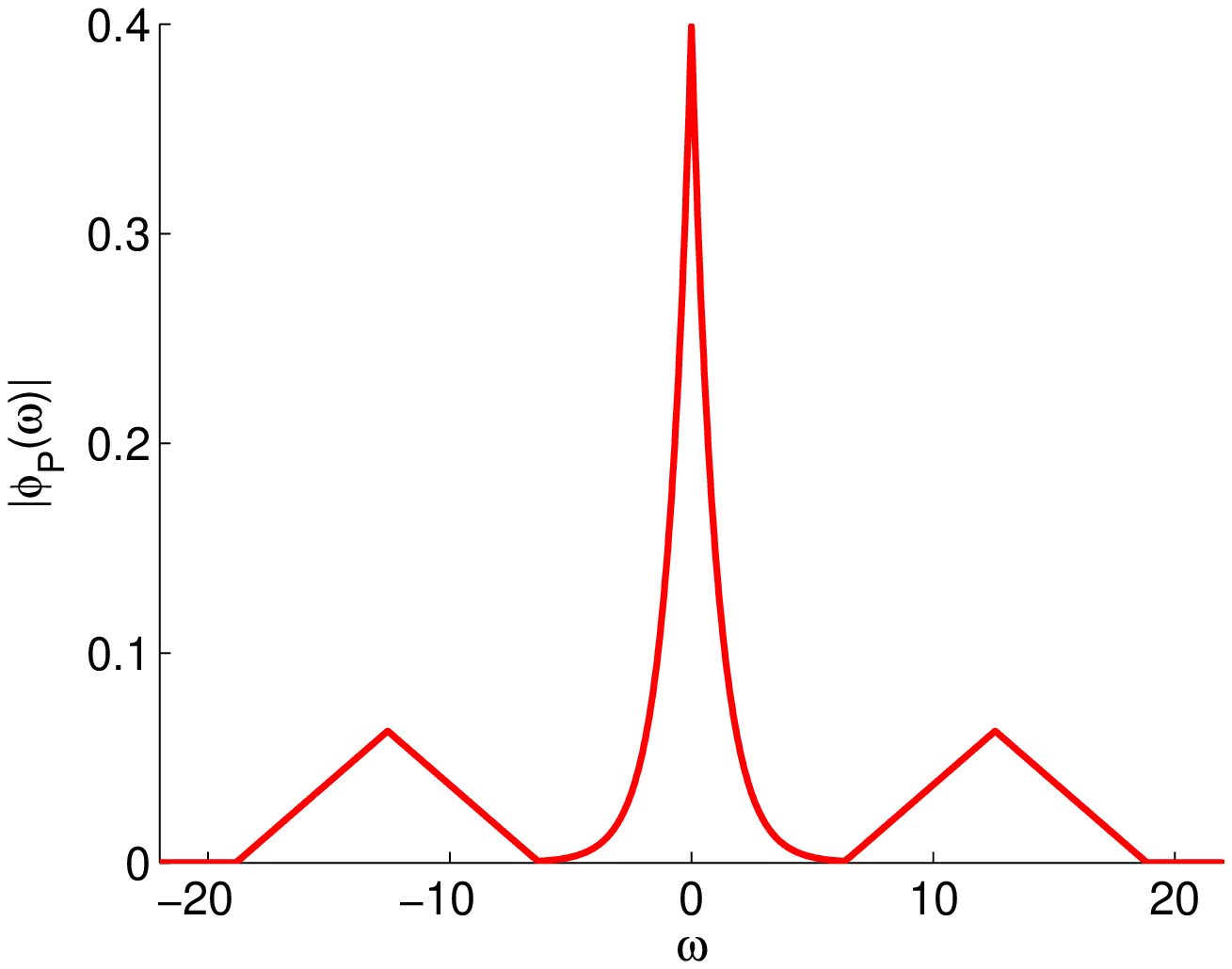}}\vspace{-3mm}
      {\small \center{(d$^\prime$)}}
    \end{minipage}
  \end{tabular}
\caption{(a-a$^\prime$) $\psi$ and its Fourier spectrum $\widehat{\psi}$, (b-b$^\prime$) $\theta^\vee$ and $i\theta$, (c-c$^\prime$) the Cauchy distribution, $q$ and its characteristic function $\phi_\bb{Q}$, and (d-d$^\prime$) $p=q+\theta^\vee$ and $|\phi_\bb{P}|$. See Example~\ref{Exm:noncompact} for details.}\vspace{-3mm}
  \label{fig:noncompact}
 \end{figure*}
We now prove Theorem~\ref{Thm:compact}.
\vspace{2mm}
\begin{proof}
\hspace{-.075in}\textbf{(Theorem~\ref{Thm:compact})} 
Suppose $\exists\, \bb{P}\ne \bb{Q},\,\bb{P},\bb{Q}\in\Scr{P}_1$ 
such that $\gamma_k(\bb{P},\bb{Q})=0$. Since any positive Borel measure on $\bb{R}^d$ is a distribution \citep[p. 157]{Rudin-91}, $\bb{P}$ and $\bb{Q}$ can be treated as distributions with compact support. By the Paley-Wiener theorem (Theorem~\ref{Thm:paley-wiener} in Appendix A), $\phi_\bb{P}$ and $\phi_\bb{Q}$ are entire on $\bb{C}^d$. Let $\theta:=\phi_\bb{P}-\phi_\bb{Q}$. Since $\gamma_k(\bb{P},\bb{Q})=0$, we have from (\ref{Eq:L2distance}) that $\int_{\bb{R}^d}|\theta(\omega)|^2\,d\Lambda(\omega)=0$. From Remark~\ref{rem:theta}, it follows that 
$\text{supp}(\theta)\subset\text{cl}(\bb{R}^d\backslash\text{supp}(\Lambda))$. Since $\text{supp}(\Lambda)$ has a non-empty interior, we have $\text{supp}(\theta)\subsetneq\bb{R}^d$. Thus, there exists an open set, $U\subset\bb{R}^d$ such that $\theta(x)=0,\,\forall\,x\in U$. Therefore, by Lemma~\ref{lem:entire} (see Appendix A), $\theta=0$, which means $\phi_\bb{P}=\phi_\bb{Q}\Rightarrow \bb{P}=\bb{Q}$, leading to a contradiction. So, there does not exist $\bb{P}\neq\bb{Q},\,\bb{P},\bb{Q}\in\Scr{P}_1$ such that $\gamma_k(\bb{P},\bb{Q})=0$ and $k$ is therefore characteristic to $\Scr{P}_1$.
\end{proof}
The condition that $\text{supp}(\Lambda)$ has a non-empty interior is important for Theorem~\ref{Thm:compact} to hold. 
In the following, we provide a simple example to show that $\bb{P}\ne \bb{Q}$, $\bb{P},\bb{Q}\in\Scr{P}_1$ can be constructed such that $\gamma_k(\bb{P},\bb{Q})=0$, if $k$ is a periodic translation invariant kernel for which $\text{int}(\text{supp}(\Lambda))=\emptyset$.
\begin{example}\label{Exm:periodic-compact}
Let $\bb{Q}$ be a uniform distribution on $[-\beta,\beta]\subset\bb{R}$, i.e., $q(x)=\frac{1}{2\beta}\mathds{1}_{[-\beta,\beta]}(x)$ with its characteristic function, $\phi_\bb{Q}(\omega)=\frac{1}{\beta\sqrt{2\pi}}\frac{\sin(\beta\omega)}{\omega}\in L^2(\bb{R})$. Let $\psi$ be the Dirichlet kernel with period $\tau$, where $\tau\le\beta$, i.e., $\psi(x)=\frac{\sin\frac{(2l+1)\pi x}{\tau}}{\sin\frac{\pi x}{\tau}}$ and $\widehat{\psi}(\omega)=\sqrt{2\pi}\sum^l_{j=-l}\delta\left(\omega-\frac{2\pi j}{\tau}\right)$ with $\emph{supp}(\widehat{\psi})=\{\frac{2\pi j}{\tau},\,j\in\{0,\pm 1,\ldots,\pm l\}\}$. Clearly, $\emph{supp}(\widehat{\psi})$ has an empty interior. Let $\theta$ be
\begin{equation}
\theta(\omega)=\frac{8\sqrt{2}\alpha}{i\sqrt{\pi}}\sin\left(\frac{\omega\tau}{2}\right)\frac{\sin^2\left(\frac{\omega\tau}{4}\right)}{\tau\omega^2},
\end{equation}
with $\alpha\le\frac{1}{2\beta}$. It is easy to verify that $\theta\in L^1(\bb{R})\cap L^2(\bb{R})\cap C_b(\bb{R})$, so $\theta$ satisfies $(i)$ in Lemma~\ref{lem:constructp}. Since $\theta(\omega)=0$ at $\omega=\frac{2\pi l}{\tau},\,l\in\bb{Z}$, $\emph{supp}(\theta)\cap\emph{supp}(\widehat{\psi})\subset\emph{supp}(\widehat{\psi})$ is a set of Lebesgue measure zero, so $(iii)$ and $(iv)$ in Lemma~\ref{lem:constructp} are satisfied. $\theta^\vee$ is given by
\begin{equation}
\theta^\vee(x)=\left\{\begin{array}{c@{\quad\quad}l}
\frac{2\alpha\left|x+\frac{\tau}{2}\right|}{\tau}-\alpha,& -\tau\le x\le 0\\
\alpha-\frac{2\alpha\left|x-\frac{\tau}{2}\right|}{\tau},& 0\le x\le \tau\\
0,&\text{otherwise,}
\end{array}\right.
\end{equation}
where $\theta^\vee\in L^1(\bb{R})\cap L^2(\bb{R})\cap C_b(\bb{R})$ satisfies $(ii)$ in Lemma~\ref{lem:constructp}. Now, consider $p=q+\theta^\vee$, which is given as
\begin{equation}
p(x)=\left\{\begin{array}{c@{\quad\quad}l}
\frac{1}{2\beta},& x\in[-\beta,-\tau]\cup[\tau,\beta]\\
\frac{2\alpha\left|x+\frac{\tau}{2}\right|}{\tau}+\frac{1}{2\beta}-\alpha,& x\in[-\tau,0]\\
\alpha+\frac{1}{2\beta}-\frac{2\alpha\left|x-\frac{\tau}{2}\right|}{\tau}, & x\in[0,\tau]\\
0,& \text{otherwise.}
\end{array}\right.
\end{equation}
Clearly, $p(x)\ge 0,\,\forall\,x$ and $\int_{\bb{R}}p(x)\,dx=1$. $\phi_\bb{P}=\phi_\bb{Q}+\theta=\phi_\bb{Q}+i\theta_I$ where $\theta_I=\emph{Im}[\theta]$ and $\phi_\bb{P}\in L^2(\bb{R})$. We have therefore constructed $\bb{P}\ne \bb{Q}$, such that $\gamma_k(\bb{P},\bb{Q})=0$, where $\bb{P}$ and $\bb{Q}$ are compactly supported in $\bb{R}$ with characteristic functions in $L^2(\bb{R})$, i.e., $\bb{P},\bb{Q}\in\Scr{P}_1$. Figure~\ref{fig:compact} shows the plots of $\psi$, $\widehat{\psi}$, $\theta$, $\theta^\vee$, $q$, $\phi_\bb{Q}$, $p$ and $|\phi_\bb{P}|$ for $\tau=2$, $l=2$, $\beta=3$ and $\alpha=\frac{1}{8}$.
\end{example}
\begin{figure*}
\vspace{-6mm}
  \centering
  \begin{tabular}{ccc}
    \begin{minipage}{7cm}
      \center{\epsfxsize=5.5cm
      \epsffile{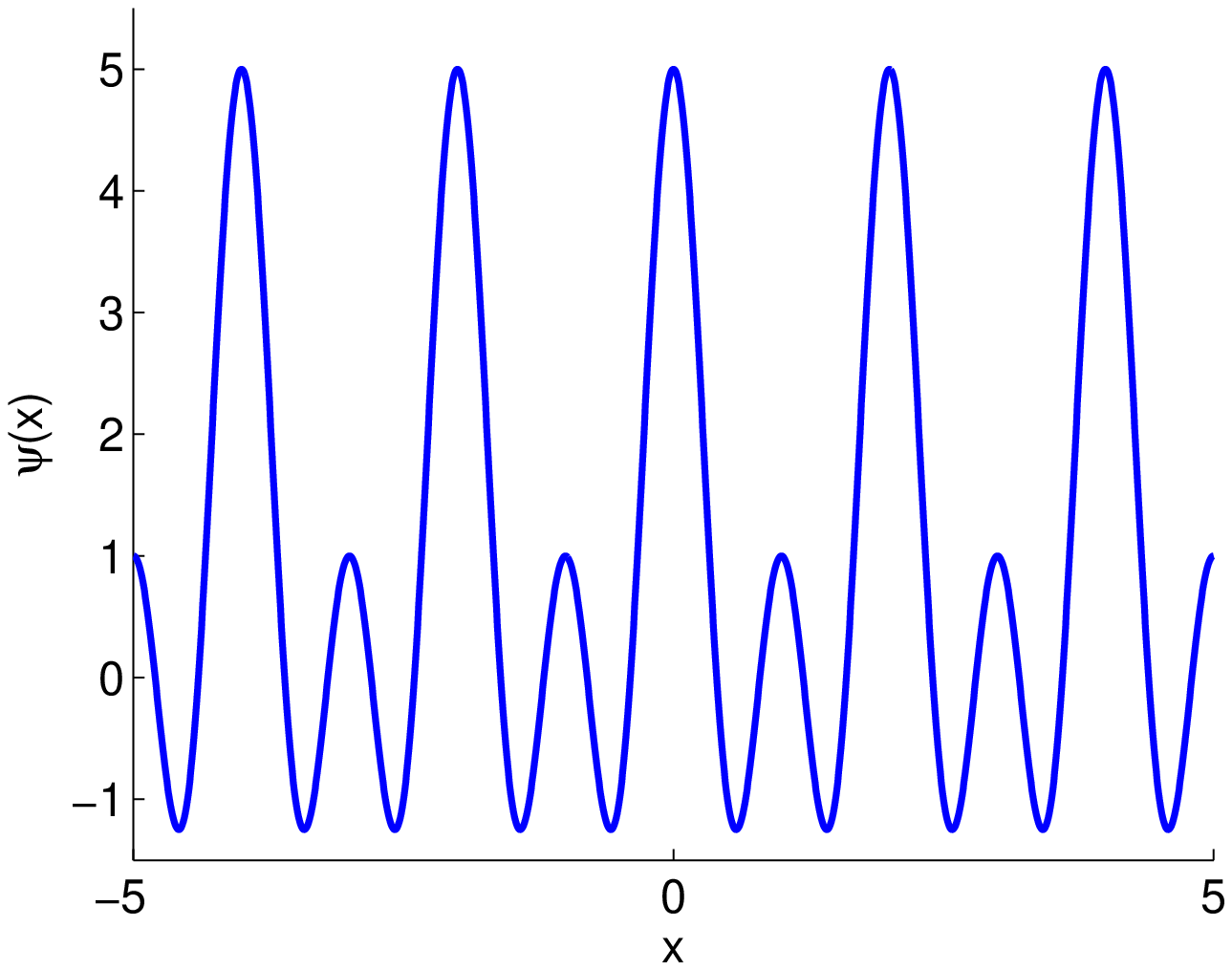}}\vspace{-5mm}
      {\small \center{(a)}}
    \end{minipage}
    \begin{minipage}{7cm}
      \center{\epsfxsize=5.5cm
      \epsffile{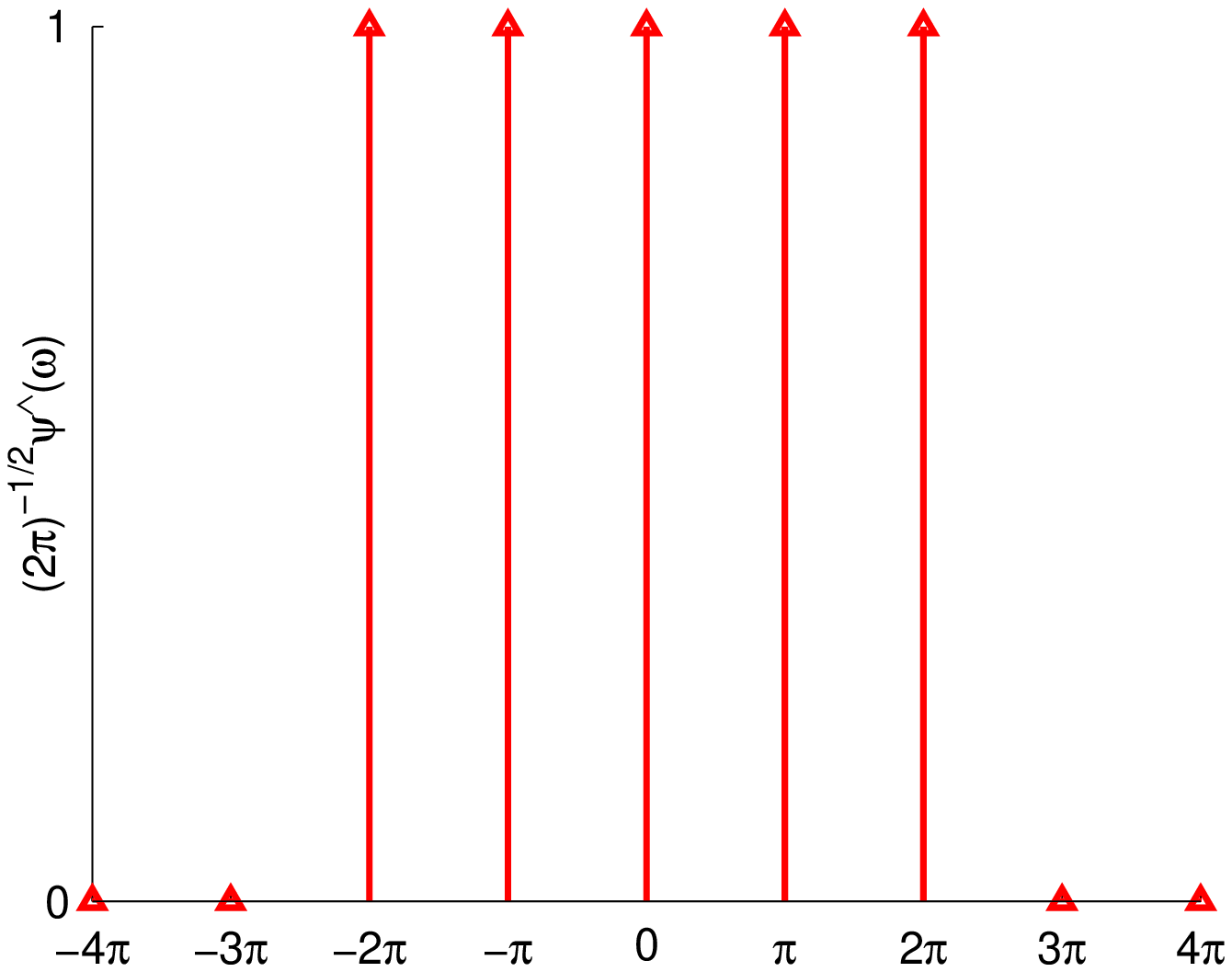}}\vspace{-3mm}
      {\small \center{(a$^\prime$)}}
    \end{minipage}
    \vspace{3mm}
  \end{tabular}
  \begin{tabular}{ccc}
    \begin{minipage}{7cm}
      \center{\epsfxsize=5.5cm
      \epsffile{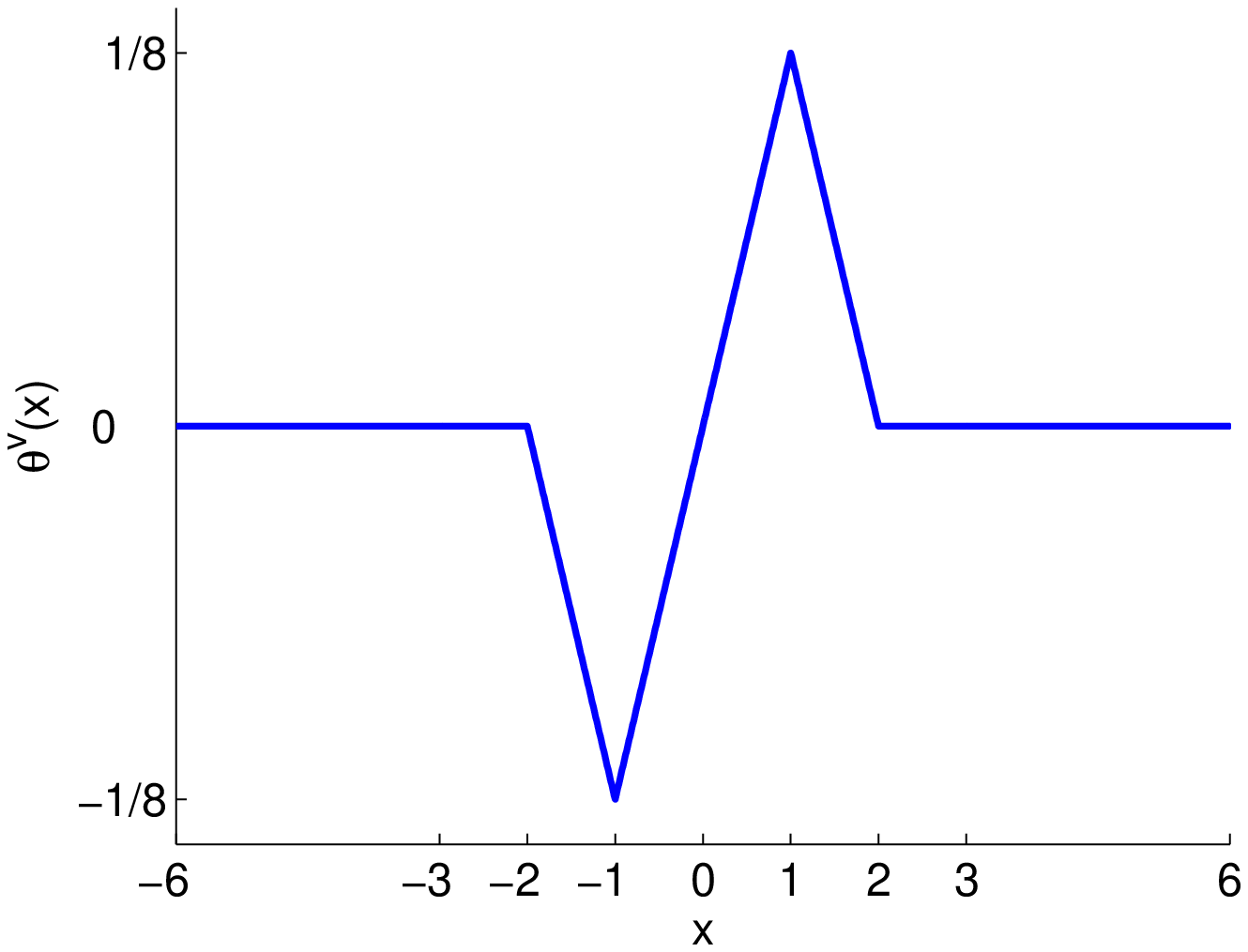}}\vspace{-3mm}
      {\small \center{(b)}}
    \end{minipage}
    \begin{minipage}{7cm}
      \center{\epsfxsize=5.5cm
      \epsffile{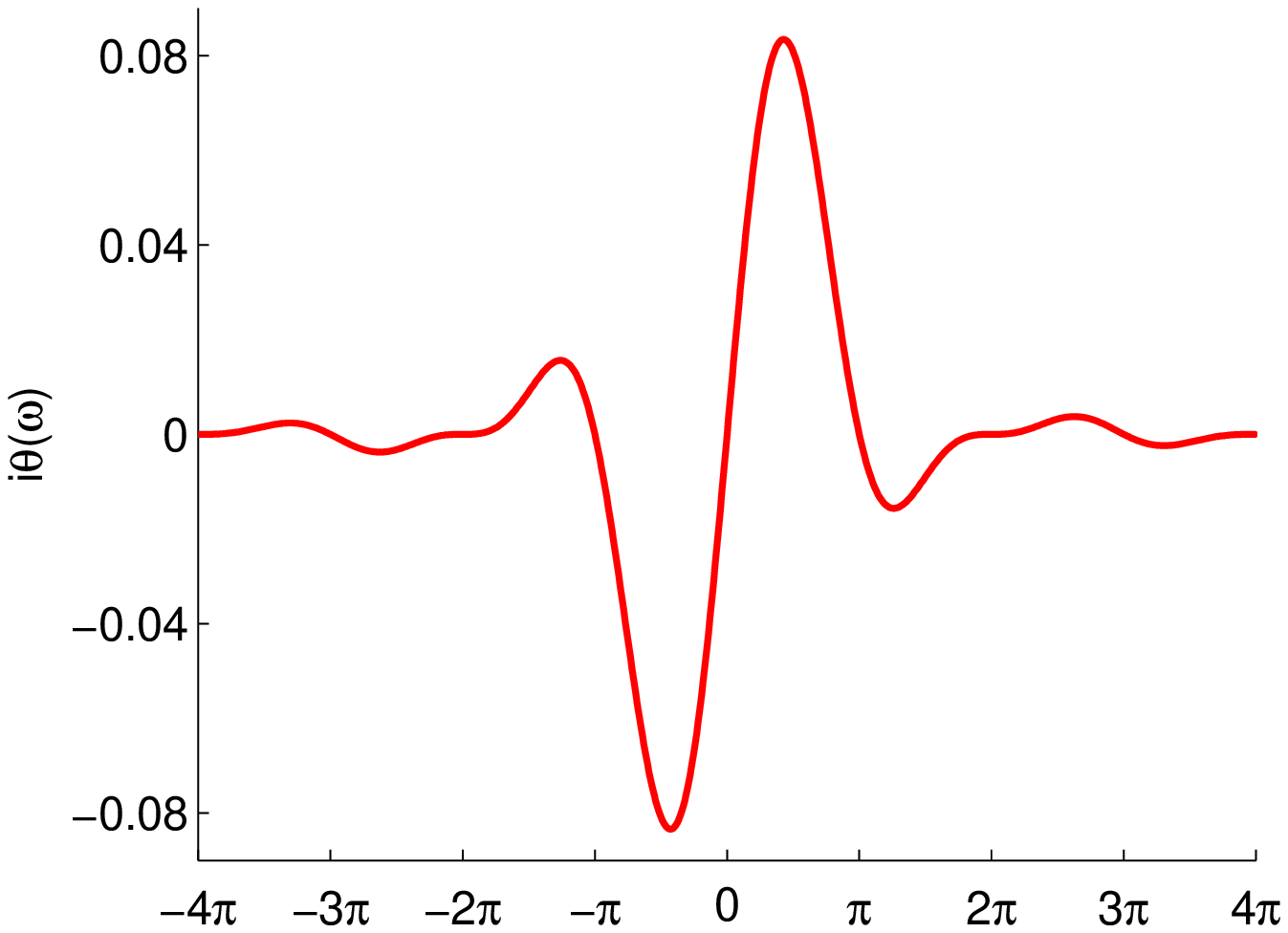}}\vspace{-3mm}
      {\small \center{(b$^\prime$)}}
    \end{minipage}
    \vspace{3mm}
  \end{tabular}
  \begin{tabular}{ccc}
    \begin{minipage}{7cm}
      \center{\epsfxsize=5.5cm
      \epsffile{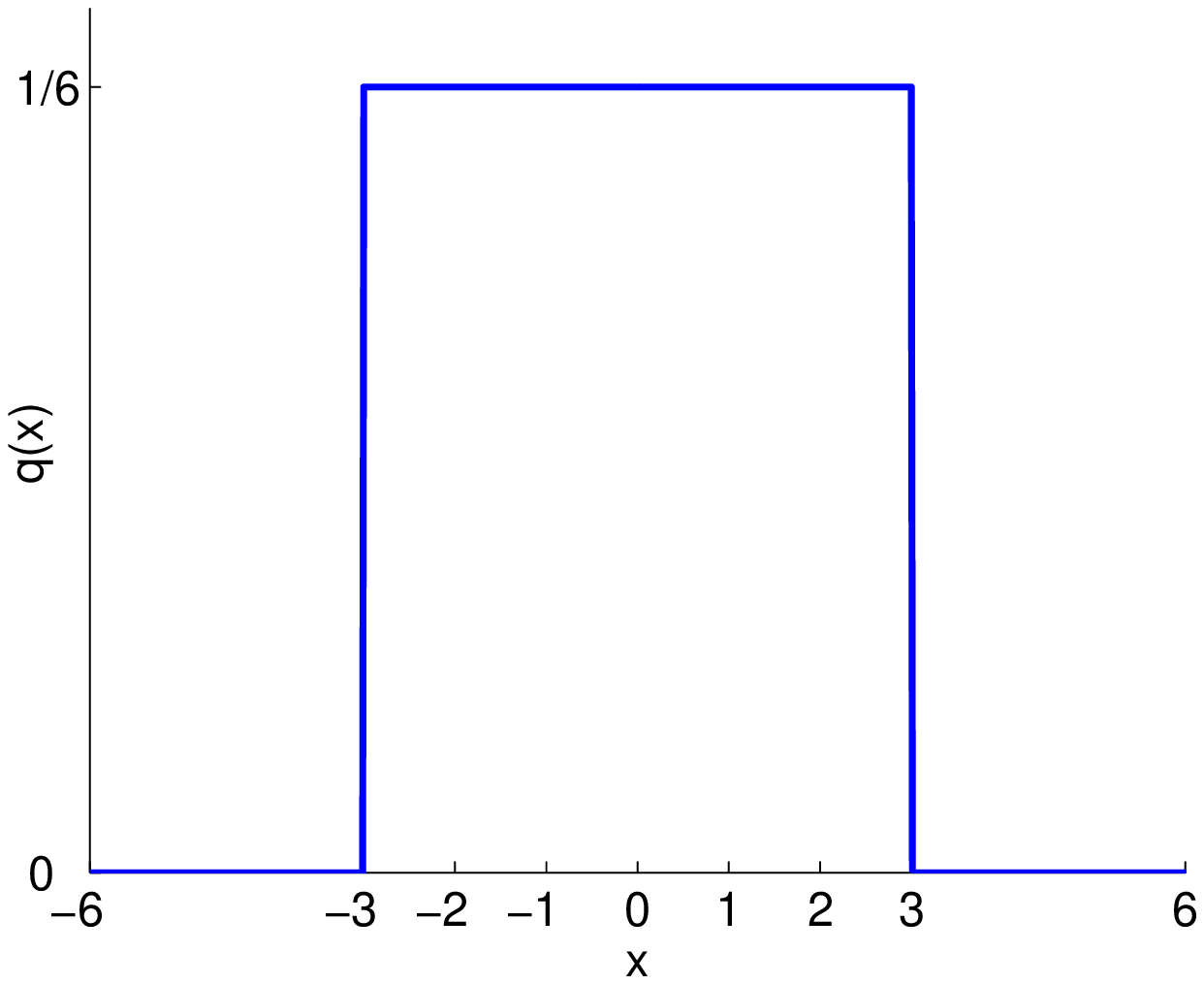}}\vspace{-3mm}
      {\small \center{(c)}}
    \end{minipage}
    \begin{minipage}{7cm}
      \center{\epsfxsize=5.5cm
      \epsffile{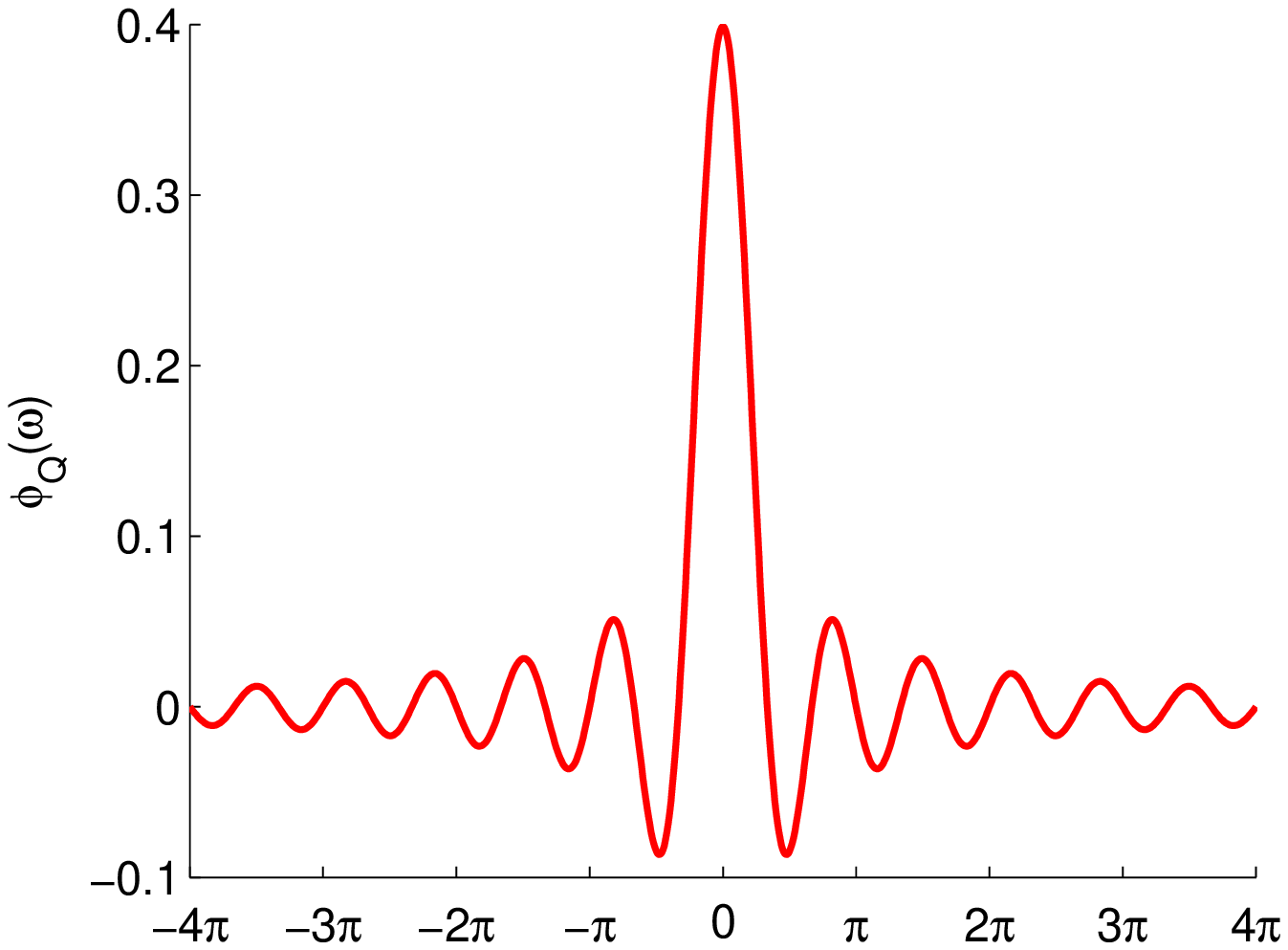}}\vspace{-3mm}
      {\small \center{(c$^\prime$)}}
    \end{minipage}
    \vspace{3mm}
  \end{tabular}
\begin{tabular}{ccc}
    \begin{minipage}{7cm}
      \center{\epsfxsize=5.5cm
      \epsffile{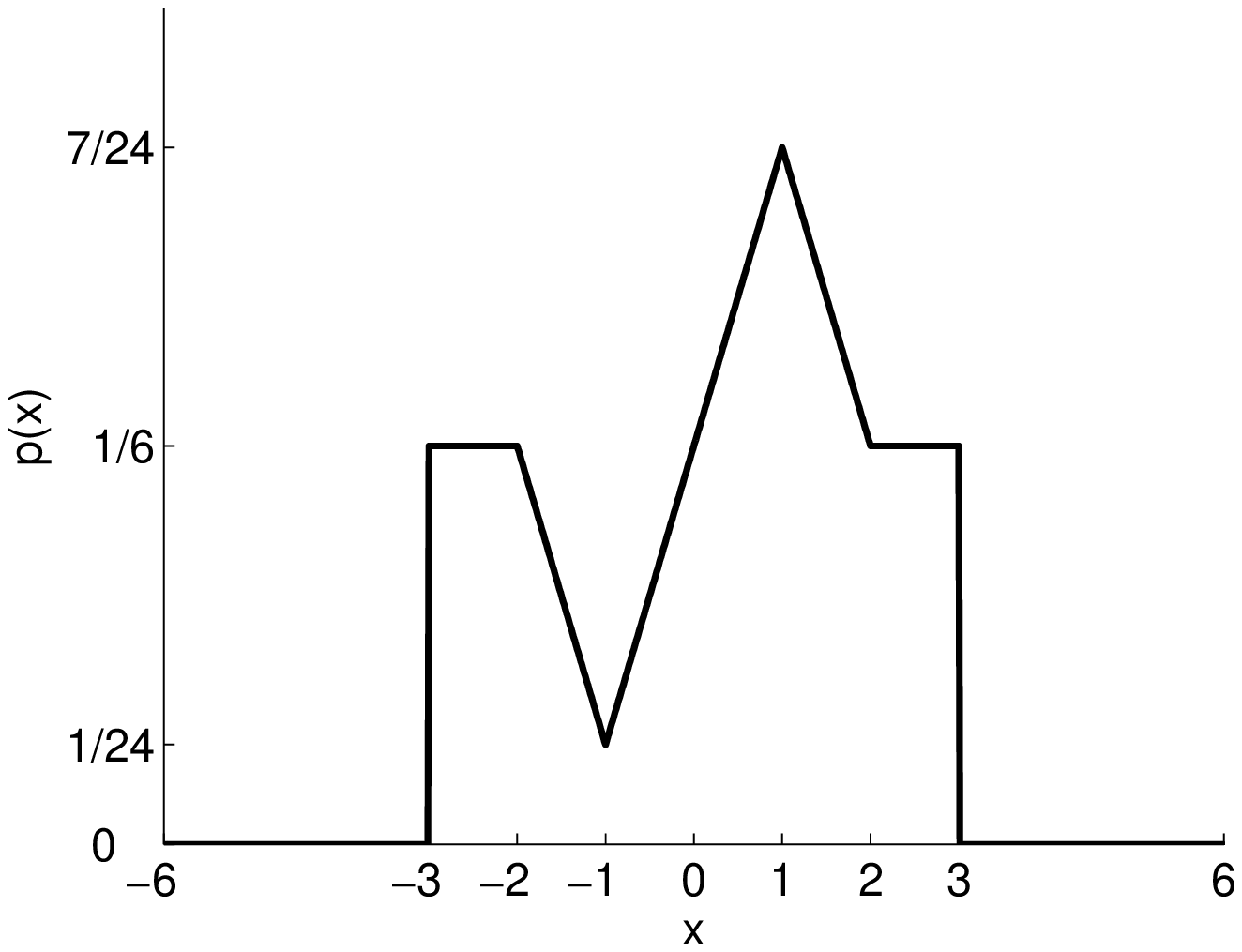}}\vspace{-3mm}
      {\small \center{(d)}}
    \end{minipage}
    \begin{minipage}{7cm}
      \center{\epsfxsize=5.5cm
      \epsffile{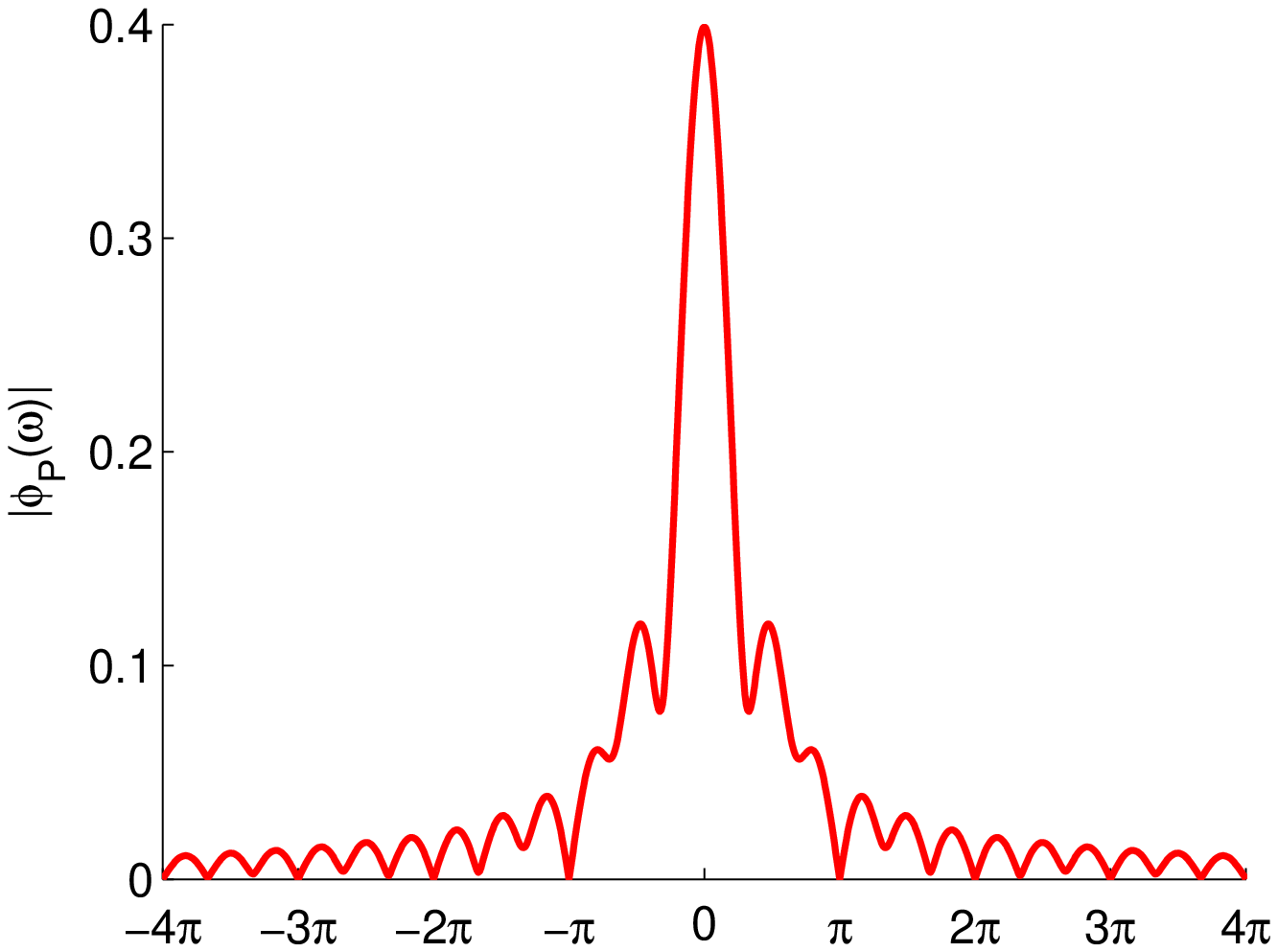}}\vspace{-3mm}
      {\small \center{(d$^\prime$)}}
    \end{minipage}
  \end{tabular}
  \caption{(a-a$^\prime$) $\psi$ and its Fourier spectrum $\widehat{\psi}$, (b-b$^\prime$) $\theta^\vee$ and $i\theta$, (c-c$^\prime$) the uniform distribution, $q$ and its characteristic function $\phi_\bb{Q}$, and (d-d$^\prime$) $p=q+\theta^\vee$ and $|\phi_\bb{P}|$. See Example~\ref{Exm:periodic-compact} for details.}\vspace{-3mm}
  \label{fig:compact}
 \end{figure*}
We now present the proof of Theorem~\ref{Thm:torus}, which is similar to that of Theorem~\ref{Thm:countablyinfinite}.
\vspace{2mm}
\begin{proof}
\hspace{-.05in}\textbf{(Theorem~\ref{Thm:torus})} ($\,\Leftarrow\,$) From (\ref{Eq:computeMMD-1}), we have
\begin{eqnarray}\label{Eq:Td-charactersistic}
\gamma^2_k(\bb{P},\bb{Q})&\!\!\!=\!\!\!&\int\!\!\!\int_{\bb{T}^d}\psi(x-y)\,d(\bb{P}-\bb{Q})(x)\,d(\bb{P}-\bb{Q})(y)\nonumber\\
&\!\!\!\stackrel{(a)}{=}\!\!\!&\int\!\!\!\int_{\bb{T}^d}\sum_{n\in\bb{Z}^d}A_\psi(n)\,e^{i(x-y)^Tn}\,d(\bb{P}-\bb{Q})(x)\,d(\bb{P}-\bb{Q})(y)\nonumber\\
&\!\!\!\stackrel{(b)}{=}\!\!\!&\sum_{n\in\bb{Z}^d}A_\psi(n)\left|\int_{\bb{T}^d}e^{-ix^Tn}\,d(\bb{P}-\bb{Q})(x)\right|^2\nonumber\\
&\!\!\!\stackrel{(c)}{=}\!\!\!&(2\pi)^{2d}\sum_{n\in\bb{Z}^d}A_\psi(n)\left|A_\bb{P}(n)-A_\bb{Q}(n)\right|^2,\label{Eq:L2-discrete}
\end{eqnarray} 
where we have invoked Bochner's theorem (Theorem~\ref{Theorem:Bochnerdiscrete}) in $(a)$, Fubini's theorem in $(b)$ and 
\begin{equation}
A_\bb{P}(n):=\frac{1}{(2\pi)^d}\int_{\bb{T}^d}e^{-in^Tx}\,d\bb{P}(x),\,n\in\bb{Z},
\end{equation}
in $(c)$. $A_\bb{P}$ is the Fourier transform of $\bb{P}$ in $\bb{T}^d$.
Since $A_\psi(0)\ge 0$ and $A_\psi(n)>0,\,\forall\,n\ne 0$, we have $A_\bb{P}(n)=A_\bb{Q}(n),\,\forall\,n$. Therefore, by the uniqueness theorem of Fourier transform, we have $\bb{P}=\bb{Q}$.\vspace{2mm}\\
($\,\Rightarrow\,$) Proving the necessity is equivalent to proving that if $A_\psi(0)\ge 0,\,A_\psi(n)>0,\,\forall\,n\ne 0$ is violated, then $k$ is not characteristic, which is equivalent to showing that $\exists\,\bb{P}\ne \bb{Q}$ such that $\gamma_k(\bb{P},\bb{Q})=0$. Let $\bb{Q}$ be a uniform probability measure with $q(x)=\frac{1}{(2\pi)^d},\,\forall\,x\in\bb{T}^d$. Let $k$ be such that $A_\psi(n)=0$ for some $n=n_0\ne 0$. Define
\begin{equation}\label{Eq:A_p}
A_\bb{P}(n):=\left\{\begin{array}{c@{\quad\quad}l}
A_\bb{Q}(n),& n\ne \pm n_0\\
A_\bb{Q}(n)+\theta(n),& n=\pm n_0
\end{array}\right.,
\end{equation} 
where $A_\bb{Q}(n)=\frac{1}{(2\pi)^d}\delta_{n_0}$ and $\theta(-n_0)=\overline{\theta(n_0)}$. So, 
\begin{equation}
p(x)=\sum_{n\in\bb{Z}^d}A_P(n)e^{ix^Tn}=\frac{1}{(2\pi)^d}+\theta(n_0)e^{ix^Tn_0}+\theta(-n_0)e^{-ix^Tn_0}.
\end{equation}
Choose $\theta(n_0)=i\alpha$, $\alpha\in\bb{R}$. Then, $p(x)=\frac{1}{(2\pi)^d}-2\alpha\sin(x^Tn_0)$. It is easy to check that $p$ integrates to one. Choosing $|\alpha|\le\frac{1}{2(2\pi)^d}$ ensures that $p(x)\ge 0,\forall\,x\in\bb{T}^d$. By using $A_\bb{P}(n)$ in (\ref{Eq:Td-charactersistic}), it is clear that $\gamma_k(\bb{P},\bb{Q})=0$. Therefore, $\exists\,\bb{P}\ne \bb{Q}$ such that $\gamma_k(\bb{P},\bb{Q})=0$, which means $k$ is not characteristic.\vspace{-3mm}
\end{proof}

\section{Dissimilar Distributions with Small $\gamma_k$}\label{Sec:limitation}
\par So far, we have studied different characterizations for the kernel $k$ such that $\gamma_k$ is a metric on $\Scr{P}$. As mentioned in Section~\ref{Sec:Introduction}, the metric property of $\gamma_k$ is crucial in many statistical inference applications like hypothesis testing. Therefore, in practice, it is important to use characteristic kernels.  
However, in this section, we show that characteristic kernels, while guaranteeing $\gamma_k$ to be a metric on $\Scr{P}$, may nonetheless have difficulty in distinguishing certain distributions on the basis of finite samples. More specifically, in Theorem~\ref{Thm:generic} we show that for a given kernel, $k$ and for any $\varepsilon>0$, there exist $\bb{P}\ne\bb{Q}$ such that $\gamma_k(\bb{P},\bb{Q})<\varepsilon$. 
Before proving the result, we motivate it through the following example.
\begin{example}\label{Exm:spline}
Let $\bb{P}$ be absolutely continuous w.r.t. the Lebesgue measure on $\bb{R}$ with the Radon-Nikodym derivative defined as 
\begin{equation}\label{Eq:lim}
p(x)=q(x)+\alpha q(x)\sin(\nu\pi x),
\end{equation}
where $q$ is the Radon-Nikodym derivative of $\bb{Q}$ w.r.t. the Lebesgue measure satisfying $q(x)=q(-x),\,\forall\,x$ and $\alpha\in[-1,1]\backslash\{0\},\,\nu\in\bb{R}\backslash\{0\}$. It is obvious that $\bb{P}\ne\bb{Q}$. The characteristic function of $\bb{P}$ is given as 
\begin{equation}
\phi_\bb{P}(\omega)=\phi_\bb{Q}(\omega)-\frac{i\alpha}{2}\left[\phi_\bb{Q}(\omega-\nu\pi)-\phi_\bb{Q}(\omega+\nu\pi)\right],\,\omega\in\bb{R},
\end{equation}
where $\phi_\bb{Q}$ is the characteristic function associated with $\bb{Q}$. Note that with increasing $|\nu|$, $p$ has higher frequency components in its Fourier spectrum
and therefore appears more noisy as shown in Figure~\ref{fig:noisy}. In Figure~\ref{fig:noisy}, (a-c) show the plots of $p$ when $q=\eu{U}[-1,1]$ (uniform distribution) and (a$^\prime$-c$^\prime$) show the plots of $p$ when $q=\eu{N}(0,2)$ (zero mean normal distribution with variance $2$) for $\nu=0,2$ and $7.5$ with $\alpha=\frac{1}{2}$.
\par Consider the $B_1$-spline kernel on $\bb{R}$ given by $k(x,y)=\psi(x-y)$ where
\begin{equation}\label{Eq:B1spline}
\psi(x)=\left\{\begin{array}{c@{\quad\quad}l}
1-|x|,& |x|\le 1\\
0,& \text{otherwise}
\end{array}\right.,
\end{equation} 
with its Fourier transform given by \begin{equation}
\widehat{\psi}(\omega)=\frac{2\sqrt{2}}{\sqrt{\pi}}\frac{\sin^2\frac{\omega}{2}}{\omega^2}.
\end{equation}
Since $\psi$ is characteristic to $\Scr{P}$, $\gamma_k(\bb{P},\bb{Q})>0$ (see Theorem~\ref{Thm:countablyinfinite}). However, it would be of interest to study the behavior of $\gamma_k(\bb{P},\bb{Q})$ as a function of $\nu$. We study the behavior of $\gamma^2_k(\bb{P},\bb{Q})$ through its unbiased, consistent estimator,\footnote{Let $\{X_j\}^m_{j=1}$ and $\{Y_j\}^m_{j=1}$ be random samples drawn i.i.d. from $\bb{P}$ and $\bb{Q}$ respectively. An unbiased \emph{empirical estimate} of $\gamma^2_k(\bb{P},\bb{Q})$, denoted as $\gamma^2_{k,u}(m,m)$ is given by $\gamma^2_{k,u}(m,m)=\frac{1}{m(m-1)}\sum^m_{l\ne j}h(Z_l,Z_j)$, which is a one-sample $U$-statistic with $h(Z_l,Z_j):=k(X_l,X_j)+k(Y_l,Y_j)-k(X_l,Y_j)-k(X_j,Y_l)$, where $Z_1,\ldots,Z_m$ are $m$ i.i.d. random variables with $Z_j:=(X_j,Y_j)$. See \citet[Lemma 7]{Gretton-06} for details.} $\gamma^2_{k,u}(m,m)$ as considered by \citet[Lemma 7]{Gretton-06}. 
\par Figure~\ref{fig:spline-gaussian}(a) shows the behavior of $\gamma^2_{k,u}(m,m)$ as a function of $\nu$ for $q=\eu{U}[-1,1]$ and $q=\eu{N}(0,2)$ using the $B_1$-spline kernel in (\ref{Eq:B1spline}). Since the Gaussian kernel, $k(x,y)=e^{-(x-y)^2}$ is also a characteristic kernel, its effect on the behavior of $\gamma^2_{k,u}(m,m)$ is shown in Figure~\ref{fig:spline-gaussian}(b) in comparison to that of the $B_1$-spline kernel. 
\begin{figure}[t]
  \centering
  \begin{tabular}{ccc}
    \begin{minipage}{5cm}
      \center{\epsfxsize=5cm
      \epsffile{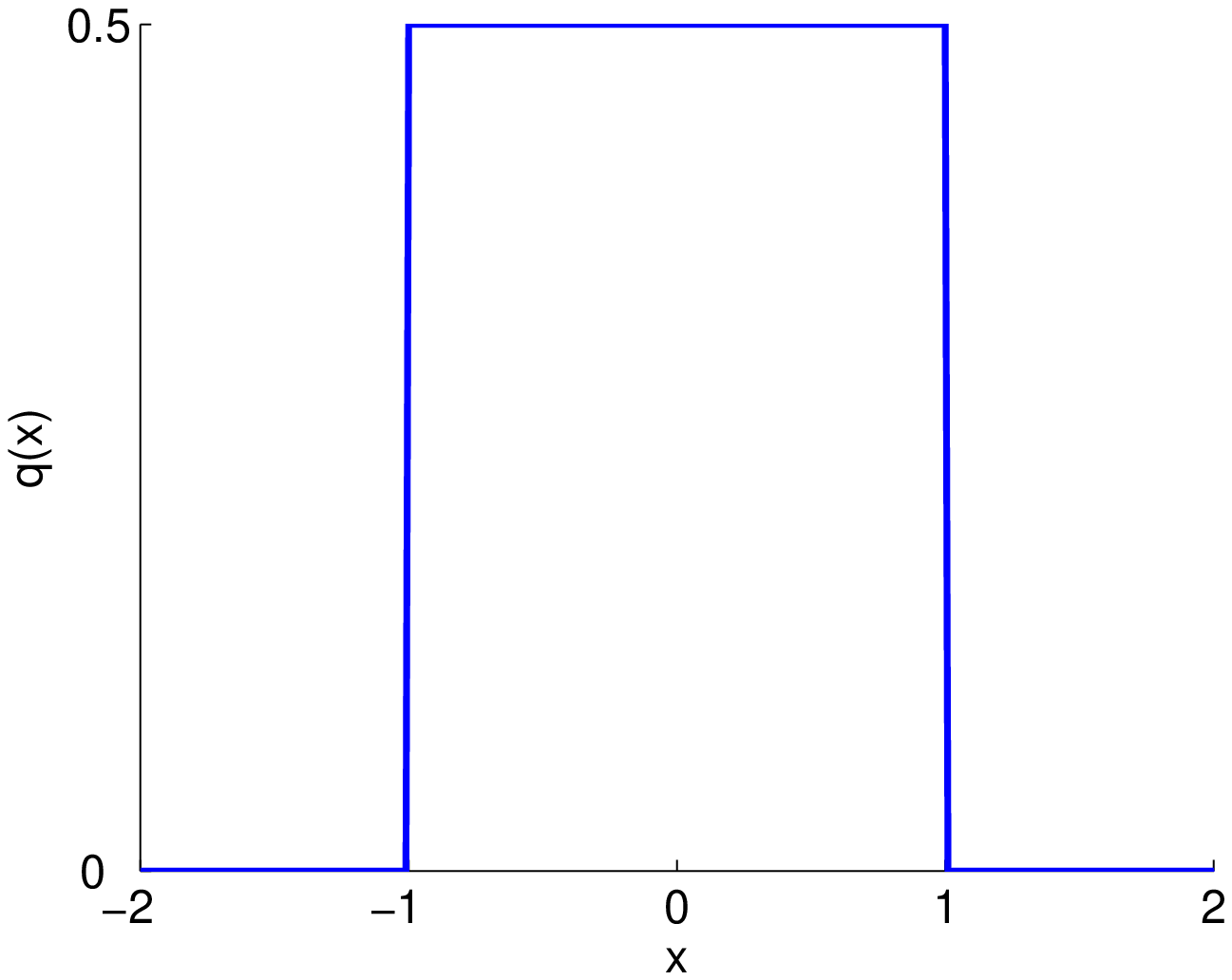}}\vspace{-4mm}
      {\small \center{(a)}}
    \end{minipage}
    \begin{minipage}{5.25cm}
      \center{\epsfxsize=5cm
      \epsffile{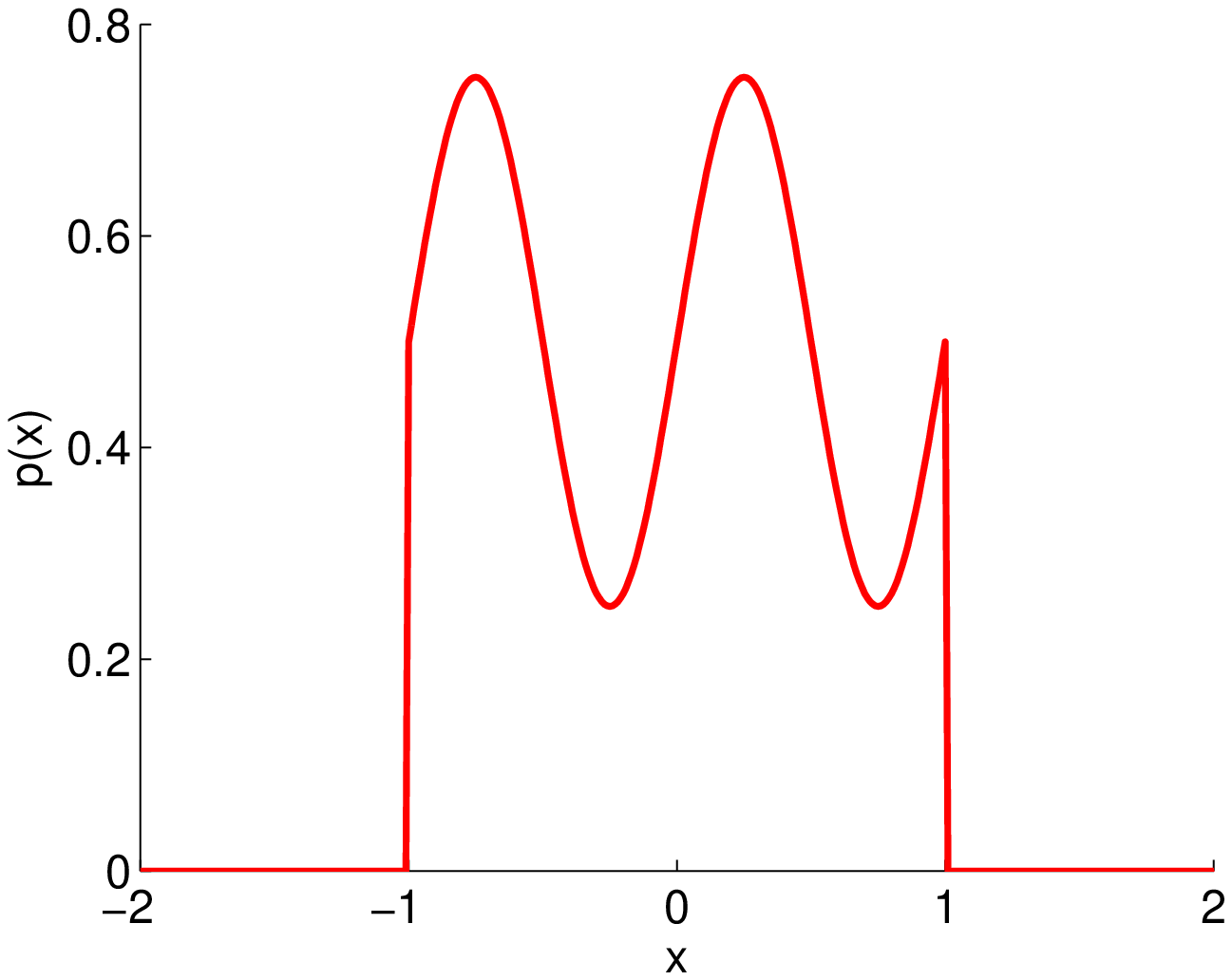}}\vspace{-4mm}
      {\small \center{(b)}}
    \end{minipage}
    \begin{minipage}{5cm}
      \center{\epsfxsize=5cm
      \epsffile{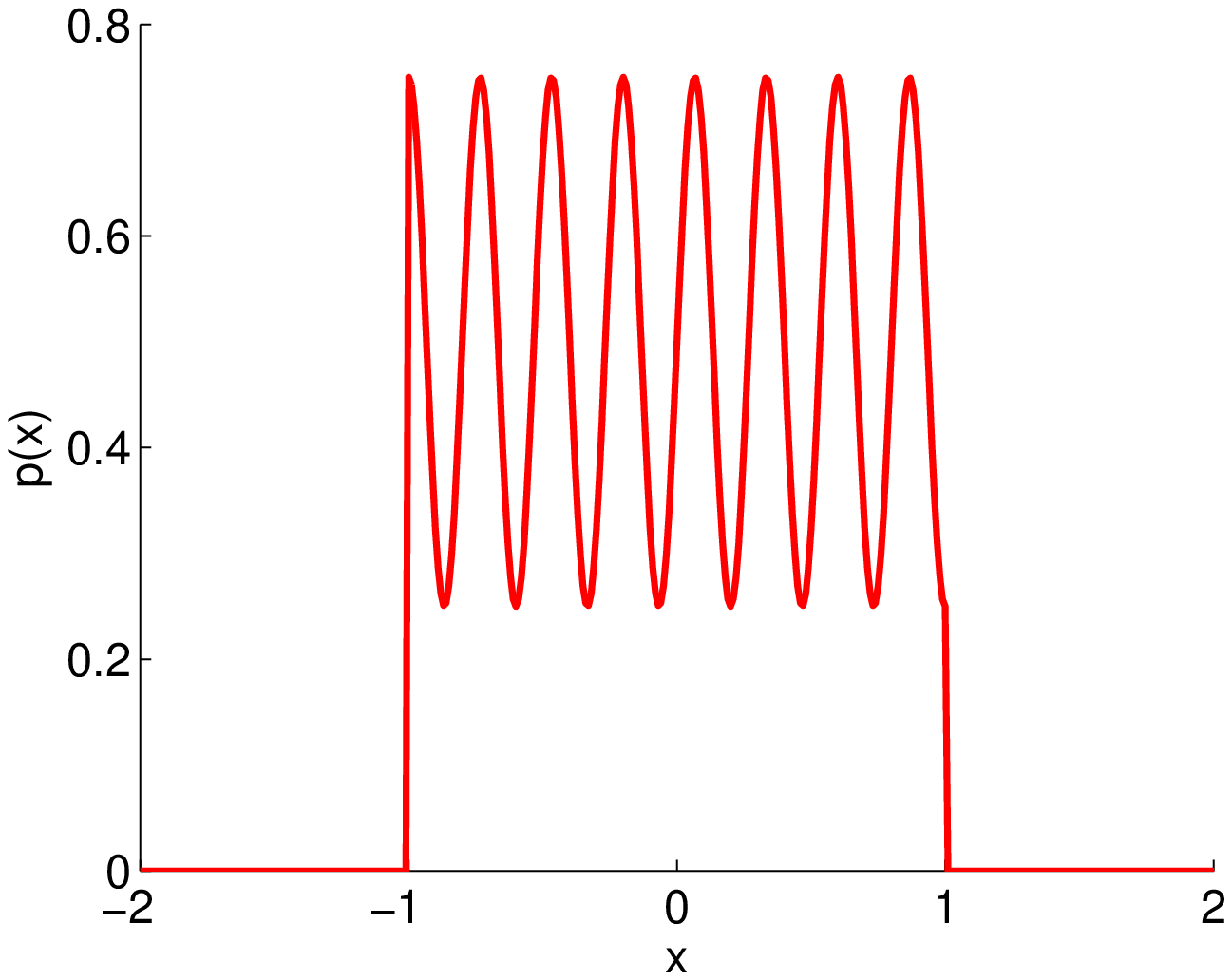}}\vspace{-4mm}
      {\small \center{(c)}}
    \end{minipage}
  \end{tabular}
  \begin{tabular}{ccc}
    \begin{minipage}{5cm}
      \center{\epsfxsize=5cm
      \epsffile{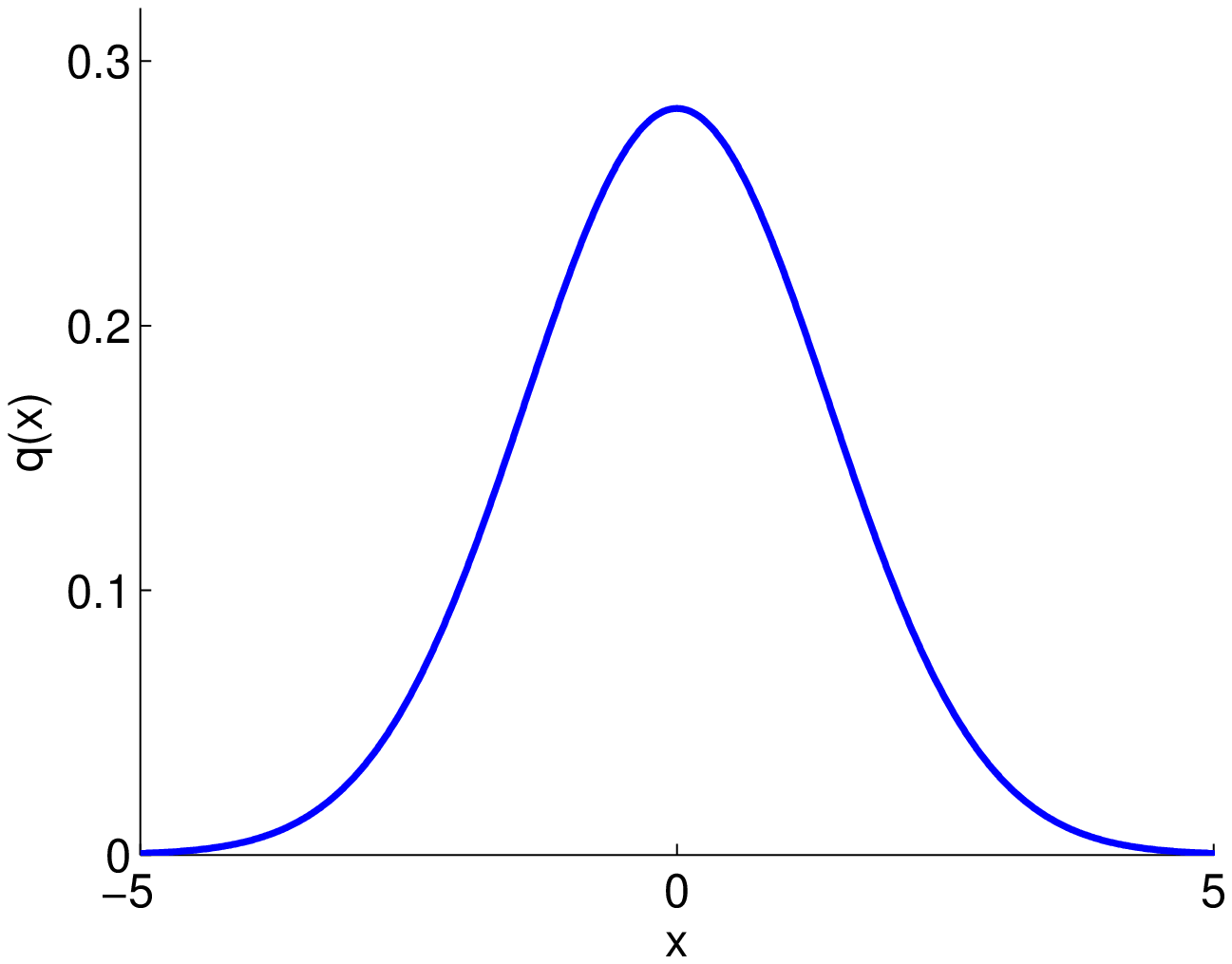}}\vspace{-4mm}
      {\small \center{(a$^\prime$)}}
    \end{minipage}
    \begin{minipage}{5cm}
      \center{\epsfxsize=5cm
      \epsffile{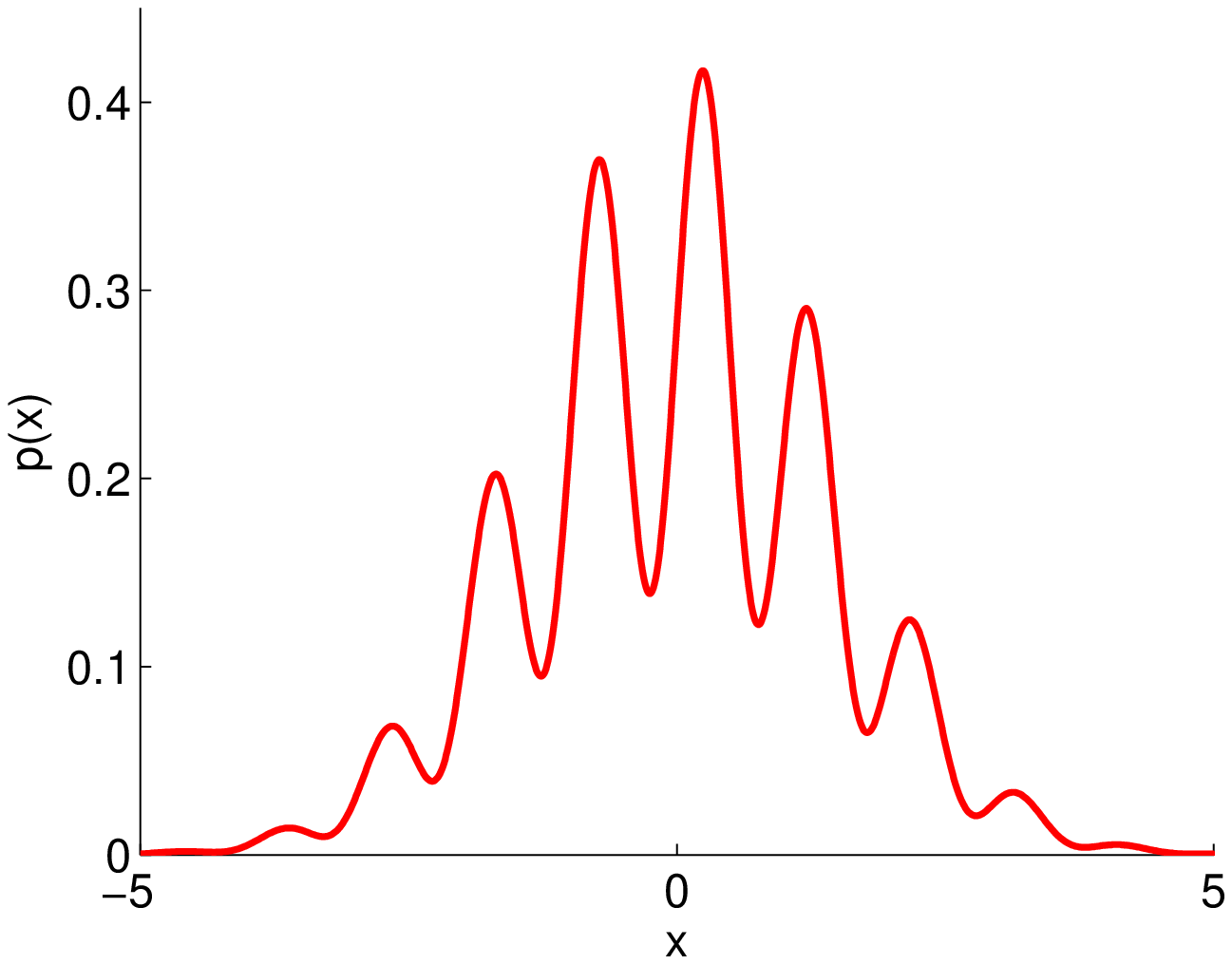}}\vspace{-4mm}
      {\small \center{(b$^\prime$)}}
    \end{minipage}
    \begin{minipage}{5cm}
      \center{\epsfxsize=5cm
      \epsffile{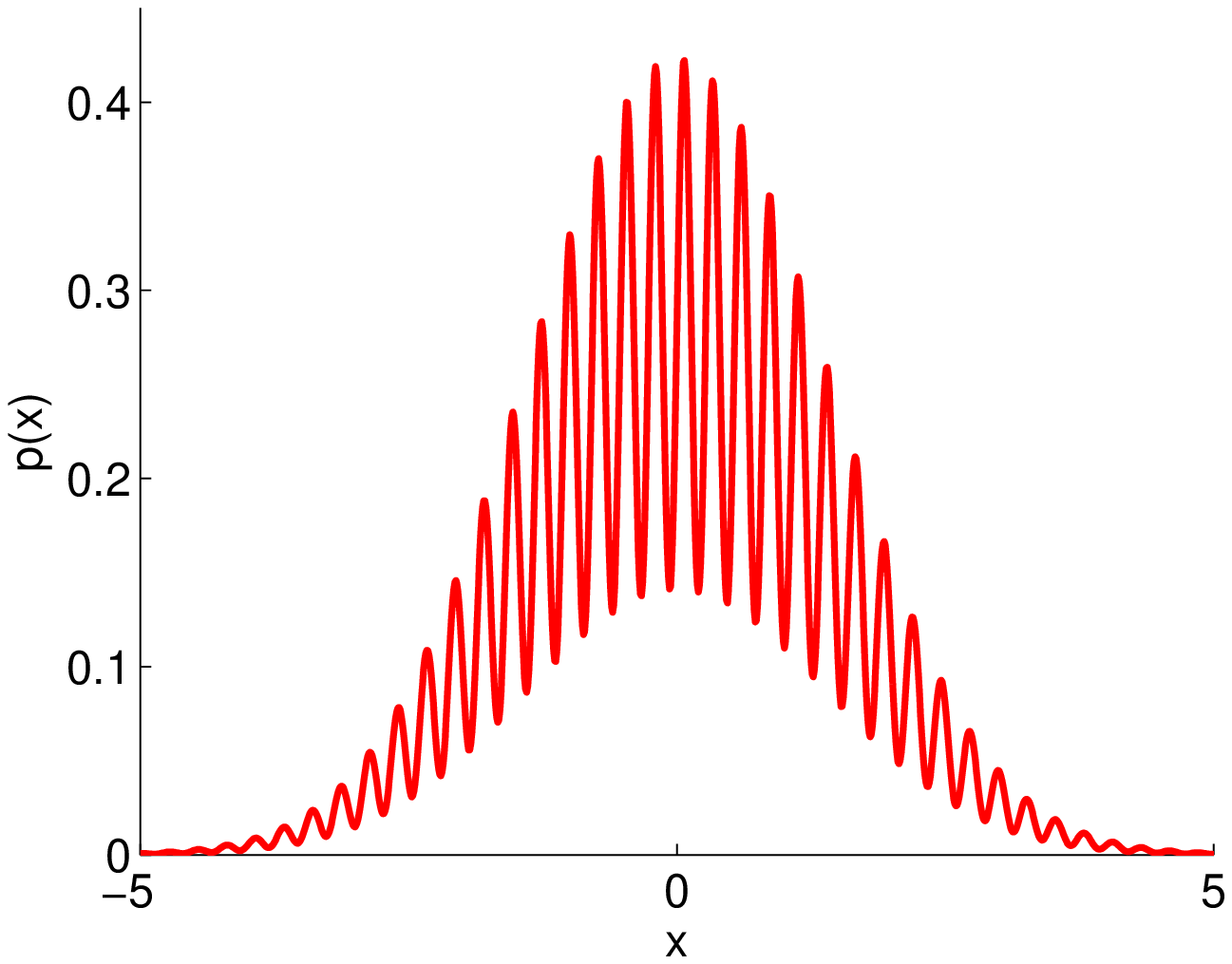}}\vspace{-4mm}
      {\small \center{(c$^\prime$)}}
    \end{minipage}
  \end{tabular}
  \vspace{-2mm}
  \caption{(a) $q=\eu{U}[-1,1]$, (a$^\prime$) $q=\eu{N}(0,2)$. (b-c) and (b$^\prime$-c$^\prime$) denote $p(x)$ computed as $p(x)=q(x)+\frac{1}{2}q(x)\sin(\nu\pi x)$ with $q=\eu{U}[-1,1]$ and $q=\eu{N}(0,2)$ respectively. $\nu$ is chosen to be $2$ in (b,b$^\prime$) and $7.5$ in (c,c$^\prime$). See Example~\ref{Exm:spline} for details.}\vspace{-4mm}
  \label{fig:noisy}
\end{figure}
\begin{figure}[h]
  \centering
  \begin{tabular}{cc}
    \begin{minipage}{8cm}
      \center{\epsfxsize=6cm
      \epsffile{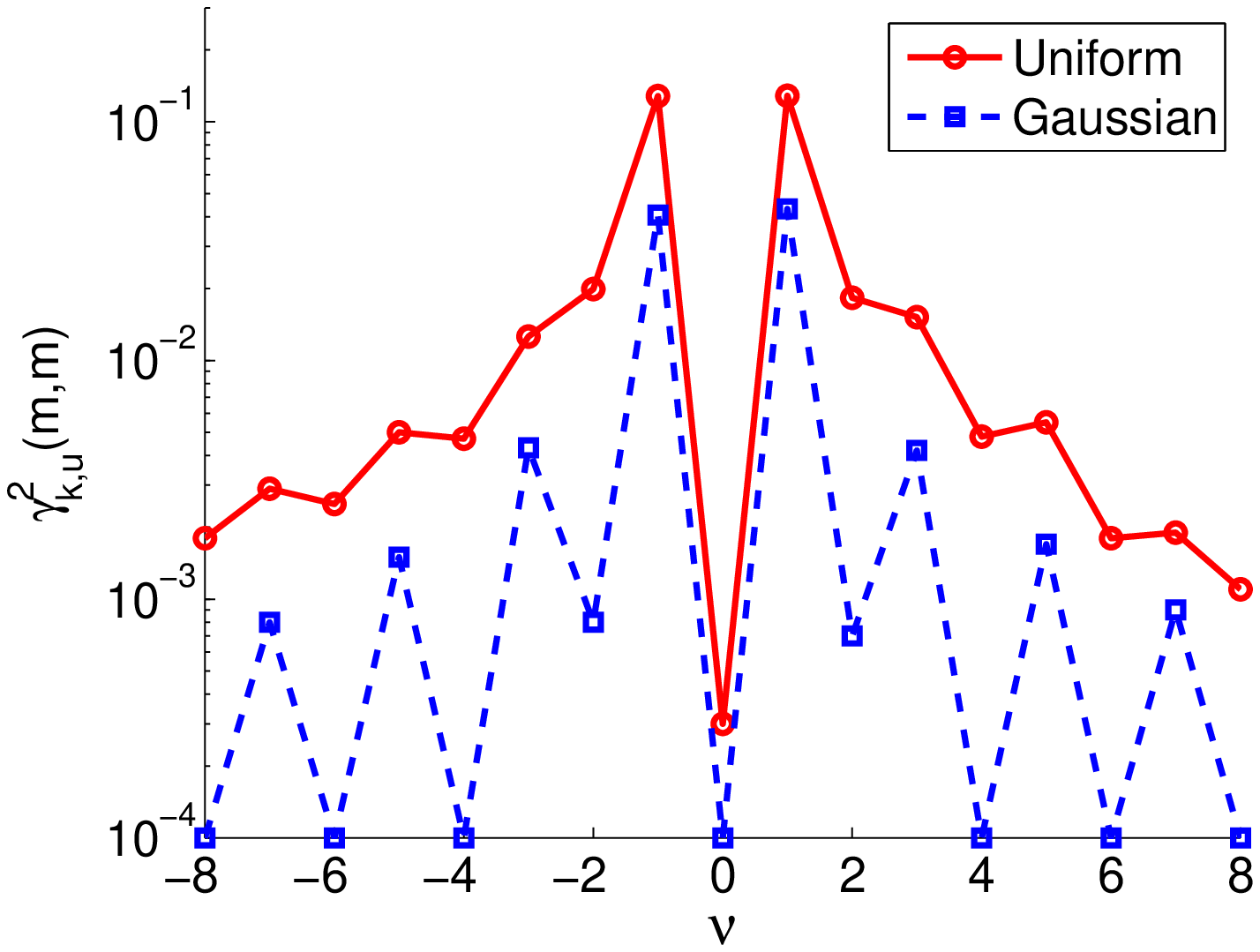}}\vspace{-4mm}
      {\small \center{(a)}}
    \end{minipage}
    \begin{minipage}{8cm}
      \center{\epsfxsize=6cm
      \epsffile{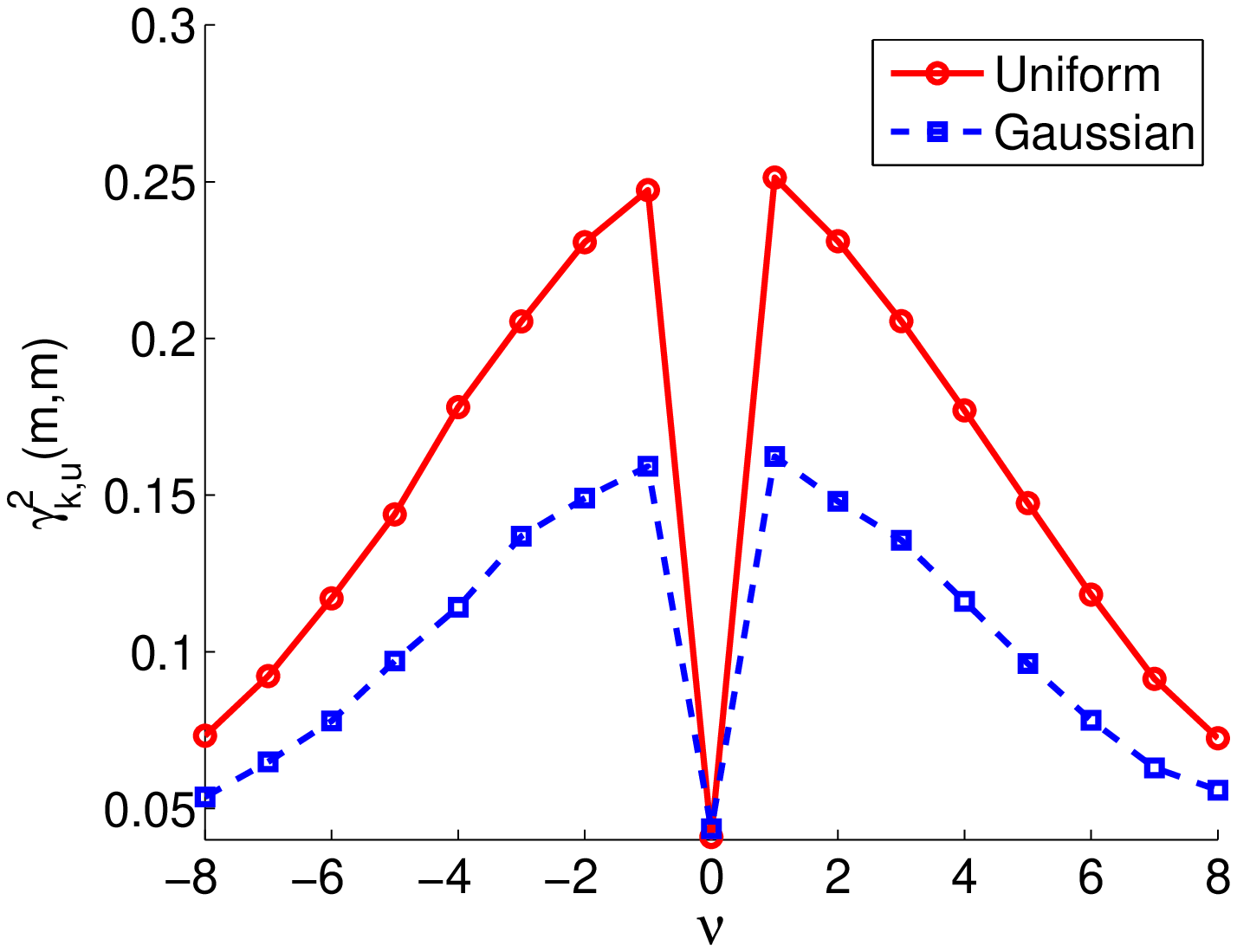}}\vspace{-4mm}
      {\small \center{(b)}}
    \end{minipage}
  \end{tabular}
  \vspace{-2mm}
  \caption{Behavior of the empirical estimate of $\gamma^2_k(\bb{P},\bb{Q})$ w.r.t. $\nu$ for (a) the $B_1$-spline kernel and (b) the Gaussian kernel. $\bb{P}$ is constructed from $\bb{Q}$ as defined in (\ref{Eq:lim}). ``Uniform" corresponds to $\bb{Q}=\eu{U}[-1,1]$ and ``Gaussian" corresponds to $\bb{Q}=\eu{N}(0,2)$. $m=1000$ samples are generated from $\bb{P}$ and $\bb{Q}$ to estimate $\gamma^2_k(\bb{P},\bb{Q})$ through $\gamma^2_{k,u}(m,m)$. This is repeated $100$ times and the average $\gamma^2_{k,u}(m,m)$ is plotted in both figures. Since the quantity of interest is the average behavior of $\gamma^2_{k,u}(m,m)$, we omit the error bars. See Example~\ref{Exm:spline} for details.}\vspace{-7mm}
  \label{fig:spline-gaussian}
\end{figure}
\par In Figure~\ref{fig:spline-gaussian}, we observe two circumstances under which $\gamma^2_k$ may be small. First, $\gamma^2_{k,u}(m,m)$ decays with increasing $|\nu|$, and can be made as small as desired by choosing a sufficiently large $|\nu|$. Second, in Figure~\ref{fig:spline-gaussian}(a), $\gamma^2_{k,u}(m,m)$ has troughs at $\nu=\frac{\omega_0}{\pi}$ where $\omega_0=\{\omega:\widehat{\psi}(\omega)=0\}$. Since $\gamma^2_{k,u}(m,m)$ is a consistent estimate of $\gamma^2_k(\bb{P},\bb{Q})$, one would expect similar behavior from $\gamma^2_k(\bb{P},\bb{Q})$. This means that, although the $B_1$-spline kernel is characteristic to $\Scr{P}$, in practice, it becomes harder to distinguish between $\bb{P}$ and $\bb{Q}$ with finite samples, when $\bb{P}$ is constructed as in (\ref{Eq:lim}) with $\nu=\frac{\omega_0}{\pi}$. In fact, one can observe from a straightforward spectral argument that the troughs in $\gamma^2_k(\bb{P},\bb{Q})$ can be made arbitrarily deep by widening $q$, when $q$ is Gaussian.
\end{example}
For characteristic kernels, although $\gamma_k(\bb{P},\bb{Q})>0$ when $\bb{P}\ne \bb{Q}$, Example~\ref{Exm:spline} demonstrates that one can construct distributions such that $\gamma^2_{k,u}(m,m)$ is indistinguishable from zero with high probability, for a given sample size $m$. 
Below, in Theorem~\ref{Thm:generic}, 
we explicitly construct $\bb{P}\ne \bb{Q}$ such that $|\bb{P}\varphi_l-\bb{Q}\varphi_l|$ is large for some large $l$, but $\gamma_k(\bb{P},\bb{Q})$ is arbitrarily small, making it hard to detect a non-zero value of $\gamma_k(\bb{P},\bb{Q})$ based on finite samples. Here, $\varphi_l\in L^2(M)$ represents the bounded orthonormal eigenfunctions of a positive definite integral operator 
associated with $k$. Based on this theorem, e.g., in Example~\ref{Exm:spline}, the decay mode of $\gamma_k$ for large $|\nu|$ can be investigated.

\par Consider the formulation of $\gamma_\eu{F}$ with $\eu{F}=\eu{F}_k$ in (\ref{Eq:MMD}). The construction of $\bb{P}$ for a given $\bb{Q}$ such that $\gamma_k(\bb{P},\bb{Q})$ is small, though not zero, can be intuitively understood by re-writing (\ref{Eq:MMD}) as 
\begin{equation}\label{Eq:MMD-rewrite}
\gamma_k(\bb{P},\bb{Q})=\sup_{f\in\eu{H}}\frac{|\bb{P}f-\bb{Q}f|}{\Vert f\Vert_{\eu{H}}}.
\end{equation}
When $\bb{P}\ne \bb{Q}$, $|\bb{P}f-\bb{Q}f|$ can be large for some $f\in\eu{H}$. However, $\gamma_k(\bb{P},\bb{Q})$ can be made small by selecting $\bb{P}$ such that the maximization of $\frac{|\bb{P}f-\bb{Q}f|}{\Vert f\Vert_{\eu{H}}}$ over $\eu{H}$ requires an $f$ with large $\Vert f\Vert_{\eu{H}}$. More specifically, higher order eigenfunctions of the kernel ($\varphi_l$ for large $l$) have large RKHS norms, so, if they are prominent in $\bb{P}$ and $\bb{Q}$ (i.e., highly non-smooth distributions), one can expect $\gamma_k(\bb{P},\bb{Q})$ to be small even when there exists an $l$ for which $|\bb{P}\varphi_l-\bb{Q}\varphi_l|$ is large. To this end, we need the following lemma, which we quote from \citet[Lemma 6]{Gretton-04}.
\begin{lemma}[\cite{Gretton-04}]\label{Lem:expansion}
Let $\eu{F}$ be the unit ball in an RKHS $(\eu{H},k)$ defined on a compact topological space, $M$, with $k$ being measurable. Let $\varphi_l\in L^2(M,\mu)$ be absolutely bounded orthonormal eigenfunctions and $\lambda_l$ be the corresponding eigenvalues (arranged in a decreasing order for increasing $l$) of a positive definite integral operator associated with $k$ and a $\sigma$-finite measure, $\mu$. Assume $\lambda^{-1}_l$ increases superlinearly with $l$. Then, for $f\in \eu{F}$ where $f(x)=\sum^{\infty}_{j=1}\widetilde{f}_j\varphi_j(x)$, $\widetilde{f}_j:=\langle f,\varphi_j\rangle_{L^2(M,\mu)}$, we have $\sum^\infty_{j=1}|\widetilde{f}_j|<\infty$ and for every $\varepsilon>0$, $\exists\,l_0\in\bb{N}$ such that $|\widetilde{f}_l|<\varepsilon$ if $l>l_0$.\vspace{-1.25mm}
\end{lemma}
\begin{theorem}[$\bb{P}\ne \bb{Q}$ can have arbitrarily small  $\gamma_k$]\label{Thm:generic}
Assume the conditions in Lemma~\ref{Lem:expansion} hold. Then, there exist probability measures $\bb{P}\ne \bb{Q}$ defined on $M$ such that
$\gamma_k(\bb{P},\bb{Q})<\varepsilon$ for any arbitrarily small $\varepsilon>0$.\vspace{-2mm}
\end{theorem}
\begin{proof}
Suppose $q$ be the Radon-Nikodym derivative associated with $\bb{Q}$ w.r.t. the $\sigma$-finite measure, $\mu$ (see Lemma~\ref{Lem:expansion}). Let us construct $p(x)=q(x)+\alpha_l e(x)+\tau\varphi_l(x)$ where $e(x)=\mathds{1}_{M}(x)$. For $\bb{P}$ to be a probability measure, the following conditions need to be satisfied:
\begin{eqnarray}
\int_{M}\left[\alpha_l e(x)+\tau\varphi_l(x)\right]\, d\mu(x)=0,\label{Eq:const1}\\
\min_{x\in M}\left[q(x)+\alpha_l e(x)+\tau\varphi_l(x)\right]\ge 0.\label{Eq:const2}
\end{eqnarray}
Expanding $e(x)$ and $f(x)$ in the orthonormal basis $\{\varphi_l\}^\infty_{l=1}$, we get $
e(x)=\sum^\infty_{l=1}\widetilde{e}_l\varphi_l(x)$ and $f(x)=\sum^\infty_{l=1}\widetilde{f}_l\varphi_l(x)$, where $\widetilde{e}_l:=\langle e,\varphi_l\rangle_{L^2(M,\mu)}$ and $\widetilde{f}_l:=\langle f,\varphi_l\rangle_{L^2(M,\mu)}$. Therefore,
\begin{eqnarray}
\bb{P}f-\bb{Q}f &=&\int_{M}f(x)\left[\alpha_l e(x)+\tau\varphi_l(x)\right]\,d\mu(x)\nonumber\\
&=&\int_{M} \left[\alpha_l\sum^\infty_{j=1}\widetilde{e}_j\varphi_j(x)+\tau\varphi_l(x)\right]
\left[\sum^\infty_{t=1}\widetilde{f}_t\varphi_t(x)\right]\,d\mu(x)\nonumber\\
&=&\alpha_l\sum^\infty_{j=1}\widetilde{e}_j\widetilde{f}_j+\tau\widetilde{f}_l,
\label{Eq:diff}\end{eqnarray}
where we used the fact that\footnote{Here, $\delta$ is used in the Kronecker sense.} $\langle\varphi_j,\varphi_t\rangle_{L^2(M,\mu)}=\delta_{jt}.$ Rewriting (\ref{Eq:const1}) and substituting for $e(x)$ gives
\begin{equation}
\int_{M}[\alpha_l e(x)+\tau\varphi_l(x)]\, d\mu(x)=\int_{M}e(x)[\alpha_l e(x)+\tau\varphi_l(x)]\, d\mu(x)=\alpha_l\sum^\infty_{j=1}\widetilde{e}^2_j+\tau\widetilde{e}_l=0,\nonumber
\end{equation}
which implies
\begin{equation}\label{Eq:alpha}
\alpha_l=-\frac{\tau\widetilde{e}_l}{\sum^\infty_{j=1}\widetilde{e}^2_j}.
\end{equation} 
Now, let us consider $\bb{P}\varphi_t-\bb{Q}\varphi_t=\alpha_l\widetilde{e}_t+\tau\delta_{tl}$.
Substituting for $\alpha_l$ gives 
\begin{equation}
\bb{P}\varphi_t-\bb{Q}\varphi_t=\tau\delta_{tl}-
\tau\frac{\widetilde{e}_t\widetilde{e}_l}{\sum^\infty_{j=1}\widetilde{e}^2_j}=\tau\delta_{tl}-\tau\rho_{tl},
\end{equation}
where $\rho_{tl}:=\frac{\widetilde{e}_t\widetilde{e}_l}{\sum^\infty_{j=1}\widetilde{e}^2_j}$. By Lemma~\ref{Lem:expansion}, $\sum^{\infty}_{l=1}|\widetilde{e}_l|<\infty\Rightarrow\sum^\infty_{j=1}\widetilde{e}^2_j<\infty$, and choosing large enough $l$ gives $|\rho_{tl}|<\eta,\,\forall\,t,$ for any arbitrary $\eta>0$. Therefore, $|\bb{P}\varphi_t-\bb{Q}\varphi_t|>\tau-\eta$ for $t=l$ and $|\bb{P}\varphi_t-\bb{Q}\varphi_t|<\eta$ for $t\ne l$, which means $\bb{P}\ne \bb{Q}$. In the following, we prove that $\gamma_k(\bb{P},\bb{Q})$ can be arbitrarily small, though non-zero.
\par Recall that $\gamma_k(\bb{P},\bb{Q})=\sup_{\Vert f\Vert_{\eu{H}}\le 1}|\bb{P}f-\bb{Q}f|$. Substituting for $\alpha_l$ in (\ref{Eq:diff}) and replacing $|\bb{P}f-\bb{Q}f|$ by (\ref{Eq:diff}) in $\gamma_k(\bb{P},\bb{Q})$, we have
\begin{equation}
\gamma_k(\bb{P},\bb{Q})=\sup_{\{\widetilde{f}_j\}^\infty_{j=1}}\left\{\tau\sum^\infty_{j=1}\nu_{jl}\widetilde{f}_j\,\,:\,\,\sum^\infty_{j=1}\frac{\widetilde{f}^2_j}{\lambda_j}\le 1\right\},\label{Eq:convex}
\end{equation}
where we used the definition of RKHS norm as $\Vert f\Vert_{\eu{H}}:=\sum^\infty_{j=1}\frac{\widetilde{f}^2_j}{\lambda_j}$ and $\nu_{jl}:=\delta_{jl}-\rho_{jl}$. (\ref{Eq:convex}) is a convex quadratically constrained quadratic program in $\{\widetilde{f}_j\}^\infty_{j=1}$. Solving the Lagrangian yields $\widetilde{f}_j=\frac{\nu_{jl}\lambda_j}{\sqrt{\sum^\infty_{j=1}\nu^2_{jl}\lambda_j}}$. Therefore, \begin{equation}
\gamma_k(\bb{P},\bb{Q})=\tau\sqrt{\sum^\infty_{j=1}\nu^2_{jl}\lambda_j}=\tau\sqrt
{\lambda_l-2\rho_{ll}\lambda_l+\sum^\infty_{j=1}\rho^2_{jl}\lambda_j}\stackrel{l\rightarrow\infty}{\longrightarrow} 0,
\end{equation}
 because \emph{(i)} by choosing sufficiently large $l$, $|\rho_{jl}|<\varepsilon,\,\forall\,j,$ for any arbitrary $\varepsilon>0$, and
\emph{(ii)} $\lambda_l\rightarrow 0$ as $l\rightarrow\infty$ \cite[Theorem 2.10]{Scholkopf-02}. Therefore, we have constructed $\bb{P}\ne\bb{Q}$ such that $\gamma_k(\bb{P},\bb{Q})<\varepsilon$ for any arbitrarily small $\varepsilon>0$.\vspace{-4mm}
\end{proof}

\section{Metrization of the Weak Topology}\label{Sec:weak}
So far, we have shown that a characteristic kernel, $k$ induces a metric, $\gamma_k$ on $\mathscr{P}$. As motivated in Section~\ref{subsubsec:contribution3}, an important question to consider that is useful both in theory and practice would be: ``How strong or weak is $\gamma_k$ related to other metrics on $\Scr{P}$?" This question is addressed in Theorem~\ref{thm:compare}, wherein we compared $\gamma_k$ to other metrics on $\Scr{P}$ like the Dudley metric ($\beta$), Wasserstein distance ($W$), total variation distance ($TV$) and showed that $\gamma_k$ is weaker than all these metrics (see footnote~\ref{fnote:strong-weak} for the definition of ``strong'' and ``weak'' metrics). Since $\gamma_k$ is weaker than the Dudley metric, which is well known to induce a topology on $\Scr{P}$ that coincides with the standard topology on $\Scr{P}$, called the weak-$^\ast$ (weak-star) topology (usually called the weak topology in probability theory), the next question we are interested in is to understand the topology that is being induced by $\gamma_k$. 
In particular, we are interested in determining the conditions on $k$ for which the topology induced by $\gamma_k$ coincides with the weak topology on $\Scr{P}$. This is answered in 
Theorems~\ref{thm:weak-1} and \ref{thm:weak-2}, wherein Theorem~\ref{thm:weak-1} deals with compact $M$ and Theorem~\ref{thm:weak-2} provides a sufficient condition on $k$ when $M=\bb{R}^d$. 
The proofs of all these results are provided in Section~\ref{subsec:proofs-weak}. Before we motivate the need for this study and its implications, we present some preliminaries. 
\par The \emph{weak topology} on $\Scr{P}$ is the weakest topology such that the map $\bb{P}\mapsto\int_M f\,d\bb{P}$ is continuous for all $f\in C_b(M)$. For a metric space, $(M,\rho)$, a sequence $\bb{P}_n$ of probability measures is said to \emph{converge weakly} to $\bb{P}$, written as $\bb{P}_n\stackrel{w}{\rightarrow}\bb{P}$, if and only if $\int_M f\,d\bb{P}_n\rightarrow\int_M f\,d\bb{P}$ for every $f\in C_b(M)$. A metric $\gamma$ on $\Scr{P}$ is said to \emph{metrize} the weak topology if the topology induced by $\gamma$ coincides with the weak topology, which is defined as follows: if, for $\bb{P},\bb{P}_1,\bb{P}_2,\ldots\in\Scr{P}$, $(\bb{P}_n\stackrel{w}{\rightarrow}\bb{P}\Leftrightarrow \gamma(\bb{P}_n,\bb{P})\stackrel{n\rightarrow\infty}{\longrightarrow} 0)$ holds, then the topology induced by $\gamma$ coincides with the weak topology.
\par In the following, we collect well-known results on the relation between various metrics on $\Scr{P}$, which will be helpful to understand the behavior of these metrics in relation to others. Let $(M,\rho)$ be a separable metric space. The \emph{Prohorov metric} on $(M,\rho)$, defined as
\begin{equation}\label{Eq:prohorov}
\varsigma(\bb{P},\bb{Q}):=\inf\{\epsilon>0:\bb{P}(A)\le \bb{Q}(A^\epsilon)+\epsilon,\,\forall\,\text{Borel sets}\,\,A\},
\end{equation}
metrizes the weak topology on $\mathscr{P}$ \citep[Theorem 11.3.3]{Dudley-02}, where $\bb{P},\bb{Q}\in\Scr{P}$ and $A^\epsilon:=\{y\in M:\rho(x,y)<\epsilon\,\,\text{for some}\,\,x\in A\}$. Since the Dudley metric is related to 
the Prohorov metric as 
\begin{equation}
\frac{1}{2}\beta(\bb{P},\bb{Q})\le \varsigma(\bb{P},\bb{Q})\le 2\sqrt{\beta(\bb{P},\bb{Q})},\label{Eq:Prohorov-Dudley}
\end{equation}
it also metrizes the weak topology on $\Scr{P}$ \citep[Theorem 11.3.3]{Dudley-02}. The Wasserstein distance and total variation distance are related to the Prohorov metric as 
\begin{equation}
\varsigma^2(\bb{P},\bb{Q}) \le W(\bb{P},\bb{Q})\le (\text{diam}(M)+1)\varsigma(\bb{P},\bb{Q}),\label{Eq:Prohorov-Wasserstein}
\end{equation} and 
\begin{equation}
\varsigma(\bb{P},\bb{Q})\le TV(\bb{P},\bb{Q}),\label{Eq:Wasserstein-TV}
\end{equation} 
where $\text{diam}(M):=\sup\{\rho(x,y)\,:\,x,y\in M\}$ \citep[Theorem 2]{Gibbs-02}. This means $W$ and $TV$ are stronger 
than $\varsigma$, while $W$ and $\varsigma$ are equivalent (i.e., induce the same topology) when $M$ is bounded. By Theorem 4 in \citet{Gibbs-02}, $TV$ and $W$ are related as 
\begin{equation}
W(\bb{P},\bb{Q})\le \text{diam}(M)TV(\bb{P},\bb{Q}),
\end{equation}
which means $W$ and $TV$ are comparable if $M$ is bounded. See \citet[Chapter 19, Theorem 2.4]{Shorack-00} and \citet{Gibbs-02} for the relationship between various metrics on $\Scr{P}$.
\par Now, let us consider a sequence of 
of probability measures on $\bb{R}$, $\bb{P}_n:=\left(1-\frac{1}{n}\right)\delta_0+\frac{1}{n}\delta_n$ and let $\bb{P}:=\delta_0$. It can be shown that $\beta(\bb{P}_n,\bb{P})\rightarrow 0$ as $n\rightarrow \infty$ which means $\bb{P}_n\stackrel{w}{\rightarrow}\bb{P}$, while $W(\bb{P}_n,\bb{P})=1$ and $TV(\bb{P}_n,\bb{P})=1$ for all $n$. $\gamma_k(\bb{P}_n,\bb{P})$ can be computed as
\begin{equation}\label{Eq:weak-kernel}
\gamma^2_k(\bb{P}_n,\bb{P})=\frac{1}{n^2}\int\!\!\!\int_{\bb{R}}k(x,y)\,d(\delta_0-\delta_n)(x)\,d(\delta_0-\delta_n)(y)=\frac{k(0,0)+k(n,n)-2k(0,n)}{n^2}.
\end{equation}
If $k$ is, e.g., a Gaussian, Laplacian or inverse multiquadratic kernel, then $\gamma_k(\bb{P}_n,\bb{P})\rightarrow 0$ as $n\rightarrow\infty$. This example shows that $\gamma_k$ is weaker than $W$ and $TV$. It also shows that $\gamma_k$ behaves similar to $\beta$ and leads to several questions we want to answer: Does $\gamma_k$ metrize the weak topology on $\Scr{P}$? What is the general behavior of $\gamma_k$ compared to 
other metrics? In other words, depending on $k$, how weak or strong is $\gamma_k$ compared to 
other metrics on $\Scr{P}$? Understanding the answer to these questions is important both in theory and practice. If $k$ is characterized such that $\gamma_k$ metrizes the weak topology on $\Scr{P}$, then it can be used as a theoretical tool in probability theory, similar to the Prohorov and Dudley metrics. On the other hand, the answer to these questions is critical in applications as it will have a bearing on the choice of kernels to be used. In applications like density estimation, one would need a strong metric to ascertain that the density estimate is a good representation of the true underlying density. For this reason, usually, the total variation distance, Hellinger distance or Kullback-Leibler distance are used. Studying the relation of $\gamma_k$ to these metrics will provide an understanding about the choice of kernels to be used, depending on the application. 
\par With the above motivation, in the following, we first compare $\gamma_k$ to $\beta$, $W$ and $TV$. Since $\beta$ is equivalent to $\varsigma$, we do not compare $\gamma_k$ to $\varsigma$. Before we provide the main result in Theorem~\ref{thm:compare} that compares $\gamma_k$ to other metrics, we present an upper bound on $\gamma_k$ in terms of the coupling formulation \citep[Section 11.8]{Dudley-02}, which is not only useful in deriving the main result but also interesting in its own right.
\begin{proposition}[Coupling bound]\label{pro:couplingrkhs}
Let $k$ be measurable and bounded on $M$. Then, for any $\bb{P},\bb{Q}\in\Scr{P}$,
\begin{equation}\label{Eq:couplingrkhs}
\gamma_k(\bb{P},\bb{Q})\le\inf_{\mu\in\mathcal{L}(\bb{P},\bb{Q})}\int\!\!\!\int_M\Vert k(\cdot,x)-k(\cdot,y)\Vert_\eu{H}\,d\mu(x,y),
\end{equation}
where $\mathcal{L}(\bb{P},\bb{Q})$ represents the set of all laws on $M\times M$ with marginals $\bb{P}$ and $\bb{Q}$.
\end{proposition}
\begin{proof}
For any $\mu\in\mathcal{L}(\bb{P},\bb{Q})$, we have
\begin{eqnarray}
\left|\int_M f\,d(\bb{P}-\bb{Q})\right|&\!\!\!=\!\!\!&\left|\int\!\!\!\int_M (f(x)-f(y))\,d\mu(x,y)\right|\le \int\!\!\!\int_M |f(x)-f(y)|\,d\mu(x,y)\nonumber\\
&\!\!\!=\!\!\!&\int\!\!\!\int_M|\langle f,k(\cdot,x)-k(\cdot,y)\rangle_\eu{H}|\,d\mu(x,y)\nonumber\\
&\!\!\!\le\!\!\!& \Vert f\Vert_{\eu{H}}\int\!\!\!\int_M\Vert k(\cdot,x)-k(\cdot,y)\Vert_{\eu{H}}\,d\mu(x,y).\label{Eq:coup}
\end{eqnarray}
Taking the supremum over $f\in\eu{F}_k$ and the infimum over $\mu\in\mathcal{L}(\bb{P},\bb{Q})$ in (\ref{Eq:coup}), where $\bb{P},\bb{Q}\in\Scr{P}$, gives the result in (\ref{Eq:couplingrkhs}).
\end{proof}
We now present the main result that compares $\gamma_k$ to $\beta$, $W$ and $TV$.
\begin{theorem}[Comparison of $\gamma_k$ to $\beta$, $W$ and $TV$]\label{thm:compare}
Assume $\sup_{x\in M}k(x,x)\le C<\infty$, where $k$ is measurable on $M$. Let \begin{equation}\label{Eq:Hilb}
\widetilde{\rho}(x,y)=\Vert k(\cdot,x)-k(\cdot,y)\Vert_\eu{H}.
\end{equation}
Then, for any $\bb{P},\bb{Q}\in\Scr{P}$, 
\begin{itemize}
\item[(i)] 
$\gamma_k(\bb{P},\bb{Q})\le W(\bb{P},\bb{Q})\le\sqrt{\gamma^2_k(\bb{P},\bb{Q})+4C}$ if $(M,\widetilde{\rho})$ is separable.
\item[(ii)]
 $\frac{\gamma_k(\bb{P},\bb{Q})}{(1+\sqrt{C})} \le \beta(\bb{P},\bb{Q})\le 2(\gamma^2_k(\bb{P},\bb{Q})+4C)^{\frac{1}{3}}$ if $(M,\widetilde{\rho})$ is separable.
\item[(iii)]
 $\gamma_k(\bb{P},\bb{Q})\le \sqrt{C}\,TV(\bb{P},\bb{Q})$.
\end{itemize}
\end{theorem}
The proof is provided in Section~\ref{subsec:proofs-weak}. Below are some remarks on Theorem~\ref{thm:compare}.
\begin{remark}\label{rem:weak}
(a) First, note that, since $k$ is bounded, $(M,\widetilde{\rho})$ is a bounded metric space. In addition, the metric, $\widetilde{\rho}$, which depends on the kernel as in (\ref{Eq:Hilb}), is a Hilbertian metric\footnote{A metric $\rho$ on $M$ is said to be \emph{Hilbertian} if there exists a Hilbert space, $H$ and a mapping $\Phi$ such that $\rho(x,y)=\Vert \Phi(x)-\Phi(y)\Vert_H,\,\forall\,x,y\in M$. In our case, $H=\eu{H}$ and $\Phi:M\rightarrow\eu{H}$, $x\mapsto k(\cdot,x)$.} \citep[Chapter 3, Section 3]{Berg-84} on $M$. A popular example of such a metric is $\widetilde{\rho}(x,y)=\Vert x-y\Vert_2$, which can be obtained by choosing $M$ to be a compact subset of $\bb{R}^d$ and $k(x,y)=x^Ty$.\vspace{2mm}\\
(b) Theorem~\ref{thm:compare} shows that $\gamma_k$ is weaker than $\beta$, $W$ and $TV$ for the assumptions being made on $k$ and $\widetilde{\rho}$. Note that the result holds irrespective of whether the kernel is characteristic or not, as we have not assumed anything about the kernel except it being measurable and bounded. Also, it is important to remember that the result holds when $\widetilde{\rho}$ is Hilbertian, as mentioned in (\ref{Eq:Hilb}) (see Remark~\ref{rem:weak}(d)).\vspace{2mm}\\
(c) Apart from showing that $\gamma_k$ is weaker than $\beta$, $W$ and $TV$, the result in Theorem~\ref{thm:compare} can be used to bound these metrics in terms of $\gamma_k$. For $\beta$, which is primarily of theoretical interest, we do not know a closed form expression, whereas a closed form expression to compute $W$ is known only for $\bb{R}$ \citep{Vallander-73}.\footnote{The explicit form for the Wasserstein distance is known for $(M,\rho(x,y))=(\bb{R},|x-y|)$, which is given as $W(\bb{P},\bb{Q})=\int_{\bb{R}}|F_\bb{P}(x)-F_\bb{Q}(x)|\,dx$, where $F_\bb{P}(x)=\bb{P}((-\infty,x])$. It is easy to show that this explicit form can be extended to $(\bb{R}^d,\Vert\cdot\Vert_1)$.} Since $\gamma_k$ is easy to compute (see (\ref{Eq:computeMMD}) and (\ref{Eq:computeMMD-1})), bounds on $W$ can be obtained from Theorem~\ref{thm:compare} in terms of $\gamma_k$. A closed form expression for $TV$ is available if $\bb{P}$ and $\bb{Q}$ have Radon-Nikodym derivatives w.r.t. a $\sigma$-finite measure. However, from Theorem~\ref{thm:compare}, a simple lower bound can be obtained on $TV$ in terms of $\gamma_k$ for any $\bb{P},\bb{Q}\in\Scr{P}$.\vspace{2mm}\\
(d) In Theorem~\ref{thm:compare}, the kernel is fixed and $\widetilde{\rho}$ is defined as in (\ref{Eq:Hilb}), which is a Hilbertian metric. On the other hand, suppose a Hilbertian metric, $\widetilde{\rho}$ is given. Then, the associated kernel, $k$ can be obtained from $\widetilde{\rho}$ \citep[Chapter 3, Lemma 2.1]{Berg-84} as 
\begin{equation}\label{Eq:hilb-kernel}
k(x,y)=\frac{1}{2}[\widetilde{\rho}^2(x,x_0)+\widetilde{\rho}^2(y,x_0)-\widetilde{\rho}^2(x,y)],\,\,x,y,x_0\in M,
\end{equation}
which can then be used to compute $\gamma_k$. 
\end{remark}
\par The discussion so far has been devoted to relating $\gamma_k$ to $\beta$, $W$ and $TV$ to understand the strength or weakness of $\gamma_k$ w.r.t. these metrics. In a next step, we address the other question of when $\gamma_k$ metrizes the weak topology on $\Scr{P}$. This question would have been answered had the result in Theorem~\ref{thm:compare} shown that under some conditions on $k$, $\gamma_k$ is equivalent to $\beta$. Since Theorem~\ref{thm:compare} does not throw light on the question we are interested in, we approach the problem differently. In the following, we provide two results related to this question. The first result states that when $(M,\rho)$ is compact, $\gamma_k$ induced by universal kernels metrizes the weak topology. In the second result, we relax the assumption of compactness but restrict ourselves to $M=\bb{R}^d$ and provide a sufficient condition on $k$ such that $\gamma_k$ metrizes the weak topology on $\Scr{P}$. The proofs of both theorems are provided in Section \ref{subsec:proofs-weak}.
\begin{theorem}[Weak convergence-I]\label{thm:weak-1}
Let $(M,\rho)$ be a compact metric space. If $k$ is universal, then $\gamma_k$ metrizes the weak topology on $\mathscr{P}$.
\end{theorem}
From Theorem~\ref{thm:weak-1}, it is clear that $\gamma_k$ is equivalent to $\varsigma$, $\beta$ and $W$ (see (\ref{Eq:Prohorov-Dudley}) and (\ref{Eq:Prohorov-Wasserstein})) when $M$ is compact and $k$ is universal. 
\begin{theorem}[Weak convergence-II]\label{thm:weak-2}
Let $M=\bb{R}^d$ and $k(x,y)=\psi(x-y)$, where $\psi\in C_0(\bb{R}^d)\cap L^1(\bb{R}^d)$ is a real-valued bounded strictly positive definite function. If there exists an $l\in\bb{N}$ such that 
\begin{equation}\label{Eq:weak-condition}
\int_{\bb{R}^d}\frac{1}{\widehat{\psi}(\omega)(1+\Vert\omega\Vert_2)^l}\,d\omega<\infty, 
\end{equation}
then $\gamma_k$ metrizes the weak topology on $\Scr{P}$.
\end{theorem}
The entire Mat\'{e}rn class of kernels in (\ref{Eq:matern}) satisfies the conditions of Theorem~\ref{thm:weak-2} and, therefore, the corresponding $\gamma_k$ metrizes the weak topology on $\mathscr{P}$. Note that Gaussian kernels on $\bb{R}^d$ do not satisfy the condition in Theorem~\ref{thm:weak-2}. 
The characterization of $k$ for general non-compact domains $M$ (not necessarily $\bb{R}^d$), such that $\gamma_k$ metrizes the weak topology on $\mathscr{P}$, still remains an open problem.

\subsection{Proofs}\label{subsec:proofs-weak}
We now present the proofs of Theorems~\ref{thm:compare}, \ref{thm:weak-1} and \ref{thm:weak-2}.\vspace{2mm}
\begin{proof}
\hspace{-.05in}\textbf{(Theorem~\ref{thm:compare})} \emph{(i)} When $(M,\rho)$ is separable, $W(\bb{P},\bb{Q})$ has a coupling formulation \cite[p. 420]{Dudley-02}, given as
\begin{equation}\label{Eq:wasserstein-coupling}
W(\bb{P},\bb{Q})=\inf_{\mu\in\mathcal{L}(\bb{P},\bb{Q})}\int\!\!\!\int_M\rho(x,y)\,d\mu(x,y),
\end{equation}
where $\bb{P},\bb{Q}\in\{\bb{P}\in\Scr{P}:\int_M\rho(x,y)\,d\bb{P}(y)<\infty,\,\forall\,x\in M\}$. In our case $\rho(x,y)=\Vert k(\cdot,x)-k(\cdot,y)\Vert_\eu{H}$. In addition, $(M,\rho)$ is bounded, which means (\ref{Eq:wasserstein-coupling}) holds for all $\bb{P},\bb{Q}\in\Scr{P}$. The lower bound therefore follows from (\ref{Eq:couplingrkhs}). The upper bound can be obtained as follows. Consider $W(\bb{P},\bb{Q})=\inf_{\mu\in\mathcal{L}(\bb{P},\bb{Q})}\int\!\!\!\int_M\Vert k(\cdot,x)-k(\cdot,y)\Vert_\eu{H}\,d\mu(x,y)$, which can be bounded as
\begin{eqnarray}\label{Eq:proof-1}
W(\bb{P},\bb{Q})&\!\!\!\le\!\!\!&\int\!\!\!\int_M\Vert k(\cdot,x)-k(\cdot,y)\Vert_\eu{H}\,d\bb{P}(x)\,d\bb{Q}(y)\nonumber\\
&\!\!\!\stackrel{(a)}{\le}\!\!\!& \left[\int\!\!\!\int_M\Vert k(\cdot,x)-k(\cdot,y)\Vert^2_\eu{H}\,d\bb{P}(x)\,d\bb{Q}(y)\right]^{\frac{1}{2}}\nonumber\\
&\!\!\!\le\!\!\!& \left[\int_M k(x,x)\,d(\bb{P}+\bb{Q})(x)-2\int\!\!\!\int_Mk(x,y)\,d\bb{P}(x)\,d\bb{Q}(y)\right]^{\frac{1}{2}}\nonumber\\
&\!\!\!\le\!\!\!& \left[\gamma^2_k(\bb{P},\bb{Q})+\int\!\!\!\int_M (k(x,x)-k(x,y))\,d(\bb{P}\otimes\bb{P}+\bb{Q}\otimes\bb{Q})(x,y)\right]^{\frac{1}{2}}\nonumber\\
&\!\!\!\le\!\!\!& \sqrt{\gamma^2_k(\bb{P},\bb{Q})+4C},
\end{eqnarray}
where we have used Jensen's inequality \citep[p. 109]{Folland-99} in $(a)$.\vspace{2mm}\\
\emph{(ii)} Let $\eu{F}:=\{f:\Vert f\Vert_{\eu{H}}<\infty\}$ and $\eu{G}:=\{f:\Vert f\Vert_{BL}<\infty\}$. For $f\in\eu{F}$, we have
\begin{eqnarray}\label{Eq:proof-3}
\Vert f\Vert_{BL}&\!\!\!=\!\!\!&\sup_{x\ne y}\frac{|f(x)-f(y)|}{\rho(x,y)}+\sup_{x\in M}|f(x)|=\sup_{x\ne y}\frac{|\langle f,k(\cdot,x)-k(\cdot,y)\rangle_\eu{H}|}{\Vert k(\cdot,x)-k(\cdot,y)\Vert_\eu{H}}+\sup_{x\in M}|\langle f,k(\cdot,x)\rangle_\eu{H}|\nonumber\\
&\!\!\!\le\!\!\!& (1+\sqrt{C})\Vert f\Vert_{\eu{H}}<\infty,
\end{eqnarray}
which implies $f\in\eu{G}$ and, therefore, $\eu{F}\subset\eu{G}$. For any $\bb{P},\bb{Q}\in\Scr{P}$, 
\begin{eqnarray}
\gamma_k(\bb{P},\bb{Q})&\!\!\!=\!\!\!&\sup\{|\bb{P}f-\bb{Q}f|:f\in\eu{F}_k\}\nonumber\\
&\!\!\!\le\!\!\!&\sup\{|\bb{P}f-\bb{Q}f|:\Vert f\Vert_{BL}\le (1+\sqrt{C}),\,f\in\eu{F}\}\nonumber\\
&\!\!\!\le\!\!\!& \sup\{|\bb{P}f-\bb{Q}f|:\Vert f\Vert_{BL}\le (1+\sqrt{C}),\,f\in\eu{G}\}\nonumber\\
&\!\!\!=\!\!\!&(1+\sqrt{C})\beta(\bb{P},\bb{Q}).\nonumber
\end{eqnarray}
The upper bound is obtained as follows. For any $\bb{P},\bb{Q}\in\Scr{P}$, by Markov's inequality \citep[Theorem 6.17]{Folland-99}, 
for all $\epsilon>0$, we have
\begin{equation}
\epsilon^2\mu(\Vert k(\cdot,X)-k(\cdot,Y)\Vert_\eu{H}>\epsilon)\le
\int\!\!\!\int_M\Vert k(\cdot,x)-k(\cdot,y)\Vert^2_\eu{H}\,d\mu(x,y),\nonumber
\end{equation}
where $X$ and $Y$ are distributed as $\bb{P}$ and $\bb{Q}$ respectively. Choose $\epsilon$ such that $\epsilon^3=\int\!\!\!\int_M\Vert k(\cdot,x)-k(\cdot,y)\Vert^2_\eu{H}\,d\mu(x,y)$, such that $\mu(\Vert k(\cdot,X)-k(\cdot,Y)\Vert_\eu{H}>\epsilon)\le\epsilon$. From the proof of Theorem 11.3.5 of \citet{Dudley-02}, when $(M,\rho)$ is separable, we have
\begin{equation}
\mu(\rho(X,Y)\ge\epsilon)<\epsilon\,\Rightarrow\,\varsigma(\bb{P},\bb{Q})\le\epsilon,\nonumber
\end{equation}
which implies that 
\begin{eqnarray}\label{Eq:proof-2}
\varsigma(\bb{P},\bb{Q})&\!\!\!\le\!\!\!& \left(\inf_{\mu\in\mathcal{L}(\bb{P},\bb{Q})}\int\!\!\!\int_M\Vert k(\cdot,x)-k(\cdot,y)\Vert^2_\eu{H}\,d\mu(x,y)\right)^{\frac{1}{3}}\nonumber\\
&\!\!\!\le\!\!\!& \left(\int\!\!\!\int_M\Vert k(\cdot,x)-k(\cdot,y)\Vert^2_\eu{H}\,d\bb{P}(x)\,d\bb{Q}(y)\right)^{\frac{1}{3}}\nonumber\\
&\!\!\!\stackrel{(b)}{\le}\!\!\!&\left(\gamma^2_k(\bb{P},\bb{Q})+4C\right)^{\frac{1}{3}},
\end{eqnarray}
where $(b)$ follows from (\ref{Eq:proof-1}). The result follows from (\ref{Eq:Prohorov-Dudley}).\vspace{2mm}\\
\emph{(iii)} The proof of this result was presented in \citet{Sriperumbudur-09} and is provided here for completeness. To prove the result, we use (\ref{Eq:couplingrkhs}) and the coupling formulation for $TV$ \citep[p. 19]{Lindvall-92}, given as
\begin{equation}
\frac{1}{2}TV(\bb{P},\bb{Q})=\inf_{\mu\in\mathcal{L}(\bb{P},\bb{Q})}\mu(X\ne Y),
\end{equation}
where $\mathcal{L}(\bb{P},\bb{Q})$ is the set of all measures on $M\times M$ with marginals $\bb{P}$ and $\bb{Q}$. Here, $X$ and $Y$ are distributed as $\bb{P}$ and $\bb{Q}$ respectively. Consider 
\begin{equation}\label{Eq:proof-4}
\Vert k(\cdot,x)-k(\cdot,y)\Vert_\eu{H}\le \mathds{1}_{\{x\ne y\}}\Vert k(\cdot,x)-k(\cdot,y)\Vert_\eu{H}\le 2\sqrt{C}\mathds{1}_{\{x\ne y\}}.
\end{equation} 
Taking expectations w.r.t. $\mu$ and the infimum over $\mu\in\mathcal{L}(\bb{P},\bb{Q})$ on both sides of (\ref{Eq:proof-4}) gives the desired result, which follows from (\ref{Eq:couplingrkhs}).\vspace{-4mm}
\end{proof}
\begin{proof}
\hspace{-.02in}\textbf{(Theorem~\ref{thm:weak-1})} We need to show that for measures $\bb{P},\bb{P}_1,\bb{P}_2,\ldots\in\Scr{P}$, $\bb{P}_n\stackrel{w}{\rightarrow}\bb{P}$ if and only if $\gamma_k(\bb{P}_n,\bb{P})\rightarrow 0$ as $n\rightarrow \infty$. One direction is trivial as $\bb{P}_n\stackrel{w}{\rightarrow} \bb{P}$ implies $\gamma_k(\bb{P}_n,\bb{P})\rightarrow 0$ as $n\rightarrow\infty$. We prove the other direction as follows. Since $k$ is universal, $\eu{H}$ is dense in $C_b(M)$, the space of bounded continuous functions, w.r.t. the uniform norm, i.e., for any $f\in C_b(M)$ and every $\epsilon>0$, there exists a $g\in \eu{H}$ such that $\Vert f-g\Vert_\infty \le \epsilon$. Therefore, \begin{eqnarray}
|\bb{P}_nf-\bb{P}f|&\!\!\!=\!\!\!&|\bb{P}_n(f-g)+\bb{P}(g-f)+(\bb{P}_ng-\bb{P}g)|\nonumber\\
&\!\!\!\le\!\!\!&  \bb{P}_n|f-g|+\bb{P}|f-g|+|\bb{P}_ng-\bb{P}g|\nonumber\\
&\!\!\!\le\!\!\!& 2\epsilon+|\bb{P}_ng-\bb{P}g|\le 2\epsilon+\Vert g\Vert_\eu{H}\gamma_k(\bb{P}_n,\bb{P}).
\end{eqnarray}
Since $\gamma_k(\bb{P}_n,\bb{P})\rightarrow 0$ as $n\rightarrow\infty$ and $\epsilon$ is arbitrary, $|\bb{P}_nf-\bb{P}f|\rightarrow 0$ for any $f\in C_b(M)$.\vspace{-4mm}
\end{proof}
\begin{proof}
\hspace{-.05in}\textbf{(Theorem~\ref{thm:weak-2})} As mentioned in the proof of Theorem~\ref{thm:weak-1}, one direction of the proof is straightforward: $\bb{P}_n\stackrel{w}{\rightarrow}\bb{P}\, \Rightarrow\, \gamma_k(\bb{P}_n,\bb{P})\rightarrow 0$ as $n\rightarrow\infty$. Let us consider the other direction. Since $\psi\in C_0(\bb{R}^d)\cap L^1(\bb{R}^d)$ is a strictly positive definite function, any $f\in\eu{H}$ satisfies \citep[Theorem 10.12]{Wendland-05}
\begin{equation}
\int_{\bb{R}^d}\frac{|\widehat{f}(\omega)|^2}{\widehat{\psi}(\omega)}\,d\omega<\infty.
\end{equation}
Assume that
\begin{equation}
\sup_{\omega\in\bb{R}^d}(1+\Vert\omega\Vert_2)^l|\widehat{f}(\omega)|^2<\infty,
\end{equation}
for any $l\in\bb{N}$, which means $f\in\Scr{S}_d$. Let (\ref{Eq:weak-condition}) be satisfied for some $l=l_0$. Then, \begin{eqnarray}\label{Eq:weak-2}
\int_{\bb{R}^d}\frac{|\widehat{f}(\omega)|^2}{\widehat{\psi}(\omega)}\,d\omega&\!\!\!=\!\!\!& \int_{\bb{R}^d}\frac{|\widehat{f}(\omega)|^2(1+\Vert\omega\Vert_2)^{l_0}}{\widehat{\psi}(\omega)(1+\Vert\omega\Vert_2)^{l_0}}\,d\omega\nonumber\\
&\!\!\!\le\!\!\!& \sup_{\omega\in\bb{R}^d}(1+\Vert\omega\Vert_2)^{l_0}|\widehat{f}(\omega)|^2\int_{\bb{R}^d}\frac{1}{\widehat{\psi}(\omega)(1+\Vert\omega\Vert_2)^{l_0}}\,d\omega<\infty,\nonumber
\end{eqnarray}
which means $f\in\eu{H}$, i.e., if $f\in\Scr{S}_d$, then $f\in\eu{H}$, which implies $\Scr{S}_d\subset\eu{H}$. Note that $\mathscr{S}(\bb{R}^d)$ is dense in $C_0(\bb{R}^d)$. Since $\psi\in C_0(\bb{R}^d)$, we have $\eu{H}\subset C_0(\bb{R}^d)$ and, therefore, $\eu{H}$ is dense in $C_0(\bb{R}^d)$ w.r.t. the uniform norm. Suppose $\bb{P},\bb{P}_1,\bb{P}_2,\ldots\in\mathscr{P}$. Using a similar analysis as in the proof of Theorem~\ref{thm:weak-1}, it can be shown that for any $f\in C_0(\bb{R}^d)$ and every $\epsilon>0$, there exists a $g\in\eu{H}$ such that $|\bb{P}_nf-\bb{P}f|\le 2\epsilon+|\bb{P}_ng-\bb{P}g|$. Since $\epsilon$ is arbitrary and $\gamma_k(\bb{P}_n,\bb{P})\rightarrow 0$ as $n\rightarrow\infty$, the result follows.\vspace{-4mm}
\end{proof}

\section{Conclusion and Discussion}\label{Sec:Conclusion}
In this paper, we have studied various properties associated with a pseudometric, $\gamma_k$ on $\Scr{P}$, which is based on the Hilbert space embedding of probability measures. First, we studied the conditions on the kernel (called the characteristic kernel) under which $\gamma_k$ is a metric and showed that, apart from universal kernels, a large family of bounded continuous kernels induce a metric on $\Scr{P}$: (a) integrally strictly pd kernels and (b) translation invariant kernels on $\bb{R}^d$ and $\bb{T}^d$ that have the support of their Fourier transform to be $\bb{R}^d$ and $\bb{Z}^d$ respectively. Next, we showed that there exist distinct distributions  
which will be considered close according to $\gamma_k$ (whether or not the  kernel is characteristic), and thus may be hard to distinguish based on finite samples. 
Finally, we compared $\gamma_k$ to other metrics on $\Scr{P}$ and explicitly presented the conditions under which it induces a weak topology on $\Scr{P}$. 
These results together provide a strong theoretical foundation for using the $\gamma_k$ metric in both statistics and machine learning applications.
\par Now, we discuss two topics related to $\gamma_k$, one about the choice of kernel parameter and the other about kernels defined on $\Scr{P}$. 
\par An important question that we did not discuss in this paper is how to choose a characteristic kernel. Let us consider the following setting: $M=\bb{R}^d$ and $k_\sigma(x,y)=\exp(-\sigma\Vert x-y\Vert^2_2),\,\sigma\in\bb{R}_+$, a Gaussian kernel with $\sigma$ as the bandwidth parameter. $\{k_\sigma:\sigma\in\bb{R}_+\}$ is the family of Gaussian kernels and $\{\gamma_{k_\sigma}:\sigma\in\bb{R}_+\}$ is the associated family of distance measures indexed by the kernel parameter, $\sigma$. Note that $k_\sigma$ is characteristic for any $\sigma\in\bb{R}_{++}$ and, therefore, $\gamma_{k_\sigma}$ is a metric on $\Scr{P}$ for any $\sigma\in\bb{R}_{++}$. In practice, one would prefer a single number that defines the distance between $\bb{P}$ and $\bb{Q}$. The question therefore to be addressed is how to choose an appropriate $\sigma$. Note that as $\sigma\rightarrow 0$, $k_\sigma\rightarrow 1$ and as $\sigma\rightarrow\infty$, $k_\sigma\rightarrow 0$ a.e., which means $\gamma_{k_\sigma}(\bb{P},\bb{Q})\rightarrow 0$ as $\sigma\rightarrow 0$ or $\sigma\rightarrow\infty$ for all $\bb{P},\bb{Q}\in\mathscr{P}$. This behavior is also exhibited by $k_\sigma(x,y)=\exp(-\sigma\Vert x-y\Vert_1),\,\sigma>0$ and $k_\sigma(x,y)=\sigma^2/(\sigma^2+\Vert x-y\Vert^2_2),\,\sigma>0$, which are also characteristic. This means choosing \emph{sufficiently small} or \emph{sufficiently large} $\sigma$ (depending on $\bb{P}$ and $\bb{Q}$) makes $\gamma_{k_\sigma}(\bb{P},\bb{Q})$ arbitrarily small. 
Therefore, $\sigma$ must be chosen appropriately in applications to effectively distinguish between $\bb{P}$ and $\bb{Q}$. 
\par To this end, one can consider the following modification to $\gamma_k$, which yields a pseudometric on $\Scr{P}$,
\begin{equation}\label{Eq:uniform}
\gamma(\bb{P},\bb{Q})=\sup\{\gamma_{k}(\bb{P},\bb{Q}):k\in\eu{K}\}=\sup\{\Vert\bb{P}k-\bb{Q}k\Vert_\eu{H}:k\in\eu{K}\}.
\end{equation}
Note that $\gamma$ is the maximal RKHS distance between $\bb{P}$ and $\bb{Q}$ over a family, $\eu{K}$ of measurable and bounded positive definite kernels. It is easy to check that, if any $k\in\eu{K}$ is characteristic, then $\gamma$ is a metric on $\mathscr{P}$. Examples for $\eu{K}$ include: 
\begin{enumerate}
\item $\eu{K}_g:=\{e^{-\sigma\Vert x-y\Vert^2_2},\,x,y\in\bb{R}^d\,:\,\sigma\in\bb{R}_+\}$.
\item $\eu{K}_l:=\{e^{-\sigma\Vert x-y\Vert_1},\,x,y\in\bb{R}^d\,:\,\sigma\in\bb{R}_+\}$.
\item $\eu{K}_\psi:=\{e^{-\sigma\psi(x,y)},\,x,y\in M\,:\sigma\in\bb{R}_+\}$, where $\psi:M\times M \rightarrow\bb{R}$ is a negative definite kernel \citep[Chapter 3]{Berg-84}.
\item $\eu{K}_{rbf}:=\{\int^\infty_0e^{-\lambda\Vert x-y\Vert^2_2}\,d\mu_\sigma(\lambda),x,y\in\bb{R}^d,\,\mu_\sigma\in\mathscr{M}^+\,:\,\sigma\in\Sigma\subset\bb{R}^d\}$, where $\mathscr{M}^+$ is the set of all finite nonnegative Borel measures, $\mu_\sigma$ on $\bb{R}_+$ that are not concentrated at zero, etc.
\item $\eu{K}_{lin}:=\{k_\lambda=\sum^l_{j=1}\lambda_jk_j\,|\, k_\lambda\,\,\text{is pd},\,\sum^l_{j=1}\lambda_j=1\}$, which is the linear combination of pd kernels $\{k_j\}_{j=1}^l$.
\item $\eu{K}_{con}:=\{k_\lambda=\sum^l_{j=1}\lambda_jk_j\,|\, \lambda_j\ge 0,\,\sum^l_{j=1}\lambda_j=1\}$, which is the convex combination of pd kernels $\{k_j\}_{j=1}^l$. 
\end{enumerate}
\par The idea and validity behind the proposal of $\gamma$ in (\ref{Eq:uniform}) can be understood from a Bayesian perspective, where we define a non-negative finite measure $\lambda$ over $\eu{K}$, and average $\gamma_k$ over that measure, i.e., $\alpha(\bb{P},\bb{Q}):=\int_{\eu{K}}\gamma_k(\bb{P},\bb{Q})\,d\lambda(k)$. This  also yields a pseudometric on $\Scr{P}$. That said, $\alpha(\bb{P},\bb{Q})\le\lambda(\eu{K})\gamma(\bb{P},\bb{Q}),\,\forall\,\bb{P},\bb{Q}$, which means that, if $\bb{P}$ and $\bb{Q}$ can be distinguished by $\alpha$, then they can be distinguished by $\gamma$, but not vice-versa. In this sense, $\gamma$ is stronger than $\alpha$ and therefore studying $\gamma$ makes sense. One further complication with the Bayesian approach is in defining a sensible $\lambda$ over $\eu{K}$. Note that $\gamma_{k_0}$ can be obtained by defining $\lambda(k)=\delta(k-k_0)$ in $\alpha(\bb{P},\bb{Q})$. Future work will include analyzing $\gamma$ and investigating its utility in applications compared to that of $\gamma_k$ (with a fixed kernel, $k$). Refer to \citet{Sriperumbudur-09c} for some preliminary work, wherein we showed that $\gamma(\bb{P}_m,\bb{Q}_n)$ is a $\sqrt{mn/(m+n)}$-consistent estimator of $\gamma(\bb{P},\bb{Q})$, for the class $\eu{K}$ of kernels shown in the previous page.
\par We now discuss how kernels on $\Scr{P}$ can be obtained from $\gamma_k$. As discussed in the paper, $\gamma_k$ is a \emph{Hilbertian metric} on $\Scr{P}$. Therefore, using (\ref{Eq:hilb-kernel}), the associated kernel can be easily computed as 
\begin{equation}
K(\bb{P},\bb{Q})=\left\langle \int_M k(\cdot,x)\,d\bb{P}(x),\int_M k(\cdot,x)\,d\bb{Q}(x)\right\rangle_\eu{H}=\int\!\!\!\int_M k(x,y)\,d\bb{P}(x)\,d\bb{Q}(y),\nonumber
\end{equation} 
where $K:\Scr{P}\times\Scr{P}\rightarrow\bb{R}$ is a positive definite kernel, which can be seen as the dot-product kernel on $\Scr{P}$. Using the results in \citet[Chapter 3, Theorems 2.2 and 2.3]{Berg-84}, Gaussian and inverse multi-quadratic kernels on $\Scr{P}$ can be defined as 
\begin{equation}
K(\bb{P},\bb{Q})=\exp(-\sigma\gamma^2_k(\bb{P},\bb{Q})),\,\sigma>0\,\,\,\text{and}\,\,\, K(\bb{P},\bb{Q})=(\sigma+\gamma^2_k(\bb{P},\bb{Q}))^{-1},\,\sigma>0\nonumber
\end{equation}
respectively. Broadly, this relates to the work on Hilbertian metrics and positive definite kernels on probability measures by \citet{Hein-05} and \citet{Fuglede-03}.
\appendix
\newcommand{\appsection}[1]{\let\oldthesection\thesection
  \renewcommand{\thesection}{Appendix \oldthesection}
  \section{#1}\let\thesection\oldthesection}
\acks{B. K. S. wishes to acknowledge support from the Max Planck Institute (MPI) for Biological Cybernetics, the National Science Foundation (grant DMS-MSPA 0625409), the Fair Isaac Corporation and the University of California MICRO program. Part of this work was done while the first author was an intern at the MPI, and part was done while A. G. was a project scientist at CMU, under grants DARPA IPTO FA8750-09-1-0141, ONR MURI N000140710747, and ARO MURI W911NF0810242. This work is also supported by the IST Program of the EC, under the FP7 Network of Excellence, ICT-216886-NOE. B. K. S. wishes to thank Agnes Radl for her comments on the manuscript.}

\appendix
\section*{Appendix A. Supplementary Results}\label{appendix-a}
For completeness, we present the supplementary results that were used to prove the results in this paper. The following result is quoted from \citet[Theorem 8.14]{Folland-99}.
\begin{theorem}\label{Thm:Folland-1}
Suppose $\phi\in L^1(\bb{R}^d)$, $\int\phi(x)\,dx=a$ and $\phi_t(x)=t^{-d}\phi(t^{-1}x)$ for $t>0$. If $f$ is bounded and uniformly continuous on $\bb{R}^d$, then $f\ast\phi_t\rightarrow af$ uniformly as $t\rightarrow 0$.
\end{theorem}
\par\noindent
By imposing slightly stronger conditions on $\phi$, the following result quoted from \citet[Theorem 8.15]{Folland-99} shows that $f\ast\phi_t\rightarrow af$ almost everywhere for $f\in L^r(\bb{R}^d)$.
\begin{theorem}\label{Thm:Folland-2}
Suppose $|\phi(x)|\le C(1+\Vert x\Vert_2)^{-d-\varepsilon}$ for some $C,\varepsilon>0$, and $\int \phi(x)\,dx=a$. If $f\in L^r(\bb{R}^d)$ ($1\le r\le\infty$), then $f\ast\phi_t(x)\rightarrow af(x)$ as $t\rightarrow 0$ for every $x$ in the Lebesgue set of $f$~---~in particular, for almost every $x$, and for every $x$ at which $f$ is continuous.
\end{theorem}
\begin{theorem}[Fourier transform of a measure]\label{Theorem:FTmeasure}
Let $\mu$ be a finite Borel measure on $\bb{R}^d$. The Fourier transform of $\mu$ is given by
\begin{equation}\label{Eq:FTmeasure}
\widehat{\mu}(\omega)=\int_{\bb{R}^d}e^{-i\omega^Tx}\,d\mu(x),\,\,\omega\in\bb{R}^d,
\end{equation}
which is a bounded, uniformly continuous function on $\bb{R}^d$. In addition, $\widehat{\mu}$ satisfies the following properties:
\begin{enumerate}
\item[(i)] $\overline{\widehat{\mu}(\omega)}=\widehat{\mu}(-\omega),\,\forall\,\omega\in\bb{R}^d$, i.e., $\widehat{\mu}$ is conjugate symmetric,
\item[(ii)] $\widehat{\mu}(0)=1$. 
\end{enumerate}
\end{theorem}
\par\noindent
The following result, called the Riemann-Lebesgue lemma, is quoted from \citet[Theorem 7.5]{Rudin-91}.
\begin{lemma}[Riemann-Lebesgue]\label{lem:Riemann-Lebesgue}
If $f\in L^1(\bb{R}^d)$, then $\widehat{f}\in C_0(\bb{R}^d)$, and $\Vert\widehat{f}\Vert_\infty\le\Vert f\Vert_1$.
\end{lemma}
\par\noindent
The following theorem is a version of the \emph{Paley-Wiener theorem} for distributions, and is proved in \citet[Theorem 7.23]{Rudin-91}.
\begin{theorem}[Paley-Wiener]\label{Thm:paley-wiener}
If $f\in\Scr{D}^\prime_d$ has compact support 
, then $\widehat{f}$ is entire.
\end{theorem}
\par \noindent The following lemma provides a property of entire functions, which is quoted from \citet[Lemma 7.21]{Rudin-91}.
\begin{lemma}\label{lem:entire}
If $f$ is an entire function in $\bb{C}^d$ that vanishes on $\bb{R}^d$, then $f=0$.
\end{lemma}

\vskip 0.2in

\begin{thebibliography}{58}
\bibitem[Ali and Silvey(1966)]{Ali-66}
S.~M. Ali and S.~D. Silvey.
\newblock A general class of coefficients of divergence of one distribution
  from another.
\newblock \emph{Journal of the Royal Statistical Society, Series B
  (Methodological)}, 28:\penalty0 131--142, 1966.

\bibitem[Aronszajn(1950)]{Aronszajn-50}
N.~Aronszajn.
\newblock Theory of reproducing kernels.
\newblock \emph{Trans. Amer. Math. Soc.}, 68:\penalty0 337--404, 1950.

\bibitem[Bach and Jordan(2002)]{Bach-02}
F.~R. Bach and M.~I. Jordan.
\newblock Kernel independent component analysis.
\newblock \emph{Journal of Machine Learning Research}, 3:\penalty0 1--48, 2002.

\bibitem[Barbour and Chen(2005)]{Barbour-05}
A.~D. Barbour and L.~H.~Y. Chen.
\newblock \emph{An Introduction to Stein's Method}.
\newblock Singapore University Press, Singapore, 2005.

\bibitem[Berg et~al.(1984)Berg, Christensen, and Ressel]{Berg-84}
C.~Berg, J.~P.~R. Christensen, and P.~Ressel.
\newblock \emph{Harmonic Analysis on Semigroups}.
\newblock Spring Verlag, New York, 1984.

\bibitem[Berlinet and Thomas-Agnan(2004)]{Berlinet-04}
A.~Berlinet and C.~Thomas-Agnan.
\newblock \emph{Reproducing Kernel {H}ilbert Spaces in {P}robability and
  {S}tatistics}.
\newblock Kluwer Academic Publishers, London, UK, 2004.

\bibitem[Borgwardt et~al.(2006)Borgwardt, Gretton, Rasch, Kriegel,
  Sch{\"{o}}lkopf, and Smola]{Borgwardt-06b}
K.~M. Borgwardt, A.~Gretton, M.~Rasch, H.-P. Kriegel, B.~Sch{\"{o}}lkopf, and
  A.~J. Smola.
\newblock Integrating structured biological data by kernel maximum mean
  discrepancy.
\newblock \emph{Bioinformatics}, 22\penalty0 (14):\penalty0 e49--e57, 2006.

\bibitem[Br{\'{e}}maud(2001)]{Bremaud-01}
P.~Br{\'{e}}maud.
\newblock \emph{Mathematical Principles of Signal Processing}.
\newblock Springer-Verlag, New York, 2001.

\bibitem[Csisz{\'{a}}r(1967)]{Csiszar-67}
I.~Csisz{\'{a}}r.
\newblock Information-type measures of difference of probability distributions
  and indirect observations.
\newblock \emph{Studia Scientiarium Mathematicarum Hungarica}, 2:\penalty0
  299--318, 1967.

\bibitem[Dahmen and Micchelli(1987)]{Dahmen-87}
W.~Dahmen and C.~A. Micchelli.
\newblock Some remarks on ridge functions.
\newblock \emph{Approx. Theory Appl.}, 3:\penalty0 139--143, 1987.

\bibitem[del Barrio et~al.(1999)del Barrio, Cuesta-Albertos, Matr{\'{a}}n, and
  Rodr{\'{i}}guez-Rodr{\'{i}}guez]{Barrio-99}
E.~del Barrio, J.~A. Cuesta-Albertos, C.~Matr{\'{a}}n, and J.~M.
  Rodr{\'{i}}guez-Rodr{\'{i}}guez.
\newblock Testing of goodness of fit based on the {$L_2$}-{W}asserstein
  distance.
\newblock \emph{Annals of Statistics}, 27:\penalty0 1230--1239, 1999.

\bibitem[Devroye and Gy{\"{o}}rfi(1990)]{Devroye-90}
L.~Devroye and L.~Gy{\"{o}}rfi.
\newblock No empirical probability measure can converge in the total variation
  sense for all distributions.
\newblock \emph{Annals of Statistics}, 18\penalty0 (3):\penalty0 1496--1499,
  1990.

\bibitem[Dudley(2002)]{Dudley-02}
R.~M. Dudley.
\newblock \emph{Real Analysis and Probability}.
\newblock Cambridge University Press, Cambridge, UK, 2002.

\bibitem[Folland(1999)]{Folland-99}
G.~B. Folland.
\newblock \emph{Real Analysis: Modern Techniques and Their Applications}.
\newblock Wiley-Interscience, New York, 1999.

\bibitem[Fuglede and Tops{\o}e(2003)]{Fuglede-03}
B.~Fuglede and F.~Tops{\o}e.
\newblock {J}ensen-{S}hannon divergence and {H}ilbert space embedding, 2003.
\newblock Preprint.

\bibitem[Fukumizu et~al.(2004)Fukumizu, Bach, and Jordan]{Fukumizu-04}
K.~Fukumizu, F.~R. Bach, and M.~I. Jordan.
\newblock Dimensionality reduction for supervised learning with reproducing
  kernel {H}ilbert spaces.
\newblock \emph{Journal of Machine Learning Research}, 5:\penalty0 73--99,
  2004.

\bibitem[Fukumizu et~al.(2008)Fukumizu, Gretton, Sun, and
  Sch{\"{o}}lkopf]{Fukumizu-08a}
K.~Fukumizu, A.~Gretton, X.~Sun, and B.~Sch{\"{o}}lkopf.
\newblock Kernel measures of conditional dependence.
\newblock In J.C. Platt, D.~Koller, Y.~Singer, and S.~Roweis, editors,
  \emph{Advances in Neural Information Processing Systems 20}, pages 489--496,
  Cambridge, MA, 2008. MIT Press.

\bibitem[Fukumizu et~al.(2009{\natexlab{a}})Fukumizu, Bach, and
  Jordan]{Fukumizu-09}
K.~Fukumizu, F.~R. Bach, and M.~I. Jordan.
\newblock Kernel dimension reduction in regression.
\newblock \emph{Annals of Statistics}, 37\penalty0 (5):\penalty0 1871--1905,
  2009{\natexlab{a}}.

\bibitem[Fukumizu et~al.(2009{\natexlab{b}})Fukumizu, Sriperumbudur, Gretton,
  and Sch{\"{o}}lkopf]{Fukumizu-08b}
K.~Fukumizu, B.~K. Sriperumbudur, A.~Gretton, and B.~Sch{\"{o}}lkopf.
\newblock Characteristic kernels on groups and semigroups.
\newblock In D.~Koller, D.~Schuurmans, Y.~Bengio, and L.~Bottou, editors,
  \emph{Advances in Neural Information Processing Systems 21}, pages 473--480,
  2009{\natexlab{b}}.

\bibitem[Gasquet and Witomski(1999)]{Gasquet-99}
C.~Gasquet and P.~Witomski.
\newblock \emph{Fourier Analysis and Applications}.
\newblock Springer-Verlag, New York, 1999.

\bibitem[Gibbs and Su(2002)]{Gibbs-02}
A.~L. Gibbs and F.~E. Su.
\newblock On choosing and bounding probability metrics.
\newblock \emph{International Statistical Review}, 70\penalty0 (3):\penalty0
  419--435, 2002.

\bibitem[Gretton et~al.(2004)Gretton, Smola, Bousquet, Herbrich,
  Sch{\"{o}}lkopf, and Logothetis]{Gretton-04}
A.~Gretton, A.~Smola, O.~Bousquet, R.~Herbrich, B.~Sch{\"{o}}lkopf, and
  N.~Logothetis.
\newblock Behaviour and convergence of the constrained covariance.
\newblock Technical Report 130, MPI for Biological Cybernetics, 2004.

\bibitem[Gretton et~al.(2005)Gretton, Herbrich, Smola, Bousquet, and
  Sch{\"{o}}lkopf]{Gretton-05a}
A.~Gretton, R.~Herbrich, A.~Smola, O.~Bousquet, and B.~Sch{\"{o}}lkopf.
\newblock Kernel methods for measuring independence.
\newblock \emph{Journal of Machine Learning Research}, 6:\penalty0 2075--2129,
  December 2005.

\bibitem[Gretton et~al.(2007)Gretton, Borgwardt, Rasch, Sch{\"{o}}lkopf, and
  Smola]{Gretton-06}
A.~Gretton, K.~M. Borgwardt, M.~Rasch, B.~Sch{\"{o}}lkopf, and A.~Smola.
\newblock A kernel method for the two sample problem.
\newblock In B.~Sch\"{o}lkopf, J.~Platt, and T.~Hoffman, editors,
  \emph{Advances in Neural Information Processing Systems 19}, pages 513--520.
  MIT Press, 2007.

\bibitem[Gretton et~al.(2008)Gretton, Fukumizu, Teo, Song, Sch{\"{o}}lkopf, and
  Smola]{Gretton-08}
A.~Gretton, K.~Fukumizu, C.~H. Teo, L.~Song, B.~Sch{\"{o}}lkopf, and A.~J.
  Smola.
\newblock A kernel statistical test of independence.
\newblock In J.~Platt, D.~Koller, Y.~Singer, and S.~Roweis, editors,
  \emph{Advances in Neural Information Processing Systems 20}, pages 585--592.
  MIT Press, 2008.

\bibitem[Hein and Bousquet(2005)]{Hein-05}
M.~Hein and O.~Bousquet.
\newblock {H}ilbertian metrics and positive definite kernels on probability
  measures.
\newblock In \emph{AISTATS}, 2005.

\bibitem[Lehmann and Romano(2005)]{Lehmann-05}
E.~L. Lehmann and J.~P. Romano.
\newblock \emph{Testing Statistical Hypothesis}.
\newblock Springer-Verlag, New York, 2005.

\bibitem[Liese and Vajda(2006)]{Liese-06}
F.~Liese and I.~Vajda.
\newblock {On} divergences and informations in statistics and information
  theory.
\newblock \emph{IEEE Trans. Information Theory}, 52\penalty0 (10):\penalty0
  4394--4412, 2006.

\bibitem[Lindvall(1992)]{Lindvall-92}
T.~Lindvall.
\newblock \emph{Lectures on the Coupling Method}.
\newblock John Wiley \& Sons, New York, 1992.

\bibitem[Micchelli et~al.(2006)Micchelli, Xu, and Zhang]{Micchelli-06}
C.~A. Micchelli, Y.~Xu, and H.~Zhang.
\newblock Universal kernels.
\newblock \emph{Journal of Machine Learning Research}, 7:\penalty0 2651--2667,
  2006.

\bibitem[M{\"{u}}ller(1997)]{Muller-97}
A.~M{\"{u}}ller.
\newblock Integral probability metrics and their generating classes of
  functions.
\newblock \emph{Advances in Applied Probability}, 29:\penalty0 429--443, 1997.

\bibitem[Nguyen et~al.(2008)Nguyen, Wainwright, and Jordan]{Nguyen-08}
X.~Nguyen, M.~J. Wainwright, and M.~I. Jordan.
\newblock Estimating divergence functionals and the likelihood ratio by convex
  risk minimization.
\newblock Technical Report 764, Department of Statistics, University of
  California, Berkeley, 2008.

\bibitem[Pinkus(2004)]{Pinkus-04}
A.~Pinkus.
\newblock Strictly positive definite functions on a real inner product space.
\newblock \emph{Adv. Comput. Math.}, 20:\penalty0 263--271, 2004.

\bibitem[Rachev(1991)]{Rachev-91}
S.~T. Rachev.
\newblock \emph{Probability Metrics and the Stability of Stochastic Models}.
\newblock John Wiley \& Sons, Chichester, 1991.

\bibitem[Rachev and R{\"{u}}schendorf(1998)]{Rachev-98}
S.~T. Rachev and L.~R{\"{u}}schendorf.
\newblock \emph{Mass transportation problems. Vol. I Theory, Vol. II
  Applications}.
\newblock Probability and its Applications. Springer-Verlag, Berlin, 1998.

\bibitem[Rasmussen and Williams(2006)]{Rasmussen-06}
C.~E. Rasmussen and C.~K.~I. Williams.
\newblock \emph{Gaussian Processes for Machine Learning}.
\newblock MIT Press, Cambridge, MA, 2006.

\bibitem[Reed and Simon(1972)]{Reed-72}
M.~Reed and B.~Simon.
\newblock \emph{Functional Analysis}.
\newblock Academic Press, New York, 1972.

\bibitem[Rosenblatt(1975)]{Rosenblatt-75}
M.~Rosenblatt.
\newblock A quadratic measure of deviation of two-dimensional density estimates
  and a test of independence.
\newblock \emph{Annals of Statistics}, 3\penalty0 (1):\penalty0 1--14, 1975.

\bibitem[Rudin(1991)]{Rudin-91}
W.~Rudin.
\newblock \emph{Functional Analysis}.
\newblock McGraw-Hill, USA, 1991.

\bibitem[Sch{\"{o}}lkopf and Smola(2002)]{Scholkopf-02}
B.~Sch{\"{o}}lkopf and A.~J. Smola.
\newblock \emph{Learning with Kernels}.
\newblock MIT Press, Cambridge, MA, 2002.

\bibitem[Shorack(2000)]{Shorack-00}
G.~R. Shorack.
\newblock \emph{Probability for Statisticians}.
\newblock Springer-Verlag, New York, 2000.

\bibitem[Smola et~al.(2007)Smola, Gretton, Song, and Sch{\"o}lkopf]{Smola-07}
A.~J. Smola, A.~Gretton, L.~Song, and B.~Sch{\"o}lkopf.
\newblock A {H}ilbert space embedding for distributions.
\newblock In \emph{Proc. 18th International Conference on Algorithmic Learning
  Theory}, pages 13--31. Springer-Verlag, Berlin, Germany, 2007.

\bibitem[Sriperumbudur et~al.(2008)Sriperumbudur, Gretton, Fukumizu, Lanckriet,
  and Sch{\"{o}}lkopf]{Sriperumbudur-08}
B.~K. Sriperumbudur, A.~Gretton, K.~Fukumizu, G.~R.~G. Lanckriet, and
  B.~Sch{\"{o}}lkopf.
\newblock Injective {H}ilbert space embeddings of probability measures.
\newblock In R.~Servedio and T.~Zhang, editors, \emph{Proc. of the 21$^{st}$
  Annual Conference on Learning Theory}, pages 111--122, 2008.

\bibitem[Sriperumbudur et~al.(2009{\natexlab{a}})Sriperumbudur, Fukumizu,
  Gretton, Lanckriet, and Sch{\"{o}}lkopf]{Sriperumbudur-09c}
B.~K. Sriperumbudur, K.~Fukumizu, A.~Gretton, G.~R.~G. Lanckriet, and
  B.~Sch{\"{o}}lkopf.
\newblock Kernel choice and classifiability for {RKHS} embeddings of
  probability distributions.
\newblock In Y.~Bengio, D.~Schuurmans, J.~Lafferty, C.~K.~I. Williams, and
  A.~Culotta, editors, \emph{Advances in Neural Information Processing Systems
  22}, pages 1750--1758. MIT Press, 2009{\natexlab{a}}.

\bibitem[Sriperumbudur et~al.(2009{\natexlab{b}})Sriperumbudur, Fukumizu,
  Gretton, Sch{\"{o}}lkopf, and Lanckriet]{Sriperumbudur-09}
B.~K. Sriperumbudur, K.~Fukumizu, A.~Gretton, B.~Sch{\"{o}}lkopf, and G.~R.~G.
  Lanckriet.
\newblock On integral probability metrics, $\phi$-divergences and binary
  classification.
\newblock \emph{http://arxiv.org/abs/0901.2698v4}, October 2009{\natexlab{b}}.

\bibitem[Sriperumbudur et~al.(2010{\natexlab{a}})Sriperumbudur, Fukumizu, and
  Lanckriet]{Sriperumbudur-09d}
B.~K. Sriperumbudur, K.~Fukumizu, and G.~R.~G. Lanckriet.
\newblock On the relation between universality, characteristic kernels and
  {RKHS} embedding of measures.
\newblock 2010{\natexlab{a}}.
\newblock Submitted to AISTATS.

\bibitem[Sriperumbudur et~al.(2010{\natexlab{b}})Sriperumbudur, Fukumizu, and
  Lanckriet]{Sriperumbudur-09e}
B.~K. Sriperumbudur, K.~Fukumizu, and G.~R.~G. Lanckriet.
\newblock Universality, characteristic kernels and {RKHS} embedding of
  measures.
\newblock 2010{\natexlab{b}}.
\newblock In preparation.

\bibitem[Stein(1972)]{Stein-72}
C.~Stein.
\newblock A bound for the error in the normal approximation to the distribution
  of a sum of dependent random variables.
\newblock In \emph{Proc. of the Sixth Berkeley Symposium on Mathematical
  Statistics and Probability}, 1972.

\bibitem[Steinwart(2001)]{Steinwart-01}
I.~Steinwart.
\newblock On the influence of the kernel on the consistency of support vector
  machines.
\newblock \emph{Journal of Machine Learning Research}, 2:\penalty0 67--93,
  2001.

\bibitem[Steinwart and Christmann(2008)]{Steinwart-08}
I.~Steinwart and A.~Christmann.
\newblock \emph{Support Vector Machines}.
\newblock Springer, 2008.

\bibitem[Stewart(1976)]{Stewart-76}
J.~Stewart.
\newblock Positive definite functions and generalizations, an historical
  survey.
\newblock \emph{Rocky Mountain Journal of Mathematics}, 6\penalty0
  (3):\penalty0 409--433, 1976.

\bibitem[Vajda(1989)]{Vajda-89}
I.~Vajda.
\newblock \emph{{T}heory of Statistical Inference and Information}.
\newblock Kluwer Academic Publishers, Boston, 1989.

\bibitem[Vallander(1973)]{Vallander-73}
S.~S. Vallander.
\newblock Calculation of the {W}asserstein distance between probability
  distributions on the line.
\newblock \emph{Theory Probab. Appl.}, 18:\penalty0 784--786, 1973.

\bibitem[van~der Vaart and Wellner(1996)]{Vaart-96}
A.~W. van~der Vaart and J.~A. Wellner.
\newblock \emph{Weak Convergence and Empirical Processes}.
\newblock Springer-Verlag, New York, 1996.

\bibitem[Vapnik(1998)]{Vapnik-98}
V.~N. Vapnik.
\newblock \emph{{S}tatistical {L}earning {T}heory}.
\newblock Wiley, New York, 1998.

\bibitem[Wang et~al.(2005)Wang, Kulkarni, and Verd{\'{u}}]{Wang-05}
Q.~Wang, S.~R. Kulkarni, and S.~Verd{\'{u}}.
\newblock Divergence estimation of continuous distributions based on
  data-dependent partitions.
\newblock \emph{IEEE Trans. Information Theory}, 51\penalty0 (9):\penalty0
  3064--3074, 2005.

\bibitem[Weaver(1999)]{Weaver-99}
N.~Weaver.
\newblock \emph{Lipschitz Algebras}.
\newblock World Scientific Publishing Company, 1999.

\bibitem[Wendland(2005)]{Wendland-05}
H.~Wendland.
\newblock \emph{Scattered Data Approximation}.
\newblock Cambridge University Press, Cambridge, UK, 2005.

\end{thebibliography}

\end{document}